\documentclass[journal]{IEEEtran}

\usepackage[utf8]{inputenc}
\usepackage[T1]{fontenc}
\usepackage{amsmath,amssymb,amsfonts,amsthm}
\usepackage{booktabs}
\usepackage{array}
\usepackage{multirow}
\usepackage{graphicx}
\usepackage{xcolor}
\usepackage{placeins}
\usepackage{longtable}
\usepackage{float}
\usepackage{algorithm}
\usepackage{algpseudocode}
\usepackage[numbers,sort&compress]{natbib}

\usepackage[colorlinks=true,linkcolor=blue,citecolor=blue,urlcolor=blue]{hyperref}
\hypersetup{pdftitle={Why Does Train--Validation Separation Emerge? Gradient-Allocation Dynamics in Pretrained Fine-Tuning},pdfauthor={Yuchen Li, Zongqi Fan, Ken-Tye Yong, Nguyen H. Tran}}

\newtheorem{lemma}{Lemma}
\newtheorem{proposition}{Proposition}
\newtheorem{theorem}{Theorem}

\newtheorem{assumption}{Assumption}

\newcommand{\R}{\mathbb{R}}
\newcommand{\E}{\mathbb{E}}
\newcommand{\DeltaC}{\Delta^{C-1}}
\newcommand{\Dden}{D_{\mathrm{den}}}
\newcommand{\Dw}{D_{\mathrm{within}}^{\mathrm{wt}}}
\newcommand{\Db}{D_{\mathrm{between}}}
\newcommand{\Dmod}{D_{\mathrm{mod}}}
\newcommand{\Uden}{U_{\mathrm{den}}}
\newcommand{\Uw}{U_{\mathrm{within}}^{\mathrm{wt}}}
\newcommand{\GapAcc}{\mathrm{GapAcc}}
\newcommand{\GapLoss}{\mathrm{GapLoss}}
\newcommand{\TrainAcc}{\mathrm{TrainAcc}}
\newcommand{\ValAcc}{\mathrm{ValAcc}}
\newcommand{\TrainLoss}{\mathrm{TrainLoss}}
\newcommand{\ValLoss}{\mathrm{ValLoss}}
\newcommand{\tr}{\mathrm{tr}}
\newcommand{\val}{\mathrm{val}}
\newcommand{\od}{\mathrm{d}}
\newcommand{\Rdir}{R^{\mathrm{dir}}}
\newcommand{\Rsc}{R^{\mathrm{sc}}}
\newcolumntype{L}[1]{>{\raggedright\arraybackslash}p{#1}}
\newcommand{\Supp}[1]{Supplementary Section~\ref{#1}}
\newcommand{\SuppTab}[1]{Supplementary Table~\ref{#1}}
\newcommand{\SuppFig}[1]{Supplementary Figure~\ref{#1}}

\newcommand{\NumTaskMccUndefined}{2}
\newcommand{\NumTaskWJoint}{30}
\newcommand{\NumTaskDdenJoint}{30}
\newcommand{\NumModRobertaRaw}{.316}
\newcommand{\NumModDebertaRaw}{.554}
\newcommand{\NumModQwenRaw}{.744}
\newcommand{\NumModQwenDdenRaw}{.770}
\newcommand{\NumModNlpRaw}{.538}
\newcommand{\NumModNlpLowess}{.601}
\newcommand{\NumModNlpJoint}{79}
\newcommand{\NumModNlpDenHigher}{81}
\newcommand{\NumModNlpSd}{.365}
\newcommand{\NumModNlpDdenSd}{.174}
\newcommand{\NumModNlpPartial}{.388}
\newcommand{\NumModNlpPartialPos}{72}
\newcommand{\NumModNlpDiff}{.200}
\newcommand{\NumModNlpDiffPos}{66}
\newcommand{\NumModResRaw}{.855}
\newcommand{\NumModResLowess}{.894}
\newcommand{\NumModResDenHigher}{11}
\newcommand{\NumModResDdenRaw}{.887}
\newcommand{\NumModResDdenLowess}{.916}
\newcommand{\NumModCondAllPos}{12}
\newcommand{\NumModCondLowest}{\ensuremath{-.022}}
\newcommand{\NumModCondDdenLowest}{.550}
\newcommand{\NumModRobertaRteRaw}{\ensuremath{-.022}}
\newcommand{\NumModRobertaRtePos}{2}
\newcommand{\NumModRobertaMrpcRaw}{.116}
\newcommand{\NumModRobertaMrpcPos}{2}
\newcommand{\NumDepthTwoDdenRaw}{.366}
\newcommand{\NumDepthTwoDdenPos}{11}
\newcommand{\NumDepthTenDdenRaw}{.645}
\newcommand{\NumDepthTenDdenPos}{15}
\newcommand{\NumBoolqDepthMin}{.667}
\newcommand{\NumBoolqDepthMax}{.946}
\newcommand{\NumBoolqDepthLowessMin}{.778}
\newcommand{\NumBoolqDepthLowessMax}{.973}
\newcommand{\NumBoolqQwenDdenTwo}{.762}
\newcommand{\NumBoolqQwenDdenTen}{.945}
\newcommand{\NumSignedDdenRaw}{.829}
\newcommand{\NumSignedDdenPos}{15}
\newcommand{\NumSignedDwRaw}{.903}
\newcommand{\NumSignedDbRaw}{.196}
\newcommand{\NumSignedDbPos}{9}
\newcommand{\NumSignedAbsDdenSd}{.137}
\newcommand{\NumSignedAbsDdenMin}{.485}
\newcommand{\NumSignedDdenSd}{.240}
\newcommand{\NumSignedDdenMin}{.189}
\newcommand{\NumSignedAbsDbPos}{15}
\newcommand{\NumNlpDdenRaw}{.841}
\newcommand{\NumNlpDdenLowess}{.895}
\newcommand{\NumNlpWRaw}{.784}
\newcommand{\NumNlpWLowess}{.864}
\newcommand{\NumNlpDdenPartial}{.772}
\newcommand{\NumNlpDdenPartialPos}{89}
\newcommand{\NumNlpWPartial}{.706}
\newcommand{\NumNlpWPartialPos}{88}
\newcommand{\NumNlpDdenDiff}{.587}
\newcommand{\NumNlpDdenDiffPos}{85}
\newcommand{\NumNlpWDiff}{.518}
\newcommand{\NumNlpWDiffPos}{86}
\newcommand{\NumNlpDdenDiffLowessPos}{87}
\newcommand{\NumNlpWDiffLowessPos}{86}
\newcommand{\NumNlpTimeGap}{.691}
\newcommand{\NumNlpLossDdenLevel}{.803}
\newcommand{\NumNlpLossDdenLevelPos}{88}
\newcommand{\NumNlpLossDdenDiff}{.449}
\newcommand{\NumNlpLossDdenDiffPos}{82}
\newcommand{\NumNlpDzeroJoint}{90}
\newcommand{\NumNlpDoneJoint}{83}
\newcommand{\NumNlpBetweenRaw}{.783}
\newcommand{\NumNlpBetweenJoint}{88}
\newcommand{\NumNlpNoBoolqDdenRaw}{.825}
\newcommand{\NumNlpNoBoolqDdenLowess}{.884}
\newcommand{\NumNlpNoBoolqDdenJoint}{75}
\newcommand{\NumNlpNoBoolqWRaw}{.761}
\newcommand{\NumNlpNoBoolqWLowess}{.848}
\newcommand{\NumNlpNoBoolqWJoint}{75}
\newcommand{\NumNlpIntervalMin}{0.980}
\newcommand{\NumNlpIntervalMax}{1.011}
\newcommand{\NumNlpGapEpochTwo}{12.4}
\newcommand{\NumNlpGapFinal}{16.7}
\newcommand{\NumNlpBetweenRatioFirst}{8.8}
\newcommand{\NumNlpBetweenRatioLast}{23.9}
\newcommand{\NumCiNinetyLower}{.960}
\newcommand{\NumResGapThree}{.185}
\newcommand{\NumResGapFive}{.331}
\newcommand{\NumResGapSeven}{.527}
\newcommand{\NumResDdenLossRaw}{.830}
\newcommand{\NumResDdenLossPartial}{\ensuremath{-.145}}
\newcommand{\NumResDdenLossPartialPos}{6}
\newcommand{\NumTurnMin}{6}
\newcommand{\NumTurnMax}{11}
\newcommand{\NumTurnMean}{8.4}
\newcommand{\NumTurnLossGrowth}{87}
\newcommand{\NumTurnAccGrowth}{56}
\newcommand{\NumTurnValAccNotLower}{12}
\newcommand{\NumTurnValAccStart}{.587}
\newcommand{\NumTurnValAccEnd}{.643}
\newcommand{\NumTurnValLossStart}{.900}
\newcommand{\NumTurnValLossEnd}{2.555}
\newcommand{\NumNlpNegCheckpoints}{23}
\newcommand{\NumNlpNegLastEpoch}{4}
\newcommand{\NumNlpNonNegDdenRaw}{.821}
\newcommand{\NumNlpNonNegDdenRawPos}{88}
\newcommand{\NumResNegCheckpoints}{38}
\newcommand{\NumResNegLastEpoch}{7}
\newcommand{\NumResNonNegDdenRaw}{.844}
\newcommand{\NumResNonNegDdenRawPos}{15}
\newcommand{\NumNlpDdenPartialQuad}{.468}
\newcommand{\NumNlpDdenPartialQuadPos}{82}
\newcommand{\NumNlpDdenPartialLog}{.435}
\newcommand{\NumNlpDdenPartialLogPos}{79}
\newcommand{\NumNlpWPartialQuadPos}{79}
\newcommand{\NumNlpWPartialLogPos}{80}
\newcommand{\NumMassRuns}{30}
\newcommand{\NumMassQwenRuns}{10}
\newcommand{\NumMassEvaluations}{79{,}360}
\newcommand{\NumMassEncoderMin}{0.71}
\newcommand{\NumMassEncoderRatio}{\ensuremath{1.4\times10^{-12}}}
\newcommand{\NumMassQwenMin}{\ensuremath{1.7\times10^{-9}}}
\newcommand{\NumMassQwenRatio}{\ensuremath{6.0\times10^{-4}}}

\begin{document}

\title{Why Does Train--Validation Separation Emerge?\\Gradient-Allocation Dynamics in Pretrained Fine-Tuning\thanks{This work has been submitted to the IEEE for possible publication. Copyright may be transferred without notice, after which this version may no longer be accessible.}}

\author{Yuchen Li, Zongqi Fan, Ken-Tye Yong, and Nguyen H. Tran%
\thanks{The authors are with the School of Computer Science, The University of Sydney, Sydney, NSW, Australia. Corresponding authors: Ken-Tye Yong (ken.yong@sydney.edu.au) and Nguyen H. Tran (nguyen.tran@sydney.edu.au).}}

\maketitle

\begin{abstract}
During fine-tuning, training accuracy often keeps rising after validation accuracy has levelled off. We explain this growing train--validation gap with a support-transfer account: early updates use structure shared by many examples, whereas later fitting relies increasingly on narrow structure whose benefit transfers less to held-out data, so per-example gradient allocations differentiate as the gap grows. We measure this with two readouts computed from fixed training probes alone, the pooled dispersion $\Dden$ of normalized gradient magnitudes and its within-class part $\Dw$, and derive sufficient conditions under which allocation dispersion, cross-entropy separation, and the accuracy gap grow together. In a constructed feature hierarchy, raising the sample-private feature share from $.3$ to $.7$ raises the final accuracy gap from $\NumResGapThree$ to $\NumResGapSeven$, and homogenizing per-example gradient magnitudes reduces the gap mainly after substantial fitting (median full-dose changes of $-4.91$ to $-11.79$ percentage points when it starts at the three later training stages). In 90 fine-tuning runs of three language backbones on six tasks, both readouts are positively associated with the accuracy gap in every run, and $\Dden$ remains positive in $\NumNlpDdenPartialPos$ of 90 runs after linear detrending; 20 vision runs show the same association.
\end{abstract}

\begin{IEEEkeywords}
Train--validation gap, fine-tuning, per-example gradients, training dynamics, generalization.
\end{IEEEkeywords}

\section{Introduction}
\label{sec:intro}

\IEEEPARstart{W}{hen} a pretrained model is fine-tuned, its accuracy on the training data often keeps rising after its accuracy on held-out data has levelled off. We write the accuracy gap at training time $t$ as
\begin{equation}
\label{eq:gapacc}
\GapAcc(t)=\TrainAcc(t)-\ValAcc(t),
\end{equation}
where $\TrainAcc(t)$ is measured on a fixed training-evaluation subset and $\ValAcc(t)$ on the validation split, and we call the trajectory $t\mapsto\GapAcc(t)$ \emph{train--validation separation}. We study this finite-sample gap as it develops within one adaptation run and ask what changes in the way individual training examples drive the update while this gap grows.

Our answer is a \emph{support-transfer account} (Fig.~\ref{fig:concept}A). A pretrained backbone contains structure that helps many examples, including validation examples; we call it \emph{shared support} and the loss still reducible through it \emph{shared demand}. Structure that lowers the loss of only a few training examples provides \emph{narrow}, or in the limit \emph{private}, support whose benefit carries over less. Early updates are dominated by directions that many examples agree on, so shared demand is consumed first. Later, more of the remaining loss can only be reduced through narrow support, and training gains produce smaller validation gains. If different examples are then fitted through different narrow structure, their gradients load different parameters (Fig.~\ref{fig:concept}B). We measure this with the \emph{gradient allocation} of fixed training probes, the normalized absolute gradient of each probe's loss over a set of monitored parameters, and with its dispersion across probes, $\Dden$, together with its within-class part $\Dw$ and its unnormalized companion $\Uden$, which keeps gradient magnitude. All three use training data alone.

\begin{figure*}[t]
\centering
\includegraphics[width=\textwidth]{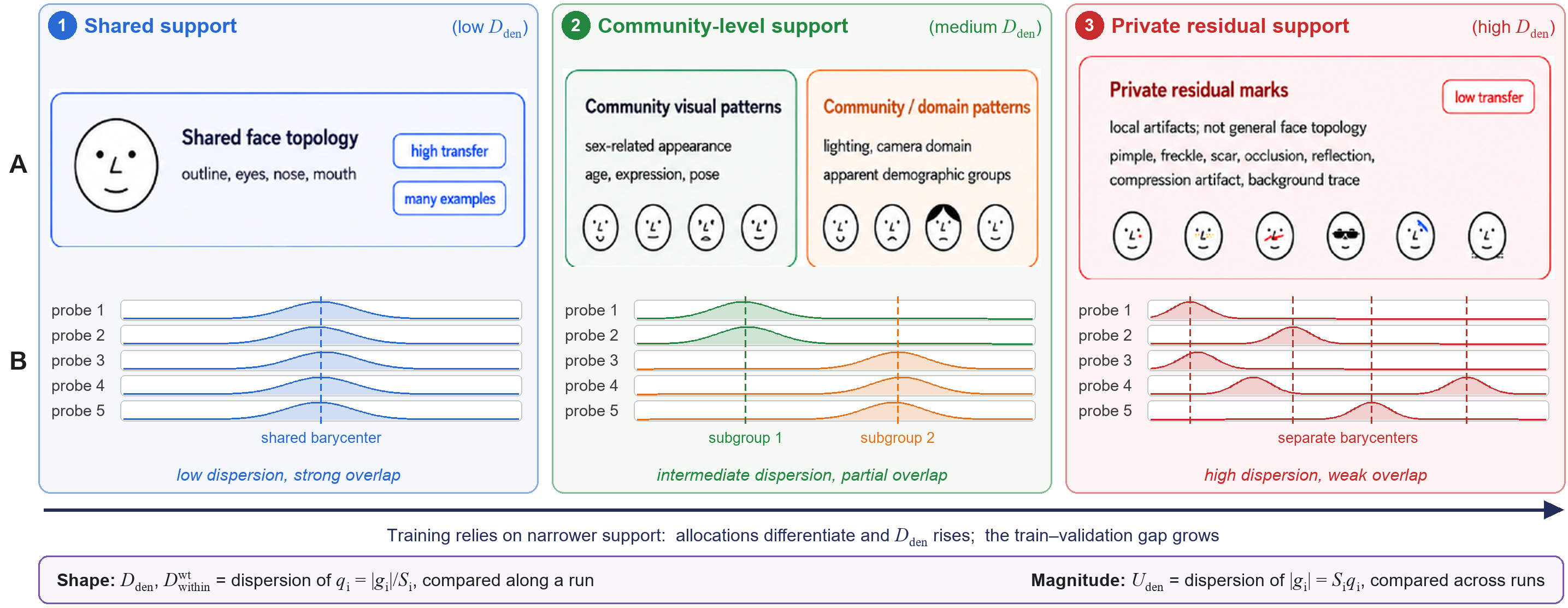}
\caption{Conceptual illustration. (A) Face recognition as an analogy for support shared by all examples, by groups of examples, or by single examples. (B) Gradient allocations of training probes at the three levels of support: they overlap while updates use shared support and differentiate as fitting relies on narrower support. The normalized readouts $\Dden$ and $\Dw$ measure this shape dispersion along a run; the unnormalized $\Uden$ keeps gradient magnitude and compares runs. Both panels are schematic.}
\label{fig:concept}
\end{figure*}

The contributions are as follows.
\begin{enumerate}
\item \emph{A local theory of how the gap develops} (Section~\ref{sec:theory}). We give sufficient conditions under which the consumption of shared demand raises allocation dispersion and lower held-out transfer of residual support raises directional cross-entropy separation, an exact direction--scale decomposition of the cross-entropy gap, and a finite-sample boundary-crossing condition for accuracy-gap growth.
\item \emph{Training-data-only readouts} (Section~\ref{sec:measurement}) with an exact class decomposition and a division of roles: the normalized readouts compare co-movement with the gap within a trajectory, and their unnormalized companions compare magnitudes across trajectories.
\item \emph{Evidence from construction, intervention, and real training} (Section~\ref{sec:experiments}). In a constructed feature hierarchy, the private share sets the size of the gap and the readouts follow it in every run. Homogenizing per-example gradient magnitudes has no systematic effect on the gap early in training and reduces it once fitting is advanced; at a fixed entry point the reduction appears in each of five runs. In 90 language fine-tuning runs, the readouts follow the gap in every run, the association persists beyond time trends and in one-epoch changes, and 20 vision runs extend the association across domains.
\end{enumerate}

\section{Related Work}
\label{sec:related}

\paragraph*{Generalization and learning order}
Uniform-convergence, norm-, and margin-based analyses bound the population gap of a trained model through capacity measures \citep{vapnik1998statistical,bartlett2002rademacher,neyshabur2017exploring,jiang2020fantastic}, and overparameterized networks generalize better than capacity alone suggests \citep{zhang2017understanding,belkin2019reconciling}. We study instead the evolving empirical gap within one adaptation run. Networks tend to fit common patterns before rare or atypical examples \citep{arpit2017closer,toneva2019empirical}, training dynamics separate easy from hard examples \citep{swayamdipta2020dataset,baldock2021deep}, and memorizing rare examples can be necessary on long-tailed data \citep{feldman2020does,feldman2020neural}. We share the premise that what is being fitted changes over training and record how each example's gradient is allocated over monitored parameters.

\paragraph*{Per-example gradient statistics}
Coherent-gradient analyses argue that gradient descent amplifies directions supported by many examples \citep{chatterjee2020coherent}; stiffness measures how a step on one example changes the loss of another \citep{fort2019stiffness}; gradient diversity and the gradient signal-to-noise ratio summarize agreement among per-example gradients \citep{yin2018gradient,liu2020understanding}; gradient norms and their variance score example difficulty \citep{paul2021deep,agarwal2022estimating}; and gradient dot products trace example influence \citep{pruthi2020estimating}. Our readout asks where, over the monitored coordinates, each example's loss sensitivity is located: absolute values and normalization keep this allocation and remove sign and total magnitude, so the readout measures allocation heterogeneity rather than whether one example's update benefits another. Linear probes, SVCCA, CKA, and neural-collapse analyses characterize activations rather than this allocation \citep{alain2017understanding,raghu2017svcca,kornblith2019similarity,papyan2020prevalence}.

\paragraph*{Transfer, logit scale, and calibration}
Features differ in how well they transfer \citep{yosinski2014transferable}, gradient descent can starve some predictive features \citep{pezeshki2021gradient}, and full fine-tuning can distort pretrained features \citep{kumar2022finetuning}; we model the consequence of differing held-out transfer for the gap over time. Late in training, validation cross-entropy can rise while accuracy is unchanged \citep{guo2017calibration}, and on separable data the logit scale grows while the normalized predictor converges in direction \citep{soudry2018implicit,lyu2020gradient}. Lemma~\ref{lem:direction-scale} applies this scale--direction separation exactly to the rate of the train--validation cross-entropy gap.

\section{Gradient-Allocation Readouts}
\label{sec:measurement}

\subsection{Gradient Allocation}
Training accuracy and loss are evaluated on a fixed subset of the training set, and validation accuracy and loss on the full validation split. Besides Eq.~\eqref{eq:gapacc} we record the cross-entropy gap $\GapLoss(t)=\ValLoss(t)-\TrainLoss(t)$. Before training we fix a set of $N_0$ training examples, the probes $z_i=(x_i,y_i)$, and a set $\mathcal C$ of $C$ trainable parameter entries, the monitored coordinates, in the output projections of the final $k$ backbone blocks (Section~\ref{sec:setup}). For probe $i$ at checkpoint $t$, with parameters $\theta_t$ and per-example cross-entropy $\ell$,
\begin{equation}
\label{eq:probe-gradient}
g_i(t)=\nabla_{\theta_{\mathcal C}}\ell(z_i;\theta_t)\in\R^C,\qquad S_i(t)=\|g_i(t)\|_1,
\end{equation}
computed for the single example in evaluation mode. If $S_i(t)>0$, the \emph{gradient allocation} of probe $i$ is
\begin{equation}
\label{eq:allocation}
q_i(t)=|g_i(t)|/S_i(t)\in\DeltaC ,
\end{equation}
where $|g|$ is the entrywise absolute value and $\DeltaC$ the probability simplex. It records how the probe's first-order loss sensitivity is distributed over the monitored coordinates, independently of its total size. The implementation divides by $S_i(t)+\varepsilon$ with $\varepsilon=10^{-12}$ and uses the probes with finite $S_i(t)>\varepsilon$, whose effective total and class counts we write $N$ and $N_y$; the stabilized allocation is within $\varepsilon/S_i(t)$ of the exact one (\Supp{app:norm-stability}).

\subsection{Dispersion Readouts}
With the mean allocation $\bar q(t)$, the pooled dispersion is
\begin{equation}
\label{eq:dden}
\Dden(t)=\frac1N\sum_{i=1}^N\|q_i(t)-\bar q(t)\|_2^2 .
\end{equation}
It is zero exactly when all probes have the same allocation and grows when probes place their sensitivity on different coordinates or in different proportions. For class $y$ with probe set $I_y$, weight $\pi_y=N_y/N$, class mean $\bar q_y(t)$, and within-class dispersion $D_y(t)=N_y^{-1}\sum_{i\in I_y}\|q_i(t)-\bar q_y(t)\|_2^2$,
\begin{equation}
\label{eq:within-between}
\begin{aligned}
&\Dw=\textstyle\sum_y\pi_yD_y,\qquad \Db=\textstyle\sum_y\pi_y\|\bar q_y-\bar q\|_2^2,\\
&\Dden=\Dw+\Db .
\end{aligned}
\end{equation}
The identity is exact (\Supp{app:class-decomposition}). $\Dw$ compares only examples \emph{with the same label}, so its association with the gap cannot be produced by movement of class centres alone.

\subsection{Sign, Scale, and Shape}
\label{sec:absolute-gradients}
A probe gradient with $S_i>0$ factors as $g_i=s_i\odot S_iq_i$, with sign pattern $s_i\in\{-1,1\}^C$, scale $S_i$, and shape $q_i$, and the three factors carry different information. The signs are label-conditioned: at an affine layer the per-example weight gradient is the product $\delta_ih_i^\top$ of the backpropagated error and the input activation, and in a binary output layer $\partial\ell_i/\partial a_i=\sigma(a_i)-y_i$ has opposite signs for the two labels, so the signs form class-specific modes and can reverse for one probe between nearby checkpoints. The scale combines the probe's confidence weight with the size of its margin gradient (Section~\ref{sec:readout-roles}). The shape is the allocation of Eq.~\eqref{eq:allocation}, the quantity described by the theory of Section~\ref{sec:theory}. Taking absolute values removes the sign factor. This is the engineering choice that makes the readouts robust to sign noise, and it keeps any difference in the coordinates or proportions over which sensitivity is distributed (\Supp{app:sign}).

\subsection{Readout Roles}
\label{sec:readout-roles}
Three families of readouts combine these factors in different ways and serve three comparisons (Table~\ref{tab:readout-roles}). Fitting progress enters through the scale:

\begin{proposition}[Confidence-weight invariance]
\label{prop:binary-confidence}
In a binary task with probe margin $\varphi_i=f_{y_i}(x_i)-f_{1-y_i}(x_i)$, $\nabla_{\theta_{\mathcal C}}\ell(z_i)=-\sigma(-\varphi_i)\nabla_{\theta_{\mathcal C}}\varphi_i$, so $S_i=\sigma(-\varphi_i)\|\nabla_{\theta_{\mathcal C}}\varphi_i\|_1$ and, whenever $\nabla_{\theta_{\mathcal C}}\varphi_i\ne0$, $q_i=|\nabla_{\theta_{\mathcal C}}\varphi_i|/\|\nabla_{\theta_{\mathcal C}}\varphi_i\|_1$.
\end{proposition}
\begin{proof}
$\ell=\log(1+e^{-\varphi})$ has margin derivative $-\sigma(-\varphi)$; this positive scalar cancels in $q_i$ and remains in $S_i$.
\end{proof}

Since $\sigma(-\varphi_i)\le e^{-\varphi_i}$, the confidence weight shrinks exponentially as a probe is fitted, and it enters the scale $S_i$ alone. The normalized readouts $\Dden$ and $\Dw$ remove the scale and give every probe the same weight, so they describe the shape of the allocation pattern, which in the binary language tasks is the allocation of margin sensitivity. A within-run correlation of a normalized readout with the gap thus measures whether this shape changes with the gap along one trajectory, and such scale-free correlations are comparable across runs, tasks, backbones, and monitors.

The unnormalized companions $\Uden$ and $\Uw$ replace $q_i$ by $|g_i|=S_iq_i$ in Eqs.~\eqref{eq:dden}--\eqref{eq:within-between} and keep the scale. They compare \emph{how much} gradient dispersion one trajectory carries relative to another, which Section~\ref{sec:theory} relates to the residual mass. Such magnitudes are directly comparable when the backbone, monitored coordinates, and data construction are shared. This holds in the controlled hierarchy, whose trajectories differ in the private share, and in the intervention, whose branches start from shared model states and whose transformation acts on $|g_i|$ directly; across the language and vision backbones, the normalized readouts provide the common scale. $\Uden$ thus combines the shape dispersion with the squared gradient scale. Signed readouts keep the sign factor and serve as an ablation of the absolute-value choice (\Supp{app:sensitivity-signed}).

\begin{table}[t]
\centering
\caption{Readouts and Their Roles}
\label{tab:readout-roles}
\footnotesize
\setlength{\tabcolsep}{3pt}
\begin{tabular}{@{}L{0.2\linewidth}L{0.16\linewidth}L{0.21\linewidth}L{0.36\linewidth}@{}}
\toprule
Readout & Dispersion of & Reported in & Role\\
\midrule
$\Dden$, $\Dw$ & $|g_i|/S_i$ & all studies ($\Dden$ in the intervention) & primary; within-run association, comparable across runs\\
$\Uden$, $\Uw$ & $|g_i|$ & hierarchy, intervention & magnitude comparison across trajectories\\
signed $\Dden$, $\Dw$ & $g_i/S_i$ & hierarchy (supplement) & ablation of absolute values\\
\bottomrule
\end{tabular}
\end{table}

\section{A Local Theory}
\label{sec:theory}

The local model links the consumption of shared demand, the dispersion of allocations, and train--validation separation under explicit conditions. The latent quantities are the remaining loss demand $A(t)$ and $B(t)$ addressable by shared and by residual, that is narrow or private, support, the residual weight $\lambda=B/(A+B)$, and the held-out transfer efficiencies $\tau_u$ and $\tau_v$ of the two kinds of support; $\Dden$, $\Dw$, and the training and validation losses are observed. The results concern local intervals on which the set of valid probes is fixed; proofs are in \Supp{app:theory}.

\subsection{Residual Demand Becomes Dominant}
A gradient direction shared by many examples survives minibatch averaging, whereas approximately zero-mean example-specific directions are attenuated at rate $m^{-1/2}$ for batch size $m$ (\Supp{app:batch-coherence}); this is the optimization-side reason why broadly supported directions dominate early updates \citep{chatterjee2020coherent,arpit2017closer}.
\begin{assumption}[Rate ordering]
\label{ass:rates}
On a local interval $J$, $\dot A=-\kappa_uA$ and $\dot B=-\kappa_vB$ with fixed rates $\kappa_u>\kappa_v>0$.
\end{assumption}
Then
\begin{equation}
\label{eq:lambda-dynamics}
\dot\lambda(t)=\lambda(t)\bigl(1-\lambda(t)\bigr)(\kappa_u-\kappa_v)>0 :
\end{equation}
residual demand becomes relatively dominant because shared demand is consumed faster, even if it does not grow in absolute terms.

\subsection{Residual Weight and Allocation Dispersion}
On a short interval we approximate the monitored gradient magnitude of probe $i$ as a shared source with mass $A$ and shape $u$ plus a residual source with mass $B$ and probe-dependent shape $v_i$, $|g_i|\approx c_i(Au+Bv_i)$. Normalization removes $c_i$, and the unnormalized companions keep it.
\begin{assumption}[Two-source allocation]
\label{ass:mixture}
At time $t$, for every probe,
\begin{equation}
\label{eq:latent-mixture}
q_i=(1-\lambda)u+\lambda v_i+e_i,\qquad u,v_i\in\DeltaC,
\end{equation}
where $e_i$ collects the approximation error.
\end{assumption}
Let $\sigma_v^2$ and $\eta^2$ be the cross-probe variances of the residual shapes $v_i$ and of the errors $e_i$ (defined like Eq.~\eqref{eq:dden}).
\begin{theorem}[Allocation dispersion under residual growth]
\label{thm:observable-dispersion}
Under Assumption~\ref{ass:mixture} with $N\ge2$, $(\lambda\sigma_v-\eta)_+^2\le\Dden\le(\lambda\sigma_v+\eta)^2$, and $\Dden=\lambda^2\sigma_v^2$ if $e_i=0$. For $t_1<t_2$, $\Dden(t_2)>\Dden(t_1)$ whenever $\lambda(t_2)\sigma_v(t_2)-\lambda(t_1)\sigma_v(t_1)>\eta(t_1)+\eta(t_2)$.
\end{theorem}
Centering removes the shared shape $u$, so the readout constrains the product of residual weight and residual heterogeneity. On a differentiable noiseless interval,
\begin{equation}
\label{eq:relative-growth}
\frac{\dot\Dden}{2\Dden}=(1-\lambda)(\kappa_u-\kappa_v)+\frac{\dot\sigma_v}{\sigma_v},
\end{equation}
so dispersion grows unless the residual shapes become alike faster than the residual weight rises. The unnormalized companion keeps the masses: with $e_i=0$ and a common scale $c_i=c$, $\Uden=c^2B^2\sigma_v^2=c^2(A+B)^2\Dden$, the shape dispersion multiplied by the squared gradient scale: it carries the residual mass $B$ where $\Dden$ carries the residual weight $\lambda$, and unequal scales add the spread of the gradient sizes. $\Dden$ thus suits comparisons along a trajectory, and $\Uden$ comparisons of magnitude between trajectories with the same backbone, coordinates, and data construction (Section~\ref{sec:readout-roles}).

\subsection{Direction and Scale of the Cross-Entropy Gap}
\label{sec:direction-scale}
Let $R_\bullet$ be the mean cross-entropy of split $\bullet\in\{\tr,\val\}$ with $n_\bullet$ examples, $G=R_\val-R_\tr$ the cross-entropy gap, $f_i$ the logit vector of example $i$, $p_i=\mathrm{softmax}(f_i)$, and $e_y$ the one-hot vector of label $y$. Fix a positive differentiable scale $\gamma(t)$, such as the root-mean-square norm of the centred probe logits, and let $\hat f_i=f_i/\gamma$.

\begin{lemma}[Direction--scale decomposition]
\label{lem:direction-scale}
With $\omega_i=\hat f_{i,y_i}-p_i^\top\hat f_i$ and $\Omega_\bullet=n_\bullet^{-1}\sum_{i\in\mathcal E_\bullet}\omega_i$,
\begin{equation}
\label{eq:direction-scale}
\dot R_\bullet=\underbrace{\frac{\gamma}{n_\bullet}\sum_{i\in\mathcal E_\bullet}(p_i-e_{y_i})^\top\dot{\hat f}_i}_{\dot\Rdir_\bullet}\;\underbrace{-\,\dot\gamma\,\Omega_\bullet}_{\dot\Rsc_\bullet},
\end{equation}
\begin{equation}
\label{eq:gap-direction-scale}
\dot G=\Gamma^{\mathrm{dir}}+\dot\gamma\,(\Omega_\tr-\Omega_\val),\qquad \Gamma^{\mathrm{dir}}=\dot\Rdir_\val-\dot\Rdir_\tr .
\end{equation}
Predicted labels, and hence $\GapAcc$, depend on $\hat f_i$ alone.
\end{lemma}
Since $\omega_i\ge0$ for correctly classified examples and $\omega_i\to-d_i<0$ for misclassified ones as the scale grows ($d_i$ is the normalized deficit of the true class), uniform logit growth raises the cross-entropy gap at rate $\dot\gamma(\bar d_\val-\bar d_\tr)$ while no prediction changes (Fig.~\ref{fig:direction-scale}). This is the familiar late-training regime in which validation cross-entropy rises although validation accuracy does not fall \citep{guo2017calibration}. The cross-entropy gap thus combines changes of the predictions with changes of confidence, whereas the accuracy gap responds to the direction channel alone. We therefore test the readouts against both gaps, with the accuracy gap as the primary target.

\begin{figure}[t]
\centering
\includegraphics[width=\linewidth]{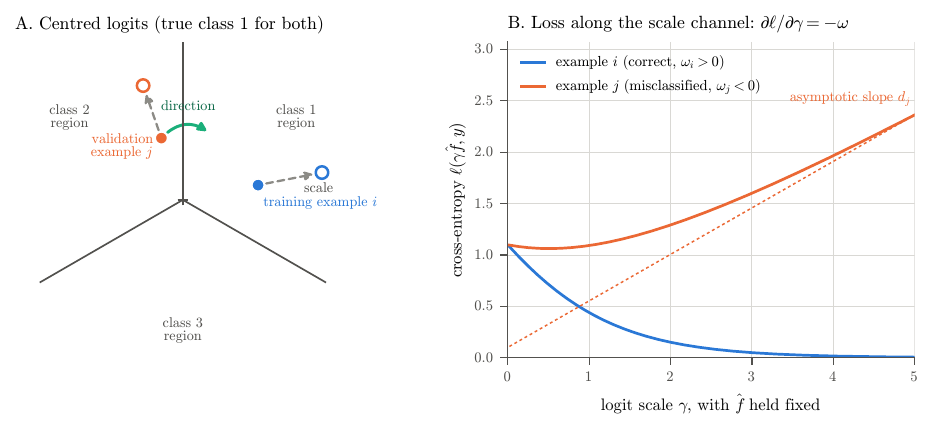}
\caption{Direction--scale decomposition for a three-class task (schematic). (A) Centred logits of a training example $i$ (blue) and a validation example $j$ (orange) of true class 1; filled markers show $\hat f$, hollow markers the same direction at a larger scale. Moving along a ray (scale channel) never changes a prediction, whereas a change of $\hat f$ (direction channel, green) can. (B) Cross-entropy of the two examples as $\gamma$ grows with $\hat f$ fixed: it falls for the correctly classified example and rises toward slope $d_j$ for the misclassified one.}
\label{fig:direction-scale}
\end{figure}

\begin{table*}[t]
\centering
\caption{Predictions of the Account and Where the Evidence Stands}
\label{tab:evidence}
\footnotesize
\setlength{\tabcolsep}{4pt}
\begin{tabular}{@{}L{0.03\textwidth}L{0.36\textwidth}L{0.17\textwidth}L{0.38\textwidth}@{}}
\toprule
 & Prediction & Study & Result\\
\midrule
P1 & A larger private share enlarges the accuracy gap, and private-only inputs retain more loss on validation than on training & Hierarchy, 15 runs & Final $\GapAcc$ $\NumResGapThree\to\NumResGapFive\to\NumResGapSeven$ at $p=.3,.5,.7$; the private masked-loss share grows more on validation by $+7.1$, $+7.6$, $+16.3$ pp\\
P2 & The normalized readouts rise with the gap at every private share, and the unnormalized magnitude grows with the private share & Hierarchy, 15 runs & $\Dden$ and $\Dw$ positively associated in level with $\GapAcc$ and $\GapLoss$ in 15/15 runs; mean (median) increment of $\Uden$ $156\to312\to678$ ($60\to126\to539$) at $p=.3,.5,.7$\\
P3 & Reducing gradient differentiation has no systematic effect early in training and reduces the gap once fitting relies on narrow support & Intervention, 288 branches & Median full-dose $\Delta\GapAcc$ of $-.53$, $-3.65$, $-7.30$, $-11.79$, and $-4.91$ pp at initialization and the $50$, $75$, $87.5$, and $93.75\%$ entry points, with a negative dose correlation in 6, 7, 9, 11, and 10 of 12 conditions; negative in 5/5 runs at a fixed entry point\\
P4 & The readouts follow both gaps in real fine-tuning: within classes, beyond time trends, and in local changes for the accuracy gap & Language, 90 runs; vision, 20 runs & With $\GapAcc$: positive in 90/90 language runs ($\Dden$, $\Dw$), $\NumNlpDdenPartialPos$/90 after linear detrending, $\NumNlpDdenDiffPos$/90 in one-epoch changes, 20/20 vision runs. With $\GapLoss$: $\Dden$ positive in $\NumNlpLossDdenLevelPos$/90 language runs\\
P5 & The two gaps can diverge: validation cross-entropy can rise while accuracy holds & Hierarchy, 15 runs & After its minimum, mean validation cross-entropy rises from $\NumTurnValLossStart$ to $\NumTurnValLossEnd$ while validation accuracy does not fall in $\NumTurnValAccNotLower$/15 runs\\
\bottomrule
\end{tabular}
\end{table*}

\subsection{Directional Transfer and Accuracy Changes}
Let $r^{\mathrm{dir}}=-\dot\Rdir_\tr$ be the directional reduction rate of the training loss.
\begin{assumption}[Directional transfer]
\label{ass:transfer}
On $J$ there are fixed efficiencies $\tau_u>\tau_v$ and an error bound $\epsilon(t)\ge0$ with
\begin{equation}
\label{eq:transfer}
\bigl|-\dot\Rdir_\val-[(1-\lambda)\tau_u+\lambda\tau_v]\,r^{\mathrm{dir}}\bigr|\le\epsilon .
\end{equation}
\end{assumption}
The assumption constrains only the directional channel; the scale channel is handled exactly by Lemma~\ref{lem:direction-scale}.

\begin{proposition}[Lower residual transfer raises directional separation]
\label{prop:directional-transfer}
Under Assumption~\ref{ass:transfer} with $r^{\mathrm{dir}}>0$, $|\Gamma^{\mathrm{dir}}-r^{\mathrm{dir}}\Pi(\lambda)|\le\epsilon$ with $\Pi(\lambda)=(1-\tau_u)+\lambda(\tau_u-\tau_v)$, so $\od\Pi/\od\lambda>0$ and $\Gamma^{\mathrm{dir}}>0$ whenever $r^{\mathrm{dir}}\Pi(\lambda)>\epsilon$; in the noiseless case this is equivalent to
\begin{equation}
\label{eq:effective-transfer-condition}
(1-\lambda)\tau_u+\lambda\tau_v<1 .
\end{equation}
Moreover,
\begin{equation}
\label{eq:full-gap-bound}
\dot G\ge r^{\mathrm{dir}}\Pi(\lambda)-\epsilon+\dot\gamma(\Omega_\tr-\Omega_\val).
\end{equation}
\end{proposition}
Combining Assumptions~\ref{ass:rates}--\ref{ass:transfer} with the growth condition of Eq.~\eqref{eq:relative-growth} gives an interval on which $\lambda$ and $\Dden$ increase and $\Gamma^{\mathrm{dir}}$ is positive (Theorem~\ref{thm:combined} in the supplement). Accuracy, however, changes only when predictions cross decision boundaries.

\begin{proposition}[Finite-sample accuracy-gap change]
\label{prop:accuracy-crossings}
For checkpoints $t_0<t_1$ and fixed evaluation sets, let $C_S^{\pm}$ count examples of split $S$ that change from incorrect to correct ($+$) or from correct to incorrect ($-$) under a fixed tie-breaking rule. Then
\begin{equation}
\label{eq:accuracy-crossings}
\Delta\GapAcc=\frac{C_\tr^+-C_\tr^-}{n_\tr}-\frac{C_\val^+-C_\val^-}{n_\val}.
\end{equation}
With nonzero normalized margins $m_i^r=\hat f_{i,y_i}(t_r)-\max_{k\ne y_i}\hat f_{i,k}(t_r)$, $\Delta m_i=m_i^1-m_i^0$, and $B_S(h)$ the fraction of initially incorrect examples of $S$ within distance $h$ of the boundary, if every training margin is non-decreasing, training examples with $-a<m_i^0\le0$ gain at least $a>0$, and validation margins gain at most $b\ge0$, then
\begin{equation}
\label{eq:accuracy-margin-bound}
\Delta\GapAcc\ge B_\tr(a)-B_\val(b).
\end{equation}
\end{proposition}
The accuracy gap therefore grows when training improvements correct more boundary-near examples than validation improvements do. When $B_\tr(a)>B_\val(b)$ and the condition of Theorem~\ref{thm:observable-dispersion} both hold, $\Dden$ and $\GapAcc$ increase across the same checkpoints.

Table~\ref{tab:evidence} lists the predictions tested in Section~\ref{sec:experiments}. P1 and P2 follow from Proposition~\ref{prop:directional-transfer} and Theorem~\ref{thm:observable-dispersion} when a larger private share raises the residual weight, P3 from the account's reading of allocation dispersion as the mark of narrow support, P4 from the same results along real trajectories, and P5 from Lemma~\ref{lem:direction-scale}.

\section{Experiments}
\label{sec:experiments}

\begin{figure*}[t]
\centering
\includegraphics[width=\textwidth]{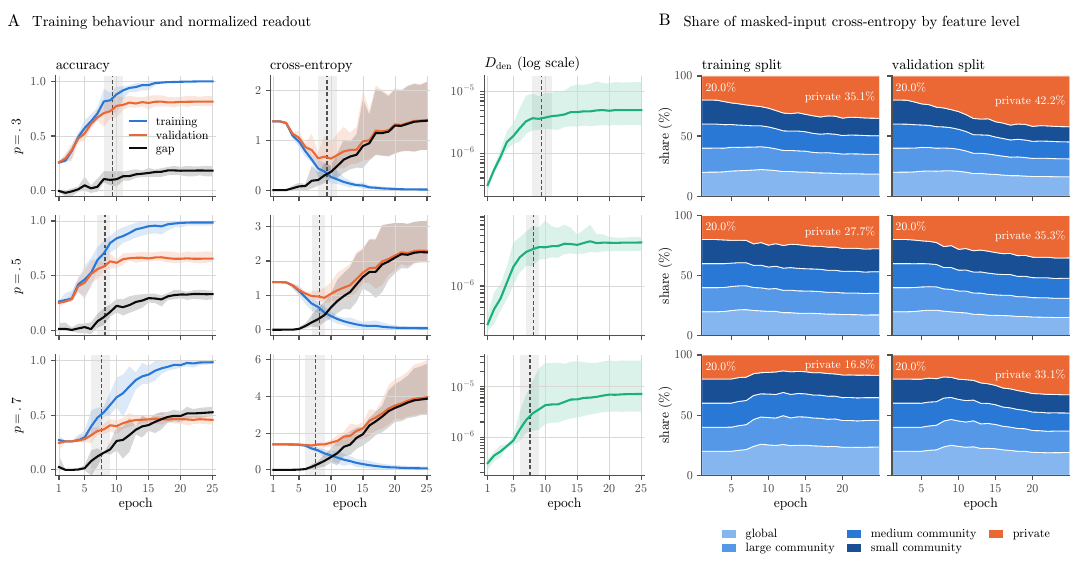}
\caption{Controlled hierarchy, one row per private share. (A) Accuracies and their gap, cross-entropies and their gap, and $\Dden$ over training: lines are five-run means (geometric mean for $\Dden$), bands the range over runs, and the grey band marks the epochs of minimal validation cross-entropy. (B) Share of the cross-entropy of single-level masked inputs carried by each feature level on the training and validation splits (five-run means); labels give the private share at epochs 1 and 25.}
\label{fig:resmlp-behaviour}
\end{figure*}

\subsection{Setup and Evidence Map}
\label{sec:setup}
\label{sec:protocol}
\begin{table}[t]
\centering
\caption{Controlled Hierarchy, Five-Run Means: Final Gaps at Epoch 25; Within-Run Correlation of Each Readout With $\GapAcc$ (Raw / LOWESS, All Runs Positive Under Both Views); Epoch-1-to-25 Increment of $\Uden$; Change of the Private Masked-Loss Share on the Training / Validation Split (Percentage Points)}
\label{tab:resmlp-main}
\footnotesize
\setlength{\tabcolsep}{2.4pt}
\begin{tabular}{@{}ccccccc@{}}
\toprule
$p$ & $\GapAcc$ & $\GapLoss$ & $\Dw$ & $\Dden$ & $\Delta\Uden$ & $\Delta$ masked share\\
\midrule
.3 & .185 & 1.398 & .703 / .757 & .911 / .951 & 156 & $+15.13$ / $+22.25$\\
.5 & .331 & 2.251 & .632 / .679 & .777 / .812 & 312 & $+7.70$ / $+15.26$\\
.7 & .527 & 3.878 & .891 / .911 & .972 / .984 & 678 & $-3.19$ / $+13.10$\\
\bottomrule
\end{tabular}
\end{table}

Four studies test the predictions of the account (Table~\ref{tab:evidence}). A constructed feature hierarchy on a pretrained ResMLP-24 \citep{touvron2021resmlp} makes the support structure explicit; a paired intervention in the same hierarchy acts on per-example gradient magnitudes; RoBERTa-base \citep{liu2019roberta}, DeBERTa-v3-base \citep{he2023debertav3}, and Qwen3-1.7B with LoRA \citep{yang2025qwen3,hu2022lora} are fine-tuned on six GLUE and SuperGLUE tasks \citep{wang2019glue,wang2019superglue}; and ResNet-18 is fine-tuned on four image datasets as an extension. All observational studies monitor the parameters of the projections that write back into the residual stream in the final $k$ blocks, with $k$ fixed per backbone before the main runs, and compute the readouts on fixed training probes, 128 in the language and hierarchy studies (\Supp{app:monitor}). Learning rates of the language and vision studies are selected on a held-out part of the official training data (\Supp{app:hpo}).

Correlations are computed within each run across its checkpoints and averaged with equal run weights; we report means and the number of positive runs, treating runs rather than checkpoints as replicates. Levels are correlated both raw and after separate LOWESS smoothing of the two series. Time controls correlate residuals after removing a linear, quadratic, or logarithmic epoch trend, and change analyses correlate changes over matched one-epoch windows (\Supp{app:statistics-details}). Every observational study reports the same two primary readouts, $\Dden$ and $\Dw$, against the accuracy gap, and the hierarchy and language studies also against the cross-entropy gap (Section~\ref{sec:direction-scale}); the unnormalized companions enter where trajectories share the backbone, coordinates, and data, in the hierarchy and the intervention. All other readouts and measurement settings are treated as ablations (Section~\ref{sec:sensitivity-summary}).

\subsection{Controlled Feature Hierarchy}
\label{sec:results-resmlp}
\label{sec:design-hierarchy}

Each synthetic example of class $y$ concatenates five feature blocks,
\begin{equation}
x=\bigl[\sqrt{w_g}\,u_{y};\sqrt{w_l}\,v^{(l)}_{c_l,y};\sqrt{w_m}\,v^{(m)}_{c_m,y};\sqrt{w_s}\,v^{(s)}_{c_s,y};\sqrt{w_p}\,e_{x}\bigr]+\nu ,
\end{equation}
a class prototype, prototypes shared within large, medium, and small communities, a coordinate owned by the example alone, and noise. No private coordinate is shared between a training and a validation example. The private share $p\in\{.3,.5,.7\}$ varies while the four reusable levels keep the ratio $1{:}2{:}3{:}4$,
\begin{equation}
\label{eq:target-mixture}
\bar w(p)=\bigl(\tfrac{1-p}{10},\tfrac{2(1-p)}{10},\tfrac{3(1-p)}{10},\tfrac{4(1-p)}{10},p\bigr),
\end{equation}
so the experiment tests the response to reallocating weight toward private support (\Supp{app:hierarchy-construction}). To inspect the loss carried by each level, each split is also evaluated with only one level retained.

Raising the private share from $.3$ to $.7$ raises the final accuracy gap from $\NumResGapThree$ to $\NumResGapSeven$ while training accuracy approaches one in all conditions (Table~\ref{tab:resmlp-main}, Fig.~\ref{fig:resmlp-behaviour}A). With only private features visible, validation examples retain relatively more loss than training examples: from epoch 1 to 25 the private masked-loss share grows more on the validation split by $7.1$, $7.6$, and $16.3$ percentage points at $p=.3,.5,.7$ (Table~\ref{tab:resmlp-main}, Fig.~\ref{fig:resmlp-behaviour}B), the split-specific burden the account predicts for support that cannot transfer.

 Both normalized readouts are positively correlated with the accuracy gap in all 15 runs under both views. The unnormalized companion adds the comparison across trajectories: its epoch-1-to-25 increment is positive in every run and scales with the private share, with mean $\Delta\Uden$ of $156$, $312$, and $678$ and medians of $60$, $126$, and $539$ at $p=.3,.5,.7$, while the normalized readouts track the gap within every trajectory at all three shares (Table~\ref{tab:resmlp-main}; \SuppFig{fig:resmlp-runs}B shows every checkpoint of the five runs).

\subsection{Gradient-Differentiation Intervention}
\label{sec:results-do}
\label{sec:design-do}

\begin{figure}[t]
\centering
\includegraphics[width=\linewidth]{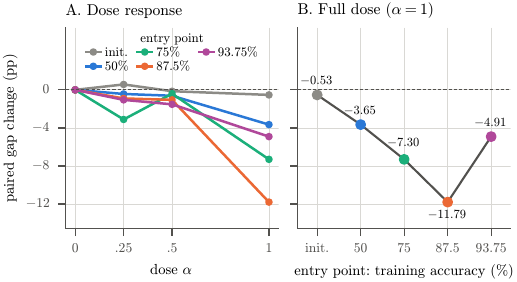}
\caption{Intervention response in the 240-branch matrix by entry point, the training accuracy at which the intervention starts: median paired accuracy-gap change relative to $\alpha=0$ over the 12 batch--private-share conditions. (A) Dose response for each entry point. (B) Full-dose response against the entry point.}
\label{fig:do-response}
\end{figure}

The intervention acts directly on per-example gradient magnitudes during training. For selected coordinates, let $V$ contain the examples of a batch with finite gradient mass above the numerical threshold; for $i\in V$ the update uses
\begin{equation}
\label{eq:do-transform}
\widetilde g_i=s_i^+\odot\bigl((1-\alpha)|g_i|+\alpha\overline a_V\bigr),\qquad \overline a_V=|V|^{-1}\textstyle\sum_{j\in V}|g_j|,
\end{equation}
where $s_i^+$ keeps each coordinate's sign. The valid examples' magnitude dispersion contracts by exactly $(1-\alpha)^2$ in every update, and the total valid-batch $\ell_1$ mass is preserved. For each batch size in $\{16,32,64,128\}$ and each private share, a baseline run stores the initial state and the first checkpoint reaching training accuracy $50\%$, $75\%$, $87.5\%$, and $93.75\%$. Each of these five entry points starts four branches, $\alpha\in\{0,.25,.5,1\}$, with identical model, optimizer, and data-stream state, which run 150 updates (240 branches). Responses are paired with the zero-dose branch at the same measurement point. A supplementary matrix adds four further training runs at batch size 16 and the $75\%$ entry point (48 branches; \Supp{app:do}).

\begin{figure*}[t]
\centering
\includegraphics[width=0.8\textwidth]{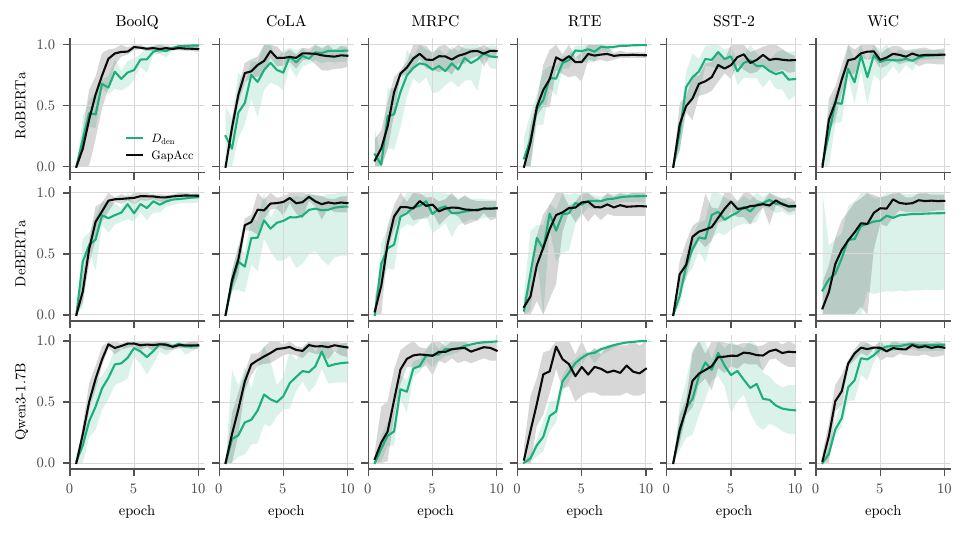}
\caption{Language fine-tuning: $\Dden$ (green) and $\GapAcc$ (black) over training for each backbone (rows) and task (columns). Each series is min--max scaled within its run before averaging; lines are five-run means and bands the range over runs.}
\label{fig:nlp-trajectories}
\end{figure*}

\begin{table*}[t]
\centering
\caption{Language Fine-Tuning: Mean Within-Run Correlation of the Readouts With the Accuracy Gap and With the Cross-Entropy Gap, by Backbone (30 Runs Each), With the Number of Positive Runs in Parentheses; Levels With $\GapAcc$ Are Raw / LOWESS}
\label{tab:nlp-main}
\footnotesize
\setlength{\tabcolsep}{6pt}
\begin{tabular}{@{}llccccccc@{}}
\toprule
 & & \multicolumn{5}{c}{With $\GapAcc$} & \multicolumn{2}{c}{With $\GapLoss$}\\
\cmidrule(lr){3-7}\cmidrule(l){8-9}
 & & & \multicolumn{3}{c}{Time trend removed} & One-epoch & & One-epoch\\
\cmidrule(lr){4-6}
Readout & Backbone & Level & Linear & Quadratic & Logarithmic & changes & Level & changes\\
\midrule
$\Dden$ & RoBERTa-base & .877 / .933 (30) & .795 (30) & .524 (28) & .475 (28) & .644 (30) & .826 (30) & .551 (29)\\
 & DeBERTa-v3-base & .878 / .935 (30) & .807 (30) & .559 (30) & .521 (28) & .596 (29) & .814 (30) & .382 (28)\\
 & Qwen3-1.7B & .770 / .819 (30) & .715 (29) & .323 (24) & .308 (23) & .522 (26) & .770 (28) & .414 (25)\\
 & All 90 runs & .841 / .895 (90) & .772 (89) & .468 (82) & .435 (79) & .587 (85) & .803 (88) & .449 (82)\\
\midrule
$\Dw$ & RoBERTa-base & .763 / .866 (30) & .647 (29) & .376 (25) & .367 (27) & .495 (29) & .695 (30) & .383 (27)\\
 & DeBERTa-v3-base & .792 / .885 (30) & .751 (30) & .509 (29) & .538 (30) & .524 (29) & .640 (29) & .241 (25)\\
 & Qwen3-1.7B & .797 / .842 (30) & .720 (29) & .364 (25) & .332 (23) & .534 (28) & .814 (29) & .359 (25)\\
 & All 90 runs & .784 / .864 (90) & .706 (88) & .416 (79) & .412 (80) & .518 (86) & .716 (88) & .328 (77)\\
\bottomrule
\end{tabular}
\end{table*}

The gap becomes sensitive to gradient differentiation as fitting develops (Fig.~\ref{fig:do-response}B). The median full-dose gap change is $-.53$ and $-3.65$ percentage points at initialization and the $50\%$ entry point and $-7.30$, $-11.79$, and $-4.91$ points at the three later entry points, where the dose correlation is negative in 9, 11, and 10 of 12 batch--private-share conditions (\SuppTab{tab:do-summary}). The account predicts this stage dependence: early updates use shared support, on which the gradients of different examples agree, so homogenizing them has no systematic effect (the dose correlation is negative in 6 of 12 conditions at initialization), whereas later fitting draws on narrow, example-specific support, whose gradient differentiation the intervention suppresses. In the five-run matrix the full-dose change is negative in every run (mean $-4.46\pm2.03$ points), with a negative dose correlation in 14 of 15 run--private-share conditions; \Supp{app:do-endpoints} splits the response by batch size. Training and validation accuracy change by $-5.44$ and $-.99$ points, so the gap shrinks mainly because training-side fitting is reduced. Measured on the fixed probes, the two readouts separate: full-dose branches end with larger $\Uden$ than their controls in 56 of 60 conditions (median ratio 23.9), while their allocation dispersion $\Dden$ is lower in 51 of 60 (median ratio .67), moving with the reduced gap. Suppressing gradient differentiation thus removes mainly the later, training-side gains, as expected once fitting relies on narrow support.

\subsection{Language Fine-Tuning}
\label{sec:results-nlp}
\label{sec:design-real}
RoBERTa-base and DeBERTa-v3-base are fine-tuned fully and Qwen3-1.7B with rank-32 LoRA adapters on BoolQ, CoLA, MRPC, RTE, SST-2, and WiC for 10 epochs, five runs per backbone--task pair. Across the 90 runs the accuracy gap grows from $\NumNlpGapEpochTwo$ percentage points near epoch 2 to $\NumNlpGapFinal$ at epoch 10 on average.

Both readouts are positively correlated with the accuracy gap in all 90 runs under both views, in each backbone (Table~\ref{tab:nlp-main}, Fig.~\ref{fig:nlp-trajectories}). Because $\Dw$ compares probes that share a label, the association is not produced by class centres alone. Level correlations describe accumulated co-variation, so we also remove fitted time trends: the association of $\Dden$ stays positive in $\NumNlpDdenPartialPos$, $\NumNlpDdenPartialQuadPos$, and $\NumNlpDdenPartialLogPos$ of 90 runs after linear, quadratic, and logarithmic detrending, and changes over matched one-epoch intervals co-vary positively in $\NumNlpDdenDiffPos$/90 runs for $\Dden$ and $\NumNlpWDiffPos$/90 for $\Dw$. Larger local increases of the readouts thus accompany larger local increases of the gap; \SuppFig{fig:nlp-distribution} shows the spread across runs.

\subsection{Cross-Entropy Gap and Accuracy Gap}
\label{sec:results-loss}
\label{sec:results-resmlp-loss}
Lemma~\ref{lem:direction-scale} allows the two gaps to diverge. In the hierarchy, after the validation-loss minimum (epochs $\NumTurnMin$--$\NumTurnMax$), mean validation cross-entropy rises from $\NumTurnValLossStart$ to $\NumTurnValLossEnd$ while mean validation accuracy rises from $\NumTurnValAccStart$ to $\NumTurnValAccEnd$ and does not fall in $\NumTurnValAccNotLower$ of 15 runs, the pattern the scale channel permits. The readouts also rise with the cross-entropy gap: the level association of $\Dden$ is positive in all 15 hierarchy runs and in $\NumNlpLossDdenLevelPos$ of 90 language runs (mean $\NumNlpLossDdenLevel$; Table~\ref{tab:nlp-main}). The readouts track the direction channel; the scale channel adds cross-entropy change at fixed predictions.

\subsection{Ablations and Extension}
\label{sec:results-vision}
\label{sec:sensitivity-summary}
As a cross-domain extension, ImageNet-pretrained ResNet-18 is fine-tuned for 40 epochs on CIFAR-10, CIFAR-100, STL-10, and DTD \citep{cimpoi2014describing}, five runs each. $\Dden$ and $\Dw$ are positively correlated with the accuracy gap in all 20 runs under both views, with mean raw / LOWESS correlations of $.410$ / $.518$ and $.446$ / $.556$ (\Supp{app:results-vision}). The ablations change one choice of the measurement at a time: the monitor depth, the number of probes, the exclusion of low-mass probes, the numerical precision, a module-level projection of the allocation, and signed gradients. The mean association with the accuracy gap stays positive in each of them (\Supp{app:sensitivity}).

\section{Discussion}
\label{sec:discussion}

\paragraph*{Alternative explanations}
The analyses above separate the readout--gap association from several alternative sources. A common dependence on training time is addressed by the detrended and one-epoch-change correlations (Table~\ref{tab:nlp-main}). Separation of class centres is excluded as the sole source because $\Dw$ removes the between-class term and remains positive in every language and hierarchy run. Changing confidence enters the cross-entropy gradient through one scalar weight, which normalization cancels exactly in the binary language tasks (Proposition~\ref{prop:binary-confidence}); in the hierarchy, excluding low-mass probes leaves the association almost unchanged. Example difficulty enters the normalized readout through each probe's allocation pattern, since every probe has equal weight, and the hierarchy controls example-specific demand directly through the private share. More differentiated processing of examples is the mechanism the account describes, and the hierarchy and the intervention link it to held-out transfer by varying private support and gradient differentiation directly (\SuppTab{tab:alternatives} summarizes these analyses).

\paragraph*{Scope and future work}
The theory describes local dynamics under explicit sharing and transfer conditions, and the experiments cover adaptation from pretrained backbones in language and vision. Training from random initialization, in which feature formation precedes support transfer, estimating the transfer efficiencies $\tau_u$ and $\tau_v$, and measuring the direction and scale channels along real trajectories would test the account more finely.

\section{Conclusion}
\label{sec:conclusion}
We developed a dynamical account of train--validation separation during pretrained fine-tuning: training increasingly relies on narrow structure whose benefits transfer less to held-out examples. The theory links this shift to allocation dispersion and directional cross-entropy separation, and an exact direction--scale decomposition with a finite-sample boundary-crossing result explains how loss and accuracy gaps relate. In a constructed hierarchy, shifting weight toward private support enlarges the gap, and homogenizing per-example gradient magnitudes reduces it mainly after substantial fitting. In 90 language runs and a 20-run vision extension, readouts computed from training probes alone follow the accuracy gap, and in the language runs the association persists beyond time trends and in local changes. Per-example gradient allocation thus offers a training-data-only view of how train--validation separation develops.

\FloatBarrier
\section*{Acknowledgment}
The authors used OpenAI's ChatGPT/Codex and Anthropic's Claude to assist with manuscript drafting and editing, translation, figure design, code inspection, and formatting. All research ideas, experimental decisions, analyses, claims, and conclusions are the authors' own; the authors reviewed all AI-assisted content and take full responsibility for the article.

\clearpage
\onecolumn
\makeatletter
\long\def\@makecaption#1#2{%
  \vskip\abovecaptionskip
  \sbox\@tempboxa{\small #1: #2}%
  \ifdim\wd\@tempboxa>\hsize
    {\small\noindent #1: #2\par}%
  \else
    \global\@minipagefalse
    \hb@xt@\hsize{\hfil\box\@tempboxa\hfil}%
  \fi
  \vskip\belowcaptionskip}
\makeatother
\setlength{\abovecaptionskip}{5pt}
\setlength{\belowcaptionskip}{5pt}
\renewcommand{\tablename}{Table}
\renewcommand{\figurename}{Figure}
\setcounter{section}{0}
\setcounter{table}{0}
\setcounter{figure}{0}
\setcounter{equation}{0}
\setcounter{theorem}{0}
\setcounter{proposition}{0}
\setcounter{lemma}{0}
\setcounter{algorithm}{0}
\renewcommand{\thesection}{S\arabic{section}}
\renewcommand{\thetable}{S\arabic{table}}
\renewcommand{\thefigure}{S\arabic{figure}}
\renewcommand{\theequation}{S\arabic{equation}}
\renewcommand{\thetheorem}{S\arabic{theorem}}
\renewcommand{\theproposition}{S\arabic{proposition}}
\renewcommand{\thelemma}{S\arabic{lemma}}
\renewcommand{\thealgorithm}{S\arabic{algorithm}}
\makeatletter
\def\theHsection{S\arabic{section}}
\def\theHtable{S\arabic{table}}
\def\theHfigure{S\arabic{figure}}
\def\theHequation{S\arabic{equation}}
\def\theHtheorem{S\arabic{theorem}}
\def\theHproposition{S\arabic{proposition}}
\def\theHlemma{S\arabic{lemma}}
\def\theHalgorithm{S\arabic{algorithm}}
\makeatother
\begin{center}
{\LARGE Supplementary Material}
\end{center}
\vspace{0.5em}

\noindent This supplementary material gives the proofs and technical results (Section~\ref{app:theory}), the experimental protocol (Section~\ref{app:repro}), the detailed results of the four studies (Section~\ref{app:results}), the complete intervention matrices (Section~\ref{app:do}), the sensitivity and readout ablations (Section~\ref{app:sensitivity}), and the training settings of every study (Section~\ref{app:config}). Section, equation, table, and figure numbers without the prefix S refer to the main text, and citations refer to its reference list.

\section{Proofs and Technical Results}
\label{app:theory}

\subsection{Notation}
\label{app:notation}
Table~\ref{tab:notation} lists the quantities of the local theory of Section~\ref{sec:theory} and whether each is observed.

\begin{table}[htbp]
\centering
\caption{Notation of the local theory and the status of each quantity.}
\label{tab:notation}
\small
\begin{tabular}{p{0.2\linewidth}p{0.55\linewidth}p{0.15\linewidth}}
\toprule
Symbol & Meaning & Status\\
\midrule
$g_i(t)$, $q_i(t)$ & Signed monitored gradient and gradient allocation of probe $i$ & Observed\\
$\Dden(t)$, $\Dw(t)$ & Pooled and within-class allocation dispersion, Eqs.~\eqref{eq:dden}--\eqref{eq:within-between} & Observed\\
$A(t)$, $B(t)$ & Remaining loss demand addressable by shared and by residual support & Latent\\
$\lambda(t)$ & Residual weight $B/(A+B)$ & Latent\\
$u(t)$, $v_i(t)$, $e_i(t)$ & Shared allocation shape, residual shape of probe $i$, approximation error & Latent\\
$\sigma_v(t)$, $\eta(t)$ & Cross-probe spread of residual shapes and of approximation errors & Latent\\
$R_\tr$, $R_\val$, $G$ & Training and validation cross-entropy and their difference & Observed\\
$\gamma(t)$, $\hat f_i(t)$ & Logit scale and scale-normalized logits & Defined from logits\\
$\Gamma^{\mathrm{dir}}$, $r^{\mathrm{dir}}$ & Directional rate of $G$ and directional training-loss reduction & Defined from logits\\
$\tau_u$, $\tau_v$ & Held-out transfer efficiencies of shared and residual support & Latent\\
\bottomrule
\end{tabular}
\end{table}

\subsection{Normalization Stability}
\label{app:norm-stability}
Let $\mathcal N(x)=x/\|x\|_1$ for $x\in\R^C$ with $x_c\ge0$ and $\|x\|_1>0$. If $\|x\|_1\ge m$ and $\|y\|_1\ge m$, then
\begin{equation}
\|\mathcal N(x)-\mathcal N(y)\|_2\le\frac{2}{m}\|x-y\|_1 ,
\end{equation}
because $\mathcal N(x)-\mathcal N(y)=(x-y)/\|x\|_1+y\,(1/\|x\|_1-1/\|y\|_1)$, the first term has $\ell_2$ norm at most $\|x-y\|_1/m$, and the second has norm at most $\|y\|_2\,|\|y\|_1-\|x\|_1|/(\|x\|_1\|y\|_1)\le\|x-y\|_1/m$. The implementation uses $\widehat{\mathcal N}(x)=x/(\|x\|_1+\varepsilon)$; since $\widehat{\mathcal N}(x)-\mathcal N(x)=-\mathcal N(x)\,\varepsilon/(\|x\|_1+\varepsilon)$ and $\|\mathcal N(x)\|_2\le1$,
\begin{equation}
\|\widehat{\mathcal N}(x)-\mathcal N(x)\|_2\le\frac{\varepsilon}{\|x\|_1}.
\end{equation}
Away from near-zero monitored mass, the logged allocation therefore differs from the exact one by a negligible amount. For a probe with $S_i\le\varepsilon$, the stabilized vector is far from the simplex; such probes are excluded as invalid.

\subsection{Class Decomposition of the Pooled Dispersion}
\label{app:class-decomposition}
With $I_y$, $N_y$, $\pi_y$, $\bar q_y$, and $D_y$ as in Section~\ref{sec:measurement}, write $q_i-\bar q=(q_i-\bar q_y)+(\bar q_y-\bar q)$ for $i\in I_y$. Summing the squared norm over $i\in I_y$, the cross term vanishes because $\sum_{i\in I_y}(q_i-\bar q_y)=0$, so $\sum_{i\in I_y}\|q_i-\bar q\|_2^2=N_yD_y+N_y\|\bar q_y-\bar q\|_2^2$. Summing over classes and dividing by $N$ gives
\begin{equation}
\Dden=\sum_y\pi_yD_y+\sum_y\pi_y\|\bar q_y-\bar q\|_2^2=\Dw+\Db .
\end{equation}
The weights must be the probe counts of the same probes that enter $\Dden$; with equal class weights $1/|\mathcal Y|$ in place of $\pi_y$, the identity does not hold in general for unbalanced probe classes. If two class-conditional signed modes differ only in sign and have the same allocation, their between-class contribution is zero.

\subsection{Sign Invariance of the Allocation}
\label{app:sign}
Figure~\ref{fig:sign-jitter} illustrates the sign factor of Section~\ref{sec:absolute-gradients}, and the following statement makes its removal precise.

\begin{proposition}[Sign invariance and stability of the allocation]
\label{prop:sign}
Let $a\in\R^C$ with $a_c\ge0$ and $\|a\|_1>0$, let $s,s'\in\{-1,+1\}^C$, and let $\odot$ denote the entrywise product. The allocations of $s\odot a$ and $s'\odot a$ coincide, $|s\odot a|/\|s\odot a\|_1=a/\|a\|_1$, whereas the signed normalized vectors $r_s=(s\odot a)/\|a\|_1$ satisfy
\begin{equation}
\|r_s-r_{s'}\|_2=\frac{2}{\|a\|_1}\Bigl(\sum_{c:\,s_c\ne s'_c}a_c^2\Bigr)^{1/2}.
\end{equation}
If the magnitude envelope changes from $a$ to $a+\delta$ with $a+\delta\ge0$ and both $\ell_1$ masses are at least $m>0$, then $\bigl\|(a+\delta)/\|a+\delta\|_1-a/\|a\|_1\bigr\|_2\le 2\|\delta\|_1/m$.
\end{proposition}
\begin{proof}
The first identity holds because $|s_c|=1$. For the signed vectors, $r_s-r_{s'}=((s-s')\odot a)/\|a\|_1$, and $s_c-s'_c$ is $0$ where the signs agree and $\pm2$ where they differ. The last bound is the normalization stability bound of Section~\ref{app:norm-stability}.
\end{proof}

\begin{figure}[htbp]
\centering
\includegraphics[width=\linewidth]{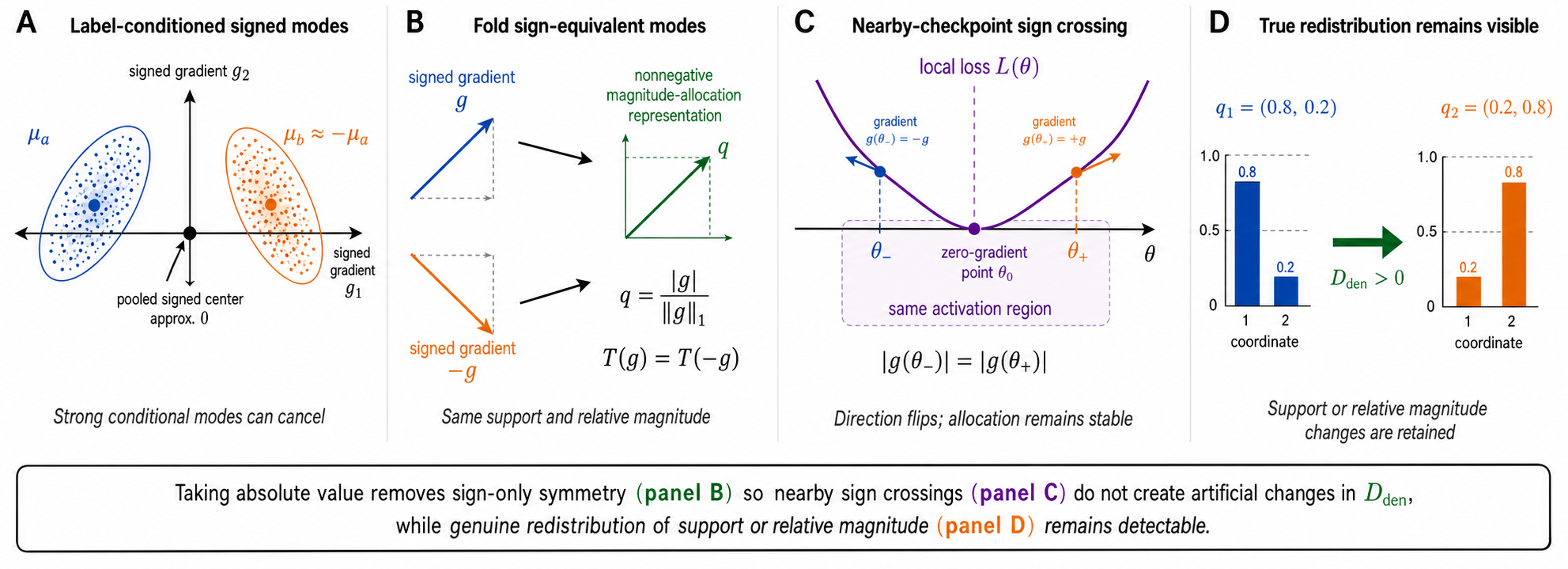}
\caption{Sign invariance of the allocation (schematic). A: two label-conditioned modes with nearly opposite centres in signed-gradient space. B: sign-equivalent gradients map to the same allocation $q=|g|/\|g\|_1$. C: a schematic sign reversal between checkpoints with the normalized magnitude allocation held fixed. D: a genuine redistribution across coordinates remains visible after the map.}
\label{fig:sign-jitter}
\end{figure}

\subsection{Finite-Probe Estimation}
\label{app:finite-probe}
Let $Q\in\DeltaC$ be the allocation of a probe drawn from the training population at a fixed model state, $\mu=\E Q$, and $D^\star=\E\|Q-\mu\|_2^2$. For $N\ge2$ independent draws, $D_N=N^{-1}\sum_i\|Q_i-\bar Q\|_2^2$ satisfies $\E[D_N]=\frac{N-1}{N}D^\star$, and
\begin{equation}
\widetilde D_N=\frac{N}{N-1}D_N=\frac{1}{2N(N-1)}\sum_{i\ne j}\|Q_i-Q_j\|_2^2
\end{equation}
is unbiased because $\E\|Q_i-Q_j\|_2^2=2D^\star$ for $i\ne j$. For simplex vectors the kernel $\|Q_i-Q_j\|_2^2/2$ lies in $[0,1]$; replacing one probe changes $\widetilde D_N$ by at most $2/N$, and the bounded-difference inequality gives $\Pr(|\widetilde D_N-D^\star|>\vartheta)\le2\exp(-N\vartheta^2/2)$. This bound holds at a fixed model state for independent sampling. The experiments use one fixed, stratified probe set along each trajectory, so the bound gives the size of the sampling error at a fixed model state; the probe-count analysis in Section~\ref{app:sensitivity-probes} examines the dependence on $N$ directly.

\subsection{Batch Coherence}
\label{app:batch-coherence}
\begin{proposition}[Batch coherence suppresses zero-mean residuals]
\label{prop:batch-coherence}
Let $\mathsf{Z}_j=s+\xi_j$ for $j=1,\ldots,m$, where the residuals $\xi_j$ are independent with $\E[\xi_j]=0$ and $\E[\xi_j\xi_j^\top]=\Sigma$. Then $\bar{\mathsf{Z}}=s+m^{-1}\sum_j\xi_j$ and $\E\|m^{-1}\sum_j\xi_j\|_2^2=\tr(\Sigma)/m$. If in addition $\E\exp(a^\top\xi_j)\le\exp(a^\top\Sigma a/2)$ for all $a\in\R^C$ and all $j$, then for every $\delta\in(0,1)$, with probability at least $1-\delta$,
\begin{equation}
\Bigl\|\frac1m\sum_{j=1}^m\xi_j\Bigr\|_2\le\sqrt{3}\sqrt{\frac{\tr(\Sigma)+\|\Sigma\|_{\mathrm{op}}\log(1/\delta)}{m}} .
\end{equation}
\end{proposition}
\begin{proof}
Linearity gives the mean, and independence with zero means removes the cross terms. For $\bar\xi=m^{-1}\sum_j\xi_j$, independence gives $\E e^{a^\top \bar\xi}\le e^{a^\top\Sigma a/(2m)}$. For $0<\chi<m/\|\Sigma\|_{\mathrm{op}}$, averaging this bound over an auxiliary standard Gaussian direction gives $\E e^{\chi\|\bar\xi\|_2^2/2}\le\det(I-\chi\Sigma/m)^{-1/2}$. Expanding the log-determinant and applying the exponential Markov inequality yields $\|\bar\xi\|_2^2\le[\tr(\Sigma)+2\sqrt{\tr(\Sigma^2)x}+2\|\Sigma\|_{\mathrm{op}}x]/m$ with probability at least $1-e^{-x}$. Set $x=\log(1/\delta)$ and use $\tr(\Sigma^2)\le\|\Sigma\|_{\mathrm{op}}\tr(\Sigma)$ and $2\sqrt{ab}\le a+b$.
\end{proof}
The decomposition into $s$ and $\xi$ is the arithmetic mean and residual of signed per-example gradients.

\subsection{Discrete-Time Demand Dynamics}
\label{app:local-dynamics}
With step size $\rho>0$, suppose $A_{t+1}=(1-\rho\kappa_u)A_t$ and $B_{t+1}=(1-\rho\kappa_v)B_t$ with $0<\rho\kappa_v<\rho\kappa_u<1$. Then $B_t/A_t=(B_0/A_0)\bigl((1-\rho\kappa_v)/(1-\rho\kappa_u)\bigr)^t$ increases with $t$, and so does $\lambda_t=B_t/(A_t+B_t)$. In continuous time, $\frac{\od}{\od t}\log(B/A)=\kappa_u-\kappa_v>0$, which gives Eq.~\eqref{eq:lambda-dynamics}.

\subsection{Proofs of the Local Growth Results}

\subsubsection{Allocation Dispersion}
\label{app:proof-observable-dispersion}
\begin{proof}[Proof of Theorem~\ref{thm:observable-dispersion}]
Averaging Eq.~\eqref{eq:latent-mixture} over probes gives $\bar q=(1-\lambda)u+\lambda\bar v+\bar e$, hence $q_i-\bar q=\lambda(v_i-\bar v)+(e_i-\bar e)$. On $(\R^C)^N$ with the norm $\|(w_i)_i\|_N=(N^{-1}\sum_i\|w_i\|_2^2)^{1/2}$, the three centred terms have norms $\sqrt{\Dden}$, $\lambda\sigma_v$, and $\eta$. The triangle inequality gives $|\sqrt{\Dden}-\lambda\sigma_v|\le\eta$; squaring, and using $\sqrt{\Dden}\ge0$ for the lower bound, gives the bounds. The two-time condition implies $\sqrt{\Dden(t_2)}\ge\lambda\sigma_v(t_2)-\eta(t_2)>\lambda\sigma_v(t_1)+\eta(t_1)\ge\sqrt{\Dden(t_1)}$.
\end{proof}

\subsubsection{Direction--Scale Decomposition}
\label{app:proof-direction-scale}
\begin{proof}[Proof of Lemma~\ref{lem:direction-scale}]
The product rule gives $\dot f_i=\dot\gamma\hat f_i+\gamma\dot{\hat f}_i$, and $\nabla_f\ell(f,y)=\mathrm{softmax}(f)-e_y$. Hence $\frac{\od}{\od t}\ell(f_i,y_i)=\gamma(p_i-e_{y_i})^\top\dot{\hat f}_i+\dot\gamma(p_i-e_{y_i})^\top\hat f_i$, and $(p_i-e_{y_i})^\top\hat f_i=-\omega_i$. Averaging over $\mathcal E_\bullet$ gives Eq.~\eqref{eq:direction-scale}; subtracting the training identity from the validation identity gives Eq.~\eqref{eq:gap-direction-scale}. Multiplying $f_i$ by $\gamma>0$ preserves the order of its coordinates and hence its maximizer.
\end{proof}
The identity holds for every positive differentiable $\gamma$; the choice of $\gamma$ only fixes which part of the logit motion is called uniform growth. Figure~\ref{fig:direction-scale} illustrates the two channels for a three-class task.
\subsubsection{Directional Transfer}
\label{app:proof-directional-transfer}
\begin{proof}[Proof of Proposition~\ref{prop:directional-transfer}]
Let $\zeta=-\dot\Rdir_\val-[(1-\lambda)\tau_u+\lambda\tau_v]r^{\mathrm{dir}}$, so $|\zeta|\le\epsilon$. Then $\Gamma^{\mathrm{dir}}=\dot\Rdir_\val+r^{\mathrm{dir}}=r^{\mathrm{dir}}\Pi(\lambda)-\zeta$, which gives the first bound and the lower bound. Differentiating $\Pi$ at fixed efficiencies gives $\tau_u-\tau_v$. With $\epsilon=0$ and $r^{\mathrm{dir}}>0$, $\Gamma^{\mathrm{dir}}>0$ is equivalent to $\Pi>0$, which is Eq.~\eqref{eq:effective-transfer-condition}. Adding the scale term of Eq.~\eqref{eq:gap-direction-scale} gives Eq.~\eqref{eq:full-gap-bound}.
\end{proof}

\subsubsection{Coupled Local Growth}
\label{app:proof-combined}
\begin{theorem}[Coupled growth of dispersion and directional separation]
\label{thm:combined}
On an interval $J$ where $A,B>0$, the logits and $\gamma$ are continuously differentiable, and $r^{\mathrm{dir}}>0$, suppose that (i) Assumption~\ref{ass:rates} holds; (ii) Assumption~\ref{ass:mixture} holds with $e_i=0$, $\sigma_v>0$ continuously differentiable, and $\dot\sigma_v/\sigma_v>-(1-\lambda)(\kappa_u-\kappa_v)$; (iii) Assumption~\ref{ass:transfer} holds with $\epsilon=0$; and (iv) $(1-\lambda)\tau_u+\lambda\tau_v<1$. Then $\dot\lambda>0$, $\dot\Dden>0$, $\Gamma^{\mathrm{dir}}>0$, and $\od\Pi/\od\lambda=\tau_u-\tau_v>0$ on $J$. If in addition $\dot\gamma(\Omega_\tr-\Omega_\val)\ge0$, then $\dot G>0$.
\end{theorem}
\begin{proof}[Proof of Theorem~\ref{thm:combined}]
Positivity of $A$ and $B$ gives $0<\lambda<1$, and Eq.~\eqref{eq:lambda-dynamics} gives $\dot\lambda>0$. By Theorem~\ref{thm:observable-dispersion}, $\Dden=\lambda^2\sigma_v^2>0$, and Eq.~\eqref{eq:relative-growth} with condition (ii) gives $\dot\Dden>0$. Conditions (iii) and (iv) give $\Pi>0$ and hence $\Gamma^{\mathrm{dir}}=r^{\mathrm{dir}}\Pi>0$ by Proposition~\ref{prop:directional-transfer}. Under the additional sign condition, Eq.~\eqref{eq:gap-direction-scale} gives $\dot G\ge\Gamma^{\mathrm{dir}}>0$.
\end{proof}

\subsection{First-Order Surrogate-Risk Changes}
\label{app:first-order-risk}
For split $S\in\{\tr,\val\}$ and step $\delta\theta_t=\theta_{t+1}-\theta_t$, if $R_S$ has Hessian operator norm at most $H_S$ on the segment between $\theta_t$ and $\theta_{t+1}$, Taylor's theorem gives $R_S(\theta_{t+1})=R_S(\theta_t)+\nabla R_S(\theta_t)^\top\delta\theta_t+\mathcal R_{S,t}$ with $|\mathcal R_{S,t}|\le H_S\|\delta\theta_t\|_2^2/2$. A first-order transfer relation between training and validation reductions is therefore meaningful for small steps. In Assumption~\ref{ass:transfer}, $\epsilon(t)$ collects such remainders together with minibatch noise, preconditioning, interactions not represented by the two-source allocation, and temporal variation of the efficiencies. It does not collect uniform growth of the logits, which Lemma~\ref{lem:direction-scale} separates exactly.

\subsection{Finite-Sample Accuracy Changes}
\label{app:boundary}
The margin conditions of Proposition~\ref{prop:accuracy-crossings} are: (i) $\Delta m_i\ge0$ for every training example; (ii) $\Delta m_i\ge a$ for every training example with $-a<m_i^0\le0$; and (iii) $\Delta m_i\le b$ for every validation example, where $B_S(h)=n_S^{-1}\#\{i\in\mathcal E_S:-h<m_i^0\le0\}$.
\begin{proof}[Proof of Proposition~\ref{prop:accuracy-crossings}]
Let $c_i^r\in\{0,1\}$ indicate whether example $i$ is classified correctly at checkpoint $t_r$, with the same tie-breaking rule at both checkpoints. Checking the four possibilities for $(c_i^0,c_i^1)$ gives
\begin{equation}
c_i^1-c_i^0=\mathbf 1\{c_i^0=0,c_i^1=1\}-\mathbf 1\{c_i^0=1,c_i^1=0\}.
\end{equation}
Summing over a fixed evaluation set and dividing by its positive size yields $\mathrm{Acc}_S(t_1)-\mathrm{Acc}_S(t_0)=(C_S^+-C_S^-)/n_S$. Subtracting the validation identity from the training identity proves Eq.~\eqref{eq:accuracy-crossings}. This identity concerns endpoint correctness and permits arbitrary intermediate motion, including repeated boundary crossings.

For the margin bound, the no-tie assumption makes correctness equivalent to $m_i^r>0$. Condition (i) ensures that an initially correct training example remains correct, so $C_\tr^-=0$. Every training example with $-a<m_i^0\le0$ satisfies $m_i^1=m_i^0+\Delta m_i\ge m_i^0+a>0$ by condition (ii), and hence changes from incorrect to correct. Thus $C_\tr^+\ge n_\tr B_\tr(a)$. If a validation example changes from incorrect to correct, then $m_i^0\le0$ and $m_i^0+\Delta m_i>0$. By condition (iii), $m_i^0+b\ge m_i^0+\Delta m_i>0$, so $-b<m_i^0\le0$. Therefore $C_\val^+\le n_\val B_\val(b)$; any validation regressions only decrease the net validation improvement. Substitution in Eq.~\eqref{eq:accuracy-crossings} gives
\begin{equation}
\Delta\GapAcc\ge B_\tr(a)-\frac{C_\val^+}{n_\val}\ge B_\tr(a)-B_\val(b),
\end{equation}
which proves Eq.~\eqref{eq:accuracy-margin-bound} and its strict-positivity consequence.
\end{proof}

The counting identity applies with any fixed tie-breaking rule. The no-tie assumption is used only for the margin bound, so that the sign of a margin determines correctness. Positive uniform rescaling of the logits preserves every predicted label and all the counts in Eq.~\eqref{eq:accuracy-crossings}. This explains the different roles of directional margin motion and uniform scale growth in Figure~\ref{fig:endpoint-lemma}.

\begin{figure}[t]
\centering
\includegraphics[width=\linewidth]{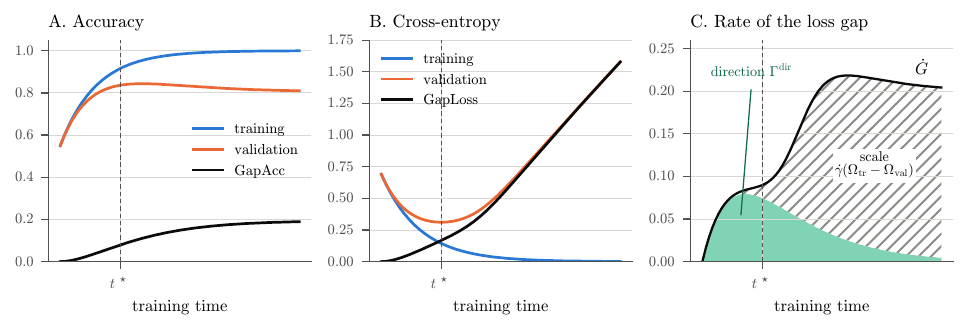}
\caption{Trajectory-level reading of Lemma~\ref{lem:direction-scale} (schematic curves, not measurements). A: accuracies and $\GapAcc$ depend on the normalized logits only, so they respond to the direction channel alone. B: training and validation cross-entropy; $t^\star$ marks the minimum of validation cross-entropy. C: the rate $\dot G$ of the cross-entropy gap, Eq.~\eqref{eq:gap-direction-scale}, split into the direction channel $\Gamma^{\mathrm{dir}}$ (green) and the scale channel $\dot\gamma(\Omega_\tr-\Omega_\val)$ (hatched). In the illustrated regime, net decision-boundary crossings subside while the scale channel stays positive, so the accuracy gap levels off and the cross-entropy gap continues to grow. Section~\ref{sec:results-resmlp-loss} examines this pattern in measured trajectories.}
\label{fig:endpoint-lemma}
\end{figure}

\clearpage
\section{Experimental Protocol Details}
\label{app:repro}

\subsection{Run Protocol and Initialization}
\label{app:run-protocol}
Each observational condition has five independent runs; the intervention uses the matrices of Section~\ref{sec:results-do}. Within a backbone all runs start from the same pretrained checkpoint. The seed determines the initialization of the newly added classification head (and, for Qwen3-1.7B, of the LoRA adapters) and the order of the training data; in the hierarchy it also draws the synthetic examples and their split. Dropout is fixed at 0 in all experiments. The probe branch is observational: it runs in evaluation mode between optimizer steps, clears its gradients, and changes neither the parameters nor the training schedule. Experiments ran on NVIDIA GeForce RTX 4090 and RTX 3090 GPUs.

\subsection{Hierarchy Construction}
\label{app:hierarchy-construction}
The five feature blocks of Section~\ref{sec:design-hierarchy} have 16, 24, 32, 48, and 2048 coordinates. Their class prototypes and large-, medium-, and small-community prototypes are shared at the corresponding levels; the communities cover 60\%, 40\%, and 20\% of the pool per community, the last community of each level taking the remaining examples. The feature noise $\nu$ is Gaussian with standard deviation $0.01$, and there is no label noise. Per-example level weights $w=(w_g,w_l,w_m,w_s,w_p)$ are Dirichlet draws with concentration $0.6$ around the target mixture of Eq.~\eqref{eq:target-mixture}, and a quarter of the examples receive weight on a single level.

A pool of 2048 examples in four classes is split by class into about 1843 training and 205 validation examples, and evaluation uses the entire validation portion. A fixed random arrangement with random signs and a scale factor of 4, drawn once from a fixed seed and shared by all runs, repeats each of the 2168 feature coordinates about 69 times and maps each feature vector to a $3\times224\times224$ input of a pretrained ResMLP-24 \citep{touvron2021resmlp}. Because a private coordinate belongs to one example, private support cannot transfer to the validation split by construction; the experiment tests how the readouts and the gap respond to the amount of such support.

\subsection{Monitoring and Evaluation}
\label{sec:design-monitor}
\label{app:monitor}
The monitoring scheme is the shared-output weight-gradient scheme: in each of the final $k$ blocks, the monitored coordinates are the trainable parameters of the projections that write back into the residual stream (Table~\ref{tab:monitors}). These are the attention-output and feed-forward-output projections of the transformer backbones (for Qwen3-1.7B, the trainable LoRA parameters attached to them), the token-mixing and channel-MLP output projections of ResMLP-24, and the last convolution and normalization of each ResNet basic block. Projection parameters include weights and biases where present. The classification head is never monitored. The intervention acts on, and takes its native probes from, the coordinate set specified in Section~\ref{app:do-native}.

\paragraph{Monitor depth}
The depth $k$ counts the monitored blocks backward from the output. For each language backbone, $k$ was calibrated once on BoolQ before the main study, in short one-epoch runs with two seeds disjoint from the main runs, with the gap measured on the fixed 10\% held out from the official training split. Among $k\in\{2,4,6,8,10\}$ we took the depth with the strongest readout--gap association and, between equally strong depths, the smaller one. This gave $k=4$ for RoBERTa-base, $k=6$ for DeBERTa-v3-base, and $k=2$ for Qwen3-1.7B, held fixed across all six tasks. The hierarchy fixes $k=10$. For the vision extension, the same rule selected $k=2$ from $k\in\{2,4,6,8\}$ on STL-10, and $k$ was then held fixed for CIFAR-10, CIFAR-100, and DTD. Depth sensitivity, including the language runs without the calibration task, is reported in Section~\ref{app:sensitivity-depth-nlp}.

\begin{table}[htbp]
\centering
\caption{Fixed monitoring depths of the main protocol, using the shared-output weight-gradient scheme.}
\label{tab:monitors}
\small
\begin{tabular}{@{}llp{0.6\linewidth}@{}}
\toprule
Backbone & $k$ & Monitored parameters in each of the final $k$ blocks\\
\midrule
RoBERTa-base & 4 & attention-output and feed-forward-output dense layers\\
DeBERTa-v3-base & 6 & attention-output and feed-forward-output dense layers\\
Qwen3-1.7B & 2 & LoRA parameters of \texttt{o\_proj} and \texttt{down\_proj}\\
ResMLP-24 & 10 & token-mixing and channel-MLP output projections\\
ResNet-18 & 2 & \texttt{conv2} and \texttt{bn2}\\
\bottomrule
\end{tabular}
\end{table}

\paragraph{Language training and evaluation}
The language runs fine-tune RoBERTa-base and DeBERTa-v3-base fully in FP32, and Qwen3-1.7B with a frozen backbone and rank-32 LoRA adapters on all attention and MLP projections in BF16, each with a trained classification head, for 10 epochs. The main runs use the full official training split and evaluate validation accuracy and loss on the full official validation split. Training accuracy is evaluated on a fixed class-proportional 10\% subset of the training split; those examples remain part of the optimization data. The 128 probes are drawn from the training split in proportion to its class frequencies (BoolQ 48/80, CoLA 38/90, MRPC 42/86, RTE 64/64, SST-2 57/71, WiC 64/64 for classes 0/1).

\paragraph{Training probes}
The probes serve only the gradient measurement and never enter the held-out evaluation. The language runs use 128 probes drawn in proportion to the training-class frequencies. The hierarchy uses 128 probes, 32 per class. The vision runs use a fixed number of probes per class, 10 on CIFAR-10 and STL-10 (100 probes), 2 on CIFAR-100 (200 probes), and 4 on DTD (188 probes); these training splits are class-balanced, so this is proportional sampling, and every class has at least two probes for its within-class dispersion. In all studies the within-class readout weights each class by its probe count (Eq.~\eqref{eq:within-between}).

\subsection{Hyperparameter Selection}
\label{sec:design-hpo}
\label{app:hpo}
The language learning rate is selected separately for each of the 18 backbone--dataset pairs from $\{1.0,1.5,2.0,3.0\}\times10^{-5}$, using only the official training split. Candidates are trained for one epoch with two seeds on 90\% of it, and the selection score is the mean endpoint accuracy of the two seeds on a fixed class-stratified 10\% held out from it. Dropout is fixed at 0, and batch size, gradient accumulation, weight decay, warmup, and precision are fixed per backbone (Section~\ref{app:config}). The main study then trains five runs per pair for 10 epochs on the full official training split, with the selected settings and seeds disjoint from the selection seeds, and evaluates accuracy and loss on the full official validation split. The selection uses accuracy only; the readouts do not enter it.

For CIFAR-10, CIFAR-100, and STL-10, the vision learning rate is selected from $\{1,2,3,5\}\times10^{-4}$ over five epochs of the 40-epoch schedule on a class-stratified $90\%/10\%$ split of the official training data; DTD uses its official training and validation partitions (split 1). The main vision runs retrain with the adopted settings on the full official training split (DTD: training plus validation) and evaluate on the full official test split. The hierarchy and the intervention use specified settings without search.

\subsection{Vision Extension Protocol}
\label{app:vision-protocol}
The vision extension fine-tunes ImageNet-pretrained ResNet-18 on CIFAR-10, CIFAR-100, STL-10, and DTD for 40 epochs, with five runs per dataset. Training, training-side evaluation, gradient probes, and test evaluation use the same deterministic preprocessing: resize to 256 pixels, centre crop to $224\times224$, and ImageNet normalization, without random augmentation. Training accuracy and loss use a fixed class-stratified $10\%$ training subset for CIFAR-10, CIFAR-100, and STL-10, and the full combined training set for DTD; held-out accuracy and loss use the full official test split. DTD split 1 has 47 classes and 1,880 images in each official partition, so its final training set contains 3,760 images and its test set 1,880 \citep{cimpoi2014describing}. The monitor reads the last convolution and normalization parameters of the final two basic blocks ($k=2$), and both $\Dden$ and $\Dw$ include every class. The selected learning rates are $3\times10^{-4}$ for CIFAR-10, CIFAR-100, and DTD and $10^{-4}$ for STL-10. All four datasets use AdamW, weight decay $.01$, training/evaluation batch sizes 64/256, no warmup, linear decay to zero, and BF16 autocast.

\subsection{Statistical Implementation}
\label{app:statistics-details}
For the observational studies, correlations are Pearson correlations computed within one run across its checkpoints and then averaged with equal weight over runs; we report means, between-run standard deviations, and the number of runs with a positive correlation. Checkpoints of one run are not treated as independent observations. As a second view of the levels, both series are smoothed separately by LOWESS (local-linear, tricube weights, fraction $.28$, no robust iterations) before correlating. The linear-time control correlates the residuals of the two series after each is regressed on an intercept and the epoch; the quadratic and logarithmic controls replace the linear epoch trend with the corresponding trend. The change correlations of the language runs pair each checkpoint with the checkpoint nearest one epoch later (realized intervals $\NumNlpIntervalMin$ to $\NumNlpIntervalMax$ epochs). Positive counts treat runs as independent replicates; for example, 90 positive runs out of 90 bound the positive fraction below by $\NumCiNinetyLower$ (exact 95\% interval). The intervention's dose-level correlations and paired aggregation are defined in Section~\ref{app:do-definitions}.

\subsection{Probe Computation}
\label{app:probe-computation}
At each checkpoint the model is switched to evaluation mode, so dropout is inactive. For every probe separately, the gradients are cleared, the probe is passed through the model alone, its logits are converted to FP32 before the cross-entropy loss is evaluated, and the loss is backpropagated. The gradients of the monitored parameters are flattened in a fixed order, their absolute values and $\ell_1$ mass $S_i$ are formed, and the stabilized allocation $\tilde q_i$ of Section~\ref{sec:measurement} is accumulated into running sums from which $\Dden$, the class dispersions $D_y$, and $\Db$ follow (Algorithm~\ref{alg:tracking}). Gradients are cleared after each probe, parameters are not updated, and training mode is restored before the next optimizer step. RoBERTa-base and DeBERTa-v3-base hold FP32 parameters, so their probe gradients are FP32. In the observational hierarchy runs, ResMLP-24 also holds FP32 parameters and BF16 autocast applies only to the training forward pass, so its probe gradients are FP32 as well. For Qwen3-1.7B the frozen backbone is held in BF16, so the gradients of the monitored LoRA parameters are computed through BF16 activations.

\begin{algorithm}[htbp]
\caption{Fixed-probe gradient-allocation tracking}
\label{alg:tracking}
\begin{algorithmic}[1]
\State Fix probes $I$, monitored coordinates $\mathcal C$, logging schedule, and $\varepsilon=10^{-12}$.
\For{each optimizer step $t$}
  \State Update the model with the ordinary optimizer and schedule.
  \If{$t$ is a logging step}
    \State Switch to evaluation mode.
    \For{each probe $i\in I$}
      \State Clear gradients; backpropagate $\ell(z_i;\theta_t)$ alone; read $g_i=\nabla_{\theta_{\mathcal C}}\ell$.
      \State $S_i\gets\|g_i\|_1$; if $S_i$ is finite and $S_i>\varepsilon$, accumulate $\tilde q_i=|g_i|/(S_i+\varepsilon)$ for $\Dden$ and $D_{y_i}$.
    \EndFor
    \State Record $\Dden$, $D_y$, $\Dw$, $\Db$, valid-probe counts, and training and validation accuracy and loss.
    \State Restore training mode.
  \EndIf
\EndFor
\end{algorithmic}
\end{algorithm}

\subsection{Code}
\label{app:package}
The reference implementation covers the language and hierarchy studies. It provides the shared-output weight-gradient monitors of the language and hierarchy backbones in Table~\ref{tab:monitors}, the normalized readouts of Section~\ref{sec:measurement} and their unnormalized companions (tested against a brute-force computation), the five-level generator with the single-level masked inputs, the training loop on the full training data, and the probe and training-evaluation sampling of Section~\ref{app:monitor}. For the intervention it provides the update transformation of Eq.~\eqref{eq:do-transform} on the coordinates of Section~\ref{app:do-native}, the anchor baselines, the dose branches, and the paired contrasts. Run settings are passed to it as listed in Section~\ref{app:config}. It also records the signed readouts of Section~\ref{app:sensitivity-signed}.

\clearpage
\section{Detailed Results}
\label{app:results}

\subsection{Language Runs}
\label{app:results-nlp}
Figure~\ref{fig:nlp-distribution} summarizes the run-level correlations of the language runs, including the time controls and one-epoch changes detailed here.

\begin{figure}[htbp]
\centering
\includegraphics[width=\linewidth]{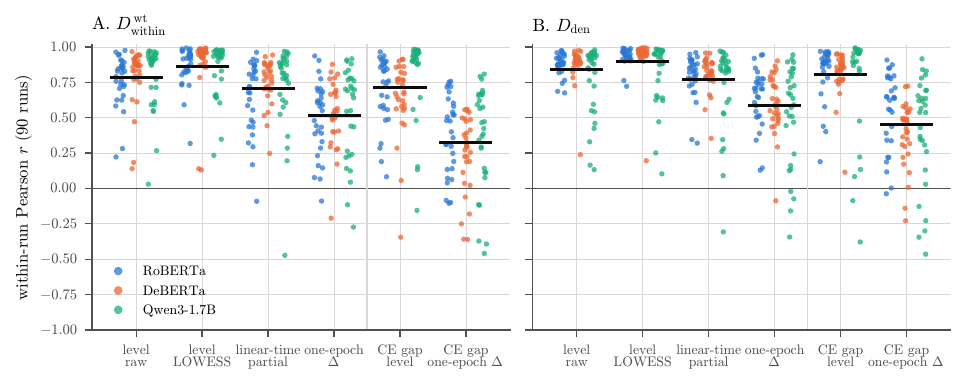}
\caption{Run-level correlations in the 90 language runs, for $\Dw$ (A) and $\Dden$ (B). Each dot is one run, coloured by backbone; black bars are means. The four left columns pair the readout with the accuracy gap and the two right columns with the cross-entropy gap (CE). ``One-epoch $\Delta$'' correlates changes over matched one-epoch intervals.}
\label{fig:nlp-distribution}
\end{figure}

 Table~\ref{tab:nlp-main} of the main paper gives the level, time-controlled, one-epoch, and cross-entropy associations of the two primary readouts by backbone.

\paragraph{Class-specific and overall level associations}
The class-specific readouts $D_0$ and $D_1$, each paired with the accuracy gap of its own class, are positive under both views in $\NumNlpDzeroJoint$ and $\NumNlpDoneJoint$ of 90 runs. Table~\ref{tab:nlp-class-summary} gives these readouts alongside the two primary readouts.

\begin{table}[htbp]
\centering
\caption{Language runs: class-specific and overall level associations, as mean raw $r$ / mean LOWESS $r$ / runs positive under both. $D_0$ and $D_1$ are paired with the accuracy gap of their own class; $\Dw$ and $\Dden$ with the overall gap.}
\label{tab:nlp-class-summary}
\small
\resizebox{\linewidth}{!}{%
\begin{tabular}{lcccc}
\toprule
Backbone & $D_0$ vs.\ class-0 gap & $D_1$ vs.\ class-1 gap & $D_{\mathrm{within}}^{\mathrm{wt}}$ vs.\ $\GapAcc$ & $D_{\mathrm{den}}$ vs.\ $\GapAcc$\\
\midrule
RoBERTa & .749 / .858 / 30/30 & .579 / .707 / 28/30 & .763 / .866 / 30/30 & .877 / .933 / 30/30\\
DeBERTa & .745 / .856 / 30/30 & .650 / .780 / 28/30 & .792 / .885 / 30/30 & .878 / .935 / 30/30\\
Qwen3-1.7B & .683 / .805 / 30/30 & .682 / .765 / 27/30 & .797 / .842 / 30/30 & .770 / .819 / 30/30\\
\midrule
All 90 runs & .726 / .840 / 90/90 & .637 / .751 / 83/90 & .784 / .864 / 90/90 & .841 / .895 / 90/90\\
\bottomrule
\end{tabular}
}
\end{table}

\paragraph{Between-class contribution}
The exact decomposition of Eq.~\eqref{eq:within-between} also separates the class-centre contribution $\Db$. Its mean raw correlation with the overall accuracy gap is $\NumNlpBetweenRaw$, with $\NumNlpBetweenJoint$/90 runs positive under both raw and LOWESS views. Averaging the per-run ratio $\Db/\Dden$ gives $\NumNlpBetweenRatioFirst\%$ at the first recorded checkpoint and $\NumNlpBetweenRatioLast\%$ at the final checkpoint. Class-centre separation thus contributes to the pooled readout, while the positive association of $\Dw$ shows that the relation also holds among probes of the same class.

\paragraph{Time trends}
Across the 90 language runs, the mean within-run correlation between epoch and the accuracy gap is $\NumNlpTimeGap$. Table~\ref{tab:nlp-main} removes a linear, quadratic, or logarithmic trend in the epoch from both series and correlates the residuals, by backbone. With a linear trend removed, the association remains positive in $\NumNlpDdenPartialPos$/90 runs for $\Dden$ (mean $\NumNlpDdenPartial$) and $\NumNlpWPartialPos$/90 for $\Dw$ (mean $\NumNlpWPartial$). With a quadratic or logarithmic trend removed, which also absorbs the saturating shape of both curves, the association of $\Dden$ remains positive in $\NumNlpDdenPartialQuadPos$/90 and $\NumNlpDdenPartialLogPos$/90 runs (means $\NumNlpDdenPartialQuad$ and $\NumNlpDdenPartialLog$), and that of $\Dw$ in $\NumNlpWPartialQuadPos$/90 and $\NumNlpWPartialLogPos$/90.

\paragraph{One-epoch changes}
The changes of the readouts and of the gap over matched one-epoch intervals are positively correlated in $\NumNlpDdenDiffPos$/90 runs for $\Dden$ (mean $\NumNlpDdenDiff$) and $\NumNlpWDiffPos$/90 for $\Dw$ (mean $\NumNlpWDiff$), and in $\NumNlpDdenDiffLowessPos$/90 and $\NumNlpWDiffLowessPos$/90 runs after LOWESS smoothing (Table~\ref{tab:nlp-diff}). The intervals between the paired checkpoints range from $\NumNlpIntervalMin$ to $\NumNlpIntervalMax$ epochs. For the cross-entropy gap, the one-epoch change correlation of $\Dden$ is positive in $\NumNlpLossDdenDiffPos$/90 runs (mean $\NumNlpLossDdenDiff$; Table~\ref{tab:nlp-main}).

\begin{table}[htbp]
\centering
\caption{Language runs: correlation of changes of each readout with changes of the accuracy gap over matched one-epoch intervals, raw or after LOWESS smoothing of both series. Means with the number of positive runs.}
\label{tab:nlp-diff}
\small
\begin{tabular}{lcccc}
\toprule
 & \multicolumn{2}{c}{$\Delta D_{\mathrm{within}}^{\mathrm{wt}}$ vs.\ $\Delta\GapAcc$} & \multicolumn{2}{c}{$\Delta D_{\mathrm{den}}$ vs.\ $\Delta\GapAcc$}\\
\cmidrule(lr){2-3}\cmidrule(lr){4-5}
Backbone & Raw & LOWESS & Raw & LOWESS\\
\midrule
RoBERTa & .495 (29/30) & .686 (29/30) & .644 (30/30) & .823 (30/30)\\
DeBERTa & .524 (29/30) & .785 (29/30) & .596 (29/30) & .834 (30/30)\\
Qwen3-1.7B & .534 (28/30) & .677 (28/30) & .522 (26/30) & .666 (27/30)\\
\midrule
All 90 runs & .518 (86/90) & .716 (86/90) & .587 (85/90) & .774 (87/90)\\
\bottomrule
\end{tabular}

\end{table}

\paragraph{Task-specific metrics}
Accuracy is the common endpoint across tasks. For CoLA and MRPC we also report the GLUE task metrics, Matthews correlation (MCC) and the F1 score of class 1. Both are computed at every checkpoint from the binary confusion matrix of each evaluation side, which follows exactly from the logged per-class accuracies and the class counts (classes 0/1): 322/721 validation and 253/603 training-evaluation examples for CoLA, and 129/279 and 119/248 for MRPC. At $\NumTaskMccUndefined$ CoLA checkpoints one side predicts a single class, and its MCC is set to 0. The task-metric gap is the training-side metric minus the validation metric, and its association with the readouts is computed as for the accuracy gap. $\Dw$ and $\Dden$ are positively associated with this gap under both views in $\NumTaskWJoint$/30 and $\NumTaskDdenJoint$/30 runs, respectively (Table~\ref{tab:nlp-task-metrics}).

\begin{table}[htbp]
\centering
\caption{CoLA (Matthews correlation) and MRPC (F1 of class 1): final training-side and validation values and their gap, averaged over five runs, and the mean within-run correlations of the readouts with the task-metric gap, with the runs positive under both views.}
\label{tab:nlp-task-metrics}
\small
\resizebox{\linewidth}{!}{%
\begin{tabular}{@{}llccccc@{}}
\toprule
 & & \multicolumn{3}{c}{Final task metric} & \multicolumn{2}{c}{Association with the task-metric gap}\\
\cmidrule(lr){3-5}\cmidrule(lr){6-7}
Task & Backbone & Train & Validation & Gap & $D_{\mathrm{within}}^{\mathrm{wt}}$: raw / LOWESS & $D_{\mathrm{den}}$: raw / LOWESS\\
\midrule
CoLA (MCC) & RoBERTa & .996 & .605 & .391 & .796 / .930 (5/5) & .836 / .942 (5/5)\\
 & DeBERTa & .996 & .622 & .374 & .805 / .925 (5/5) & .861 / .943 (5/5)\\
 & Qwen3-1.7B & .998 & .625 & .374 & .626 / .704 (5/5) & .685 / .786 (5/5)\\
\midrule
MRPC (F1) & RoBERTa & .998 & .914 & .085 & .454 / .656 (5/5) & .863 / .944 (5/5)\\
 & DeBERTa & 1.000 & .923 & .077 & .790 / .919 (5/5) & .890 / .964 (5/5)\\
 & Qwen3-1.7B & 1.000 & .907 & .093 & .918 / .950 (5/5) & .934 / .962 (5/5)\\
\bottomrule
\end{tabular}
}
\end{table}

Tables~\ref{tab:app-nlp-dataset-level} and~\ref{tab:app-nlp-dataset-temporal} report the dataset-level associations behind Table~\ref{tab:nlp-main} and Tables~\ref{tab:nlp-class-summary} and~\ref{tab:nlp-diff}; of the time controls, Table~\ref{tab:app-nlp-dataset-temporal} gives the linear one.

\FloatBarrier
\begingroup\footnotesize
\setlength{\tabcolsep}{3pt}\renewcommand{\arraystretch}{1.08}
\begin{longtable}{@{}llcccc@{}}
\caption{NLP associations by dataset. Each cell gives the mean raw / LOWESS within-run Pearson $r$ over five runs and, in parentheses, the number of runs positive under both. $D_0$ and $D_1$ are paired with the accuracy gap of their own class; $W=D_{\mathrm{within}}^{\mathrm{wt}}$ and $D_{\mathrm{den}}$ with the overall accuracy gap.}\label{tab:app-nlp-dataset-level}\\
\toprule
Backbone & Dataset & $D_0$ & $D_1$ & $W$ & $D_{\mathrm{den}}$\\
\midrule
\endfirsthead
\toprule
Backbone & Dataset & $D_0$ & $D_1$ & $W$ & $D_{\mathrm{den}}$\\
\midrule
\endhead
RoBERTa & BoolQ & $.763/.854$ (5/5) & $.736/.803$ (5/5) & $.859/.910$ (5/5) & $.912/.935$ (5/5)\\
RoBERTa & CoLA & $.799/.913$ (5/5) & $.598/.776$ (5/5) & $.794/.926$ (5/5) & $.830/.936$ (5/5)\\
RoBERTa & MRPC & $.787/.849$ (5/5) & $.189/.267$ (3/5) & $.457/.661$ (5/5) & $.864/.945$ (5/5)\\
RoBERTa & RTE & $.688/.797$ (5/5) & $.752/.859$ (5/5) & $.822/.900$ (5/5) & $.934/.963$ (5/5)\\
RoBERTa & SST-2 & $.656/.801$ (5/5) & $.455/.628$ (5/5) & $.742/.822$ (5/5) & $.797/.851$ (5/5)\\
RoBERTa & WiC & $.802/.934$ (5/5) & $.742/.912$ (5/5) & $.906/.976$ (5/5) & $.923/.969$ (5/5)\\
\midrule
DeBERTa & BoolQ & $.851/.939$ (5/5) & $.793/.939$ (5/5) & $.889/.971$ (5/5) & $.924/.972$ (5/5)\\
DeBERTa & CoLA & $.781/.919$ (5/5) & $.751/.884$ (5/5) & $.817/.928$ (5/5) & $.870/.944$ (5/5)\\
DeBERTa & MRPC & $.797/.871$ (5/5) & $.659/.773$ (5/5) & $.781/.914$ (5/5) & $.889/.963$ (5/5)\\
DeBERTa & RTE & $.683/.808$ (5/5) & $.703/.887$ (5/5) & $.808/.931$ (5/5) & $.871/.946$ (5/5)\\
DeBERTa & SST-2 & $.731/.834$ (5/5) & $.395/.475$ (4/5) & $.685/.753$ (5/5) & $.912/.967$ (5/5)\\
DeBERTa & WiC & $.630/.764$ (5/5) & $.601/.720$ (4/5) & $.771/.811$ (5/5) & $.799/.818$ (5/5)\\
\midrule
Qwen3-1.7B & BoolQ & $.865/.926$ (5/5) & $.838/.929$ (5/5) & $.942/.957$ (5/5) & $.942/.957$ (5/5)\\
Qwen3-1.7B & CoLA & $.657/.786$ (5/5) & $.464/.562$ (4/5) & $.628/.704$ (5/5) & $.687/.785$ (5/5)\\
Qwen3-1.7B & MRPC & $.794/.902$ (5/5) & $.847/.925$ (5/5) & $.917/.951$ (5/5) & $.932/.963$ (5/5)\\
Qwen3-1.7B & RTE & $.455/.623$ (5/5) & $.359/.401$ (3/5) & $.601/.672$ (5/5) & $.568/.641$ (5/5)\\
Qwen3-1.7B & SST-2 & $.584/.721$ (5/5) & $.663/.803$ (5/5) & $.747/.803$ (5/5) & $.550/.601$ (5/5)\\
Qwen3-1.7B & WiC & $.741/.873$ (5/5) & $.924/.968$ (5/5) & $.945/.963$ (5/5) & $.943/.964$ (5/5)\\
\bottomrule
\end{longtable}

\begin{longtable}{@{}llcccccc@{}}
\caption{Temporal and cross-entropy checks by dataset. Cells give the mean within-run Pearson $r$ over five runs and, in parentheses, the number of positive runs. Partial: residuals after a linear epoch fit. $\Delta$: changes paired at approximately one-epoch intervals. The last two columns pair $D_{\mathrm{den}}$ with the cross-entropy gap $\GapLoss$.}\label{tab:app-nlp-dataset-temporal}\\
\toprule
 & & \multicolumn{2}{c}{Partial, vs.\ $\GapAcc$} & \multicolumn{2}{c}{$\Delta$, vs.\ $\Delta\GapAcc$} & \multicolumn{2}{c}{$D_{\mathrm{den}}$ vs.\ $\GapLoss$}\\
\cmidrule(lr){3-4}\cmidrule(lr){5-6}\cmidrule(lr){7-8}
Backbone & Dataset & $W$ & $D_{\mathrm{den}}$ & $\Delta W$ & $\Delta D_{\mathrm{den}}$ & Level & $\Delta$\\
\midrule
\endfirsthead
\toprule
Backbone & Dataset & $W$ & $D_{\mathrm{den}}$ & $\Delta W$ & $\Delta D_{\mathrm{den}}$ & Level & $\Delta$\\
\midrule
\endhead
RoBERTa & BoolQ & .728 (5) & .865 (5) & .552 (5) & .731 (5) & .934 (5) & .463 (4)\\
RoBERTa & CoLA & .583 (5) & .642 (5) & .417 (5) & .502 (5) & .892 (5) & .742 (5)\\
RoBERTa & MRPC & .281 (4) & .700 (5) & .219 (4) & .499 (5) & .869 (5) & .559 (5)\\
RoBERTa & RTE & .652 (5) & .864 (5) & .438 (5) & .720 (5) & .927 (5) & .585 (5)\\
RoBERTa & SST-2 & .795 (5) & .832 (5) & .621 (5) & .676 (5) & .455 (5) & .174 (5)\\
RoBERTa & WiC & .840 (5) & .865 (5) & .722 (5) & .733 (5) & .876 (5) & .783 (5)\\
\midrule
DeBERTa & BoolQ & .852 (5) & .856 (5) & .723 (5) & .736 (5) & .834 (5) & .454 (5)\\
DeBERTa & CoLA & .727 (5) & .791 (5) & .404 (5) & .468 (5) & .854 (5) & .338 (4)\\
DeBERTa & MRPC & .709 (5) & .833 (5) & .472 (5) & .623 (5) & .841 (5) & .595 (5)\\
DeBERTa & RTE & .691 (5) & .745 (5) & .435 (4) & .505 (4) & .850 (5) & .407 (5)\\
DeBERTa & SST-2 & .729 (5) & .794 (5) & .457 (5) & .554 (5) & .784 (5) & .143 (4)\\
DeBERTa & WiC & .800 (5) & .823 (5) & .653 (5) & .691 (5) & .719 (5) & .355 (5)\\
\midrule
Qwen3-1.7B & BoolQ & .922 (5) & .918 (5) & .810 (5) & .781 (5) & .916 (5) & .553 (5)\\
Qwen3-1.7B & CoLA & .307 (4) & .300 (4) & .165 (4) & .114 (3) & .776 (5) & .192 (4)\\
Qwen3-1.7B & MRPC & .831 (5) & .878 (5) & .617 (5) & .728 (5) & .968 (5) & .667 (5)\\
Qwen3-1.7B & RTE & .552 (5) & .492 (5) & .316 (4) & .246 (4) & .984 (5) & .750 (5)\\
Qwen3-1.7B & SST-2 & .792 (5) & .800 (5) & .524 (5) & .536 (4) & .064 (3) & \ensuremath{-.207} (1)\\
Qwen3-1.7B & WiC & .915 (5) & .900 (5) & .770 (5) & .725 (5) & .910 (5) & .530 (5)\\
\bottomrule
\end{longtable}

\endgroup

\subsection{Controlled Hierarchy}
\label{app:results-resmlp}
Table~\ref{tab:resmlp-unnorm} reports the increments of the unnormalized companions.

\begin{table}[htbp]
\centering
\caption{Increments of the unnormalized companions from epoch 1 to epoch 25: mean / median over five runs. Every increment is positive.}
\label{tab:resmlp-unnorm}
\small
\begin{tabular}{ccc}
\toprule
Private share $p$ & $\Delta U_{\mathrm{within}}^{\mathrm{wt}}$: mean / median & $\Delta U_{\mathrm{den}}$: mean / median\\
\midrule
.3 & 151.500 / 57.530 & 156.173 / 60.294\\
.5 & 302.045 / 120.888 & 312.150 / 126.111\\
.7 & 639.187 / 514.815 & 677.525 / 539.137\\
\bottomrule
\end{tabular}

\end{table}

 Table~\ref{tab:app-resmlp-runs} gives the epoch-1 and epoch-25 values of every run, Table~\ref{tab:app-resmlp-turning-runs} the per-run validation-loss minima with the corresponding loss, accuracy, and accuracy-gap values, and Figure~\ref{fig:resmlp-runs} the run-level relation between the readouts and the gap.

\paragraph{Private masked-loss shares}
Figure~\ref{fig:resmlp-behaviour}B of the main paper shows, for each split, how the loss of the single-level masked inputs is distributed over the five levels, and Table~\ref{tab:resmlp-main} gives the change of the private share on each split. At epoch 1 every level carries one fifth. During training the private share rises more on the validation split than on the training split at every private share; the widening is largest at $p=.7$, where the training-side private share even falls. The shares describe losses under single-level inputs and are not additive components of the unmasked loss.

\paragraph{Cross-entropy gap and the validation-loss minimum}
\label{app:results-resmlp-loss}
The level of $\Dden$ is positively correlated with the cross-entropy gap in all 15 hierarchy runs (mean $\NumResDdenLossRaw$). After removing a linear trend, the mean correlation is $\NumResDdenLossPartial$ ($\NumResDdenLossPartialPos$ of 15 runs positive); $\Dw$ also has a positive level association and a negative mean partial correlation (Table~\ref{tab:resmlp-local}). The scale channel of Lemma~\ref{lem:direction-scale} leaves this part of the cross-entropy gap unconstrained by the readouts. To describe how loss and accuracy evolve together, we summarize each run at the epoch $t^\star$ of minimal validation cross-entropy and at epoch 25.

\begin{table}[htbp]
\centering
\caption{Time-controlled and cross-entropy associations in the controlled hierarchy. Partial: correlation after removing a linear epoch trend from both series. Level: raw correlation of the recorded values. Entries are five-run (or fifteen-run) means with the number of positive runs.}
\label{tab:resmlp-local}
\small
\begin{tabular}{llccc}
\toprule
 & & Accuracy gap & \multicolumn{2}{c}{Cross-entropy gap}\\
\cmidrule(lr){4-5}
$p$ & Readout & Partial $r$ & Level, raw $r$ & Partial $r$\\
\midrule
\multirow{2}{*}{.3} & $D_{\mathrm{within}}^{\mathrm{wt}}$ & .514 (5/5) & .511 (5/5) & \ensuremath{-.420} (0/5)\\
 & $D_{\mathrm{den}}$ & .593 (5/5) & .843 (5/5) & \ensuremath{-.259} (1/5)\\
\midrule
\multirow{2}{*}{.5} & $D_{\mathrm{within}}^{\mathrm{wt}}$ & .240 (3/5) & .535 (5/5) & \ensuremath{-.344} (1/5)\\
 & $D_{\mathrm{den}}$ & .343 (3/5) & .713 (5/5) & \ensuremath{-.218} (2/5)\\
\midrule
\multirow{2}{*}{.7} & $D_{\mathrm{within}}^{\mathrm{wt}}$ & .509 (5/5) & .793 (5/5) & \ensuremath{-.333} (1/5)\\
 & $D_{\mathrm{den}}$ & .719 (5/5) & .935 (5/5) & .042 (3/5)\\
\midrule
\multirow{2}{*}{All} & $D_{\mathrm{within}}^{\mathrm{wt}}$ & .421 (13/15) & .613 (15/15) & \ensuremath{-.366} (2/15)\\
 & $D_{\mathrm{den}}$ & .552 (13/15) & .830 (15/15) & \ensuremath{-.145} (6/15)\\
\bottomrule
\end{tabular}

\end{table}

\begin{table}[htbp]
\centering
\caption{Behaviour around the validation-loss minimum $t^\star$. Growth after $t^\star$: share of the epoch-1-to-25 increase of each gap that occurs after $t^\star$. Validation accuracy: mean at $t^\star$ and at epoch 25, with the number of runs in which it did not fall. Per-run loss, accuracy, and accuracy-gap values are in Table~\ref{tab:app-resmlp-turning-runs}.}
\label{tab:resmlp-turning}
\small
\begin{tabular}{ccccc}
\toprule
 & & \multicolumn{2}{c}{Growth after $t^\star$} & \\
\cmidrule(lr){3-4}
$p$ & $t^\star$ & $\GapLoss$ & $\GapAcc$ & $\ValAcc(t^\star)\to\ValAcc(25)$\\
\midrule
.3 & 8--11 & 80\% & 30\% & .766$\to$.815 (4/5)\\
.5 & 7--9 & 87\% & 64\% & .615$\to$.655 (3/5)\\
.7 & 6--9 & 95\% & 75\% & .379$\to$.457 (5/5)\\
All & 6--11 & 87\% & 56\% & .587$\to$.643 (12/15)\\
\bottomrule
\end{tabular}

\end{table}

The validation loss is minimal between epochs $\NumTurnMin$ and $\NumTurnMax$ (mean $\NumTurnMean$). After this point mean validation cross-entropy rises from $\NumTurnValLossStart$ to $\NumTurnValLossEnd$, whereas mean validation accuracy rises from $\NumTurnValAccStart$ to $\NumTurnValAccEnd$ and does not fall in $\NumTurnValAccNotLower$ of 15 runs (Table~\ref{tab:resmlp-turning}). On average $\NumTurnLossGrowth\%$ of the growth of the cross-entropy gap and $\NumTurnAccGrowth\%$ of the growth of the accuracy gap occur after $t^\star$. The rising validation loss is thus not accompanied by a decline in mean validation accuracy, consistent with Lemma~\ref{lem:direction-scale}, which allows cross-entropy to change through the scale channel without changing predictions.

\FloatBarrier
\begingroup\footnotesize
\setlength{\tabcolsep}{4pt}
\begin{longtable}{@{}cccccc@{}}
\caption{Individual ResMLP runs: epoch-1 and epoch-25 values of the normalized readouts ($D$) and their unnormalized companions ($U$). Rows with the same private share and run label refer to the same training run.}\label{tab:app-resmlp-runs}\\
\toprule
$p$ & Run & $D_{\mathrm{within}}^{\mathrm{wt}}$ & $D_{\mathrm{den}}$ & $U_{\mathrm{within}}^{\mathrm{wt}}$ & $U_{\mathrm{den}}$\\
\midrule
\endfirsthead
\toprule
$p$ & Run & $D_{\mathrm{within}}^{\mathrm{wt}}$ & $D_{\mathrm{den}}$ & $U_{\mathrm{within}}^{\mathrm{wt}}$ & $U_{\mathrm{den}}$\\
\midrule
\endhead
.3 & 1 & \(1.892\mathrm{e}{-7}\)$\to$\(2.355\mathrm{e}{-6}\) & \(2.498\mathrm{e}{-7}\)$\to$\(4.472\mathrm{e}{-6}\) & \(1.762\mathrm{e}{-5}\)$\to$\(4.231\mathrm{e}{2}\) & \(4.122\mathrm{e}{-5}\)$\to$\(4.325\mathrm{e}{2}\)\\
.3 & 2 & \(1.829\mathrm{e}{-7}\)$\to$\(2.362\mathrm{e}{-6}\) & \(3.402\mathrm{e}{-7}\)$\to$\(4.545\mathrm{e}{-6}\) & \(2.995\mathrm{e}{-5}\)$\to$\(4.627\mathrm{e}{1}\) & \(1.420\mathrm{e}{-4}\)$\to$\(4.864\mathrm{e}{1}\)\\
.3 & 3 & \(2.063\mathrm{e}{-7}\)$\to$\(1.871\mathrm{e}{-6}\) & \(3.169\mathrm{e}{-7}\)$\to$\(3.735\mathrm{e}{-6}\) & \(4.426\mathrm{e}{-5}\)$\to$\(5.573\mathrm{e}{0}\) & \(1.446\mathrm{e}{-4}\)$\to$\(5.774\mathrm{e}{0}\)\\
.3 & 4 & \(1.757\mathrm{e}{-7}\)$\to$\(2.369\mathrm{e}{-6}\) & \(3.113\mathrm{e}{-7}\)$\to$\(1.414\mathrm{e}{-5}\) & \(1.383\mathrm{e}{-5}\)$\to$\(2.250\mathrm{e}{2}\) & \(7.174\mathrm{e}{-5}\)$\to$\(2.337\mathrm{e}{2}\)\\
.3 & 5 & \(2.030\mathrm{e}{-7}\)$\to$\(1.022\mathrm{e}{-6}\) & \(3.350\mathrm{e}{-7}\)$\to$\(2.874\mathrm{e}{-6}\) & \(2.645\mathrm{e}{-5}\)$\to$\(5.753\mathrm{e}{1}\) & \(7.976\mathrm{e}{-5}\)$\to$\(6.029\mathrm{e}{1}\)\\
\addlinespace[2pt]
.5 & 1 & \(2.384\mathrm{e}{-7}\)$\to$\(2.753\mathrm{e}{-6}\) & \(2.870\mathrm{e}{-7}\)$\to$\(4.477\mathrm{e}{-6}\) & \(1.738\mathrm{e}{-5}\)$\to$\(1.132\mathrm{e}{3}\) & \(3.505\mathrm{e}{-5}\)$\to$\(1.172\mathrm{e}{3}\)\\
.5 & 2 & \(2.543\mathrm{e}{-7}\)$\to$\(3.433\mathrm{e}{-6}\) & \(4.046\mathrm{e}{-7}\)$\to$\(4.804\mathrm{e}{-6}\) & \(3.382\mathrm{e}{-5}\)$\to$\(2.505\mathrm{e}{2}\) & \(1.058\mathrm{e}{-4}\)$\to$\(2.556\mathrm{e}{2}\)\\
.5 & 3 & \(1.708\mathrm{e}{-7}\)$\to$\(2.217\mathrm{e}{-6}\) & \(2.443\mathrm{e}{-7}\)$\to$\(3.936\mathrm{e}{-6}\) & \(2.919\mathrm{e}{-5}\)$\to$\(5.800\mathrm{e}{0}\) & \(1.435\mathrm{e}{-4}\)$\to$\(6.116\mathrm{e}{0}\)\\
.5 & 4 & \(1.801\mathrm{e}{-7}\)$\to$\(2.415\mathrm{e}{-6}\) & \(3.257\mathrm{e}{-7}\)$\to$\(3.987\mathrm{e}{-6}\) & \(8.258\mathrm{e}{-6}\)$\to$\(1.209\mathrm{e}{2}\) & \(4.356\mathrm{e}{-5}\)$\to$\(1.261\mathrm{e}{2}\)\\
.5 & 5 & \(1.536\mathrm{e}{-7}\)$\to$\(1.240\mathrm{e}{-6}\) & \(2.503\mathrm{e}{-7}\)$\to$\(3.064\mathrm{e}{-6}\) & \(6.968\mathrm{e}{-6}\)$\to$\(9.757\mathrm{e}{-1}\) & \(5.317\mathrm{e}{-5}\)$\to$\(1.082\mathrm{e}{0}\)\\
\addlinespace[2pt]
.7 & 1 & \(2.182\mathrm{e}{-7}\)$\to$\(1.890\mathrm{e}{-6}\) & \(3.113\mathrm{e}{-7}\)$\to$\(3.214\mathrm{e}{-6}\) & \(6.284\mathrm{e}{-6}\)$\to$\(1.935\mathrm{e}{1}\) & \(1.719\mathrm{e}{-5}\)$\to$\(2.059\mathrm{e}{1}\)\\
.7 & 2 & \(1.972\mathrm{e}{-7}\)$\to$\(4.708\mathrm{e}{-6}\) & \(3.354\mathrm{e}{-7}\)$\to$\(1.347\mathrm{e}{-5}\) & \(9.551\mathrm{e}{-6}\)$\to$\(6.389\mathrm{e}{2}\) & \(5.667\mathrm{e}{-5}\)$\to$\(6.887\mathrm{e}{2}\)\\
.7 & 3 & \(1.891\mathrm{e}{-7}\)$\to$\(2.314\mathrm{e}{-6}\) & \(2.272\mathrm{e}{-7}\)$\to$\(3.480\mathrm{e}{-6}\) & \(2.384\mathrm{e}{-5}\)$\to$\(2.403\mathrm{e}{2}\) & \(1.409\mathrm{e}{-4}\)$\to$\(2.585\mathrm{e}{2}\)\\
.7 & 4 & \(2.169\mathrm{e}{-7}\)$\to$\(2.216\mathrm{e}{-6}\) & \(3.668\mathrm{e}{-7}\)$\to$\(3.862\mathrm{e}{-6}\) & \(1.725\mathrm{e}{-5}\)$\to$\(1.783\mathrm{e}{3}\) & \(6.572\mathrm{e}{-5}\)$\to$\(1.881\mathrm{e}{3}\)\\
.7 & 5 & \(2.007\mathrm{e}{-7}\)$\to$\(8.586\mathrm{e}{-6}\) & \(3.234\mathrm{e}{-7}\)$\to$\(3.286\mathrm{e}{-5}\) & \(1.134\mathrm{e}{-5}\)$\to$\(5.148\mathrm{e}{2}\) & \(4.944\mathrm{e}{-5}\)$\to$\(5.391\mathrm{e}{2}\)\\
\bottomrule
\end{longtable}

\endgroup

\begingroup\footnotesize
\setlength{\tabcolsep}{2.5pt}
\begin{longtable}{@{}cccccc@{}}
\caption{Validation-loss minimum $t^\star$ in each ResMLP run. Each loss, accuracy, and accuracy-gap cell gives the value at $t^\star$ and at epoch 25.}\label{tab:app-resmlp-turning-runs}\\
\toprule
$p$ & Run & $t^\star$ & $\ValLoss$ & $\ValAcc$ & $\GapAcc$\\
\midrule
\endfirsthead
\toprule
$p$ & Run & $t^\star$ & $\ValLoss$ & $\ValAcc$ & $\GapAcc$\\
\midrule
\endhead
.3 & 1 & 10 & .516$\to$1.252 & .834$\to$.805 & .085$\to$.195\\
.3 & 2 & 11 & .446$\to$1.046 & .789$\to$.838 & .168$\to$.162\\
.3 & 3 & 10 & .376$\to$.811 & .868$\to$.868 & .083$\to$.132\\
.3 & 4 & 8 & .581$\to$1.760 & .750$\to$.789 & .180$\to$.211\\
.3 & 5 & 8 & .742$\to$2.189 & .587$\to$.777 & .110$\to$.223\\
.5 & 1 & 8 & .875$\to$2.237 & .537$\to$.654 & .063$\to$.336\\
.5 & 2 & 9 & .793$\to$2.073 & .642$\to$.637 & .125$\to$.320\\
.5 & 3 & 7 & .866$\to$1.971 & .649$\to$.717 & .113$\to$.283\\
.5 & 4 & 9 & .822$\to$2.021 & .632$\to$.623 & .168$\to$.361\\
.5 & 5 & 8 & .970$\to$3.166 & .617$\to$.646 & .200$\to$.354\\
.7 & 1 & 6 & 1.349$\to$3.823 & .385$\to$.468 & .155$\to$.532\\
.7 & 2 & 9 & 1.291$\to$5.962 & .353$\to$.451 & .150$\to$.517\\
.7 & 3 & 9 & 1.251$\to$3.069 & .463$\to$.488 & .153$\to$.491\\
.7 & 4 & 7 & 1.328$\to$3.422 & .348$\to$.466 & .165$\to$.518\\
.7 & 5 & 7 & 1.298$\to$3.522 & .345$\to$.413 & .153$\to$.577\\
\bottomrule
\end{longtable}

\endgroup

\begin{figure}[htbp]
\centering
\includegraphics[width=0.94\linewidth]{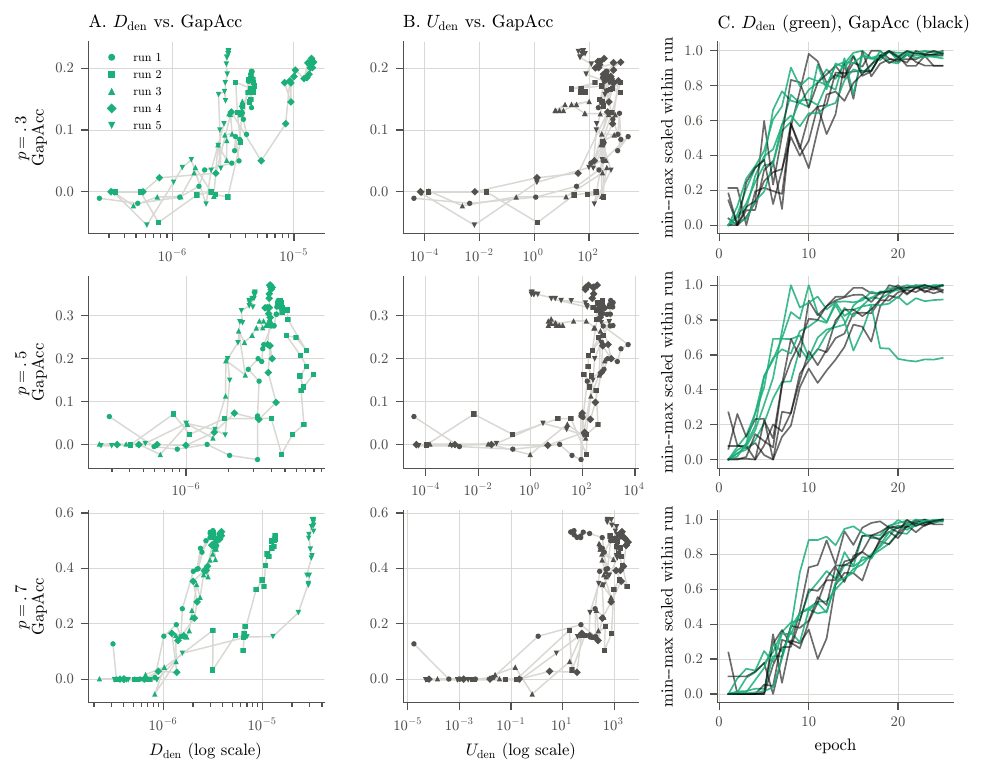}
\caption{Controlled hierarchy, run level. Rows: $p=.3,.5,.7$. A and B: every checkpoint of every run, $\GapAcc$ against $\Dden$ and against its unnormalized companion $\Uden$; grey lines connect consecutive checkpoints of one run, and marker shapes identify the five runs. C: $\Dden$ (green) and $\GapAcc$ (black) of each run, min--max scaled within the run.}
\label{fig:resmlp-runs}
\end{figure}

\FloatBarrier
\subsection{Vision Extension}
\label{app:results-vision}
Table~\ref{tab:vision-datasets} gives the dataset-level associations of the vision extension, and Table~\ref{tab:app-vision-runs} the raw and LOWESS correlations of each run. Both $\Dden$ and the probe-count-weighted $\Dw$ are positive under both views in every run. Table~\ref{tab:final-vision-runs} reports the final training and test accuracies and losses and the two readouts at the same checkpoint; Runs~1--5 identify the same independent runs in both run-level tables. The protocol and training settings are given in Sections~\ref{app:vision-protocol} and~\ref{app:config}.

\begin{table}[htbp]
\centering\small
\caption{Vision extension: mean within-run Pearson correlation with the accuracy gap, with the number of runs positive under both raw and LOWESS views. Five runs per dataset.}
\label{tab:vision-datasets}
\begin{tabular}{llccc}
\toprule
Dataset & Readout & Raw $r$ & LOWESS $r$ & Positive in both\\
\midrule
CIFAR-10 & $D_{\mathrm{den}}$ & .208 & .231 & 5/5\\
 & $D_{\mathrm{within}}^{\mathrm{wt}}$ & .167 & .187 & 5/5\\
\midrule
CIFAR-100 & $D_{\mathrm{den}}$ & .581 & .657 & 5/5\\
 & $D_{\mathrm{within}}^{\mathrm{wt}}$ & .570 & .645 & 5/5\\
\midrule
STL-10 & $D_{\mathrm{den}}$ & .428 & .633 & 5/5\\
 & $D_{\mathrm{within}}^{\mathrm{wt}}$ & .518 & .697 & 5/5\\
\midrule
DTD & $D_{\mathrm{den}}$ & .423 & .552 & 5/5\\
 & $D_{\mathrm{within}}^{\mathrm{wt}}$ & .529 & .693 & 5/5\\
\midrule
Overall & $D_{\mathrm{den}}$ & .410 & .518 & 20/20\\
 & $D_{\mathrm{within}}^{\mathrm{wt}}$ & .446 & .556 & 20/20\\
\bottomrule
\end{tabular}

\end{table}

\begin{table}[htbp]
\centering\small
\caption{Vision extension: within-run correlations with the accuracy gap, by dataset and run. Each pair gives raw / LOWESS Pearson $r$.}
\label{tab:app-vision-runs}
\resizebox{\linewidth}{!}{%
\begin{tabular}{lcccc}
\toprule
Dataset & Run & $D_{\mathrm{den}}$: raw / LOWESS & $D_{\mathrm{within}}^{\mathrm{wt}}$: raw / LOWESS & Positive in both\\
\midrule
CIFAR-10 & 1 & .271 / .299 & .244 / .269 & Yes\\
CIFAR-10 & 2 & .124 / .142 & .090 / .108 & Yes\\
CIFAR-10 & 3 & .090 / .096 & .022 / .024 & Yes\\
CIFAR-10 & 4 & .157 / .171 & .137 / .151 & Yes\\
CIFAR-10 & 5 & .395 / .446 & .343 / .382 & Yes\\
CIFAR-100 & 1 & .543 / .616 & .543 / .619 & Yes\\
CIFAR-100 & 2 & .564 / .632 & .567 / .631 & Yes\\
CIFAR-100 & 3 & .573 / .662 & .556 / .643 & Yes\\
CIFAR-100 & 4 & .636 / .696 & .597 / .663 & Yes\\
CIFAR-100 & 5 & .591 / .679 & .587 / .671 & Yes\\
STL-10 & 1 & .695 / .898 & .709 / .880 & Yes\\
STL-10 & 2 & .525 / .721 & .625 / .795 & Yes\\
STL-10 & 3 & .057 / .264 & .187 / .359 & Yes\\
STL-10 & 4 & .261 / .446 & .375 / .554 & Yes\\
STL-10 & 5 & .600 / .835 & .692 / .897 & Yes\\
DTD & 1 & .375 / .464 & .463 / .553 & Yes\\
DTD & 2 & .627 / .664 & .794 / .880 & Yes\\
DTD & 3 & .335 / .517 & .436 / .664 & Yes\\
DTD & 4 & .135 / .375 & .224 / .525 & Yes\\
DTD & 5 & .644 / .739 & .727 / .846 & Yes\\
\bottomrule
\end{tabular}
}
\end{table}

\FloatBarrier
\clearpage
\subsection{Final Checkpoint Metrics}
\label{app:final-run-metrics}
Tables~\ref{tab:final-nlp-runs}, \ref{tab:final-resmlp-runs}, and~\ref{tab:final-vision-runs} report the final checkpoint of the 125 observational runs: 90 language, 15 controlled-hierarchy, and 20 vision runs. Every row contains training accuracy, held-out accuracy, training cross-entropy, held-out cross-entropy, $\Dden$, and the probe-count-weighted $\Dw$, measured at the same checkpoint. Accuracies are percentages. In the vision rows, TestAcc and TestLoss denote accuracy and loss on the official test split. Run numbers identify the five independent runs within each condition and follow the ordering of the run-level tables above.

\subsubsection{Language Fine-Tuning}
\label{app:final-nlp}
\begingroup\footnotesize
\setlength{\tabcolsep}{3pt}\renewcommand{\arraystretch}{1.05}
\begin{longtable}{@{}lllrrrrrr@{}}
\caption{Final recorded checkpoint of every run of the 10-epoch language study. Accuracies are percentages; loss columns give cross-entropy. The two structural columns give the pooled and probe-count-weighted within-class readouts. Run numbers identify the five runs within each backbone--dataset pair.}\label{tab:final-nlp-runs}\\
\toprule
Backbone & Dataset & Run & \multicolumn{2}{c}{Accuracy (\%)} & \multicolumn{2}{c}{Cross-entropy} & $D_{\mathrm{den}}$ & $D_{\mathrm{within}}^{\mathrm{wt}}$\\
\cmidrule(lr){4-5}\cmidrule(lr){6-7}
 &  &  & Train & Validation & Train & Validation &  & \\
\midrule
\endfirsthead
\toprule
Backbone & Dataset & Run & \multicolumn{2}{c}{Accuracy (\%)} & \multicolumn{2}{c}{Cross-entropy} & $D_{\mathrm{den}}$ & $D_{\mathrm{within}}^{\mathrm{wt}}$\\
\cmidrule(lr){4-5}\cmidrule(lr){6-7}
 &  &  & Train & Validation & Train & Validation &  & \\
\midrule
\endhead
RoBERTa & BoolQ & 1 & 100.0000 & 80.1529 & 0.000349028 & 1.25114 & \(5.372\mathrm{e}{-7}\) & \(3.200\mathrm{e}{-7}\)\\
RoBERTa & BoolQ & 2 & 99.7879 & 80.5199 & 0.00429359 & 1.22107 & \(4.838\mathrm{e}{-7}\) & \(3.456\mathrm{e}{-7}\)\\
RoBERTa & BoolQ & 3 & 100.0000 & 80.6422 & 0.000143248 & 1.30162 & \(5.395\mathrm{e}{-7}\) & \(4.166\mathrm{e}{-7}\)\\
RoBERTa & BoolQ & 4 & 100.0000 & 77.8287 & 0.000482196 & 1.30888 & \(5.000\mathrm{e}{-7}\) & \(4.233\mathrm{e}{-7}\)\\
RoBERTa & BoolQ & 5 & 100.0000 & 80.7951 & 0.000155179 & 1.27091 & \(5.055\mathrm{e}{-7}\) & \(4.014\mathrm{e}{-7}\)\\
\midrule
RoBERTa & CoLA & 1 & 99.8832 & 83.8926 & 0.00191558 & 0.896428 & \(2.171\mathrm{e}{-7}\) & \(1.767\mathrm{e}{-7}\)\\
RoBERTa & CoLA & 2 & 100.0000 & 82.9338 & 0.00113167 & 0.917871 & \(1.900\mathrm{e}{-7}\) & \(1.602\mathrm{e}{-7}\)\\
RoBERTa & CoLA & 3 & 99.8832 & 85.0431 & 0.00352182 & 0.801679 & \(2.163\mathrm{e}{-7}\) & \(1.796\mathrm{e}{-7}\)\\
RoBERTa & CoLA & 4 & 99.5327 & 83.7967 & 0.00545818 & 0.873282 & \(1.854\mathrm{e}{-7}\) & \(1.601\mathrm{e}{-7}\)\\
RoBERTa & CoLA & 5 & 99.8832 & 82.1668 & 0.00224607 & 0.948831 & \(2.024\mathrm{e}{-7}\) & \(1.767\mathrm{e}{-7}\)\\
\midrule
RoBERTa & MRPC & 1 & 99.7275 & 88.7255 & 0.00595807 & 0.558065 & \(1.891\mathrm{e}{-7}\) & \(1.485\mathrm{e}{-7}\)\\
RoBERTa & MRPC & 2 & 99.4550 & 88.2353 & 0.00917208 & 0.590049 & \(1.897\mathrm{e}{-7}\) & \(1.385\mathrm{e}{-7}\)\\
RoBERTa & MRPC & 3 & 100.0000 & 88.2353 & 0.0012364 & 0.550351 & \(1.872\mathrm{e}{-7}\) & \(1.421\mathrm{e}{-7}\)\\
RoBERTa & MRPC & 4 & 99.7275 & 87.9902 & 0.00778297 & 0.564513 & \(1.817\mathrm{e}{-7}\) & \(1.277\mathrm{e}{-7}\)\\
RoBERTa & MRPC & 5 & 100.0000 & 87.7451 & 0.00136354 & 0.580965 & \(1.906\mathrm{e}{-7}\) & \(1.491\mathrm{e}{-7}\)\\
\midrule
RoBERTa & RTE & 1 & 100.0000 & 75.8123 & 0.000451056 & 1.19988 & \(4.336\mathrm{e}{-7}\) & \(2.841\mathrm{e}{-7}\)\\
RoBERTa & RTE & 2 & 99.5984 & 77.6173 & 0.00577501 & 1.21311 & \(4.377\mathrm{e}{-7}\) & \(2.589\mathrm{e}{-7}\)\\
RoBERTa & RTE & 3 & 100.0000 & 76.5343 & 0.00035315 & 1.2167 & \(4.606\mathrm{e}{-7}\) & \(3.184\mathrm{e}{-7}\)\\
RoBERTa & RTE & 4 & 100.0000 & 78.7004 & 0.000391016 & 1.13897 & \(4.390\mathrm{e}{-7}\) & \(3.038\mathrm{e}{-7}\)\\
RoBERTa & RTE & 5 & 99.5984 & 77.2563 & 0.0172711 & 1.14672 & \(4.227\mathrm{e}{-7}\) & \(2.334\mathrm{e}{-7}\)\\
\midrule
RoBERTa & SST-2 & 1 & 99.9555 & 94.2661 & 0.00151656 & 0.44246 & \(1.569\mathrm{e}{-7}\) & \(1.263\mathrm{e}{-7}\)\\
RoBERTa & SST-2 & 2 & 99.9703 & 94.1514 & 0.000699285 & 0.407494 & \(1.493\mathrm{e}{-7}\) & \(1.309\mathrm{e}{-7}\)\\
RoBERTa & SST-2 & 3 & 99.9555 & 93.2339 & 0.000864471 & 0.443833 & \(1.556\mathrm{e}{-7}\) & \(1.350\mathrm{e}{-7}\)\\
RoBERTa & SST-2 & 4 & 99.9555 & 94.1514 & 0.00149465 & 0.423883 & \(1.627\mathrm{e}{-7}\) & \(1.335\mathrm{e}{-7}\)\\
RoBERTa & SST-2 & 5 & 99.9852 & 93.9220 & 0.00108634 & 0.45863 & \(1.615\mathrm{e}{-7}\) & \(1.356\mathrm{e}{-7}\)\\
\midrule
RoBERTa & WiC & 1 & 100.0000 & 71.0031 & 0.00023054 & 1.76334 & \(1.461\mathrm{e}{-7}\) & \(1.218\mathrm{e}{-7}\)\\
RoBERTa & WiC & 2 & 100.0000 & 68.3386 & 0.000264803 & 1.90283 & \(1.327\mathrm{e}{-7}\) & \(1.126\mathrm{e}{-7}\)\\
RoBERTa & WiC & 3 & 100.0000 & 68.8088 & 0.000206183 & 1.85545 & \(1.563\mathrm{e}{-7}\) & \(1.255\mathrm{e}{-7}\)\\
RoBERTa & WiC & 4 & 100.0000 & 68.4953 & 0.000204987 & 1.99697 & \(1.426\mathrm{e}{-7}\) & \(1.183\mathrm{e}{-7}\)\\
RoBERTa & WiC & 5 & 100.0000 & 69.2790 & 0.000208679 & 1.84129 & \(1.530\mathrm{e}{-7}\) & \(1.148\mathrm{e}{-7}\)\\
\midrule
DeBERTa & BoolQ & 1 & 100.0000 & 82.0489 & 0.000723287 & 1.1249 & \(1.163\mathrm{e}{-7}\) & \(1.028\mathrm{e}{-7}\)\\
DeBERTa & BoolQ & 2 & 100.0000 & 82.2018 & 0.000954584 & 1.09873 & \(1.192\mathrm{e}{-7}\) & \(1.065\mathrm{e}{-7}\)\\
DeBERTa & BoolQ & 3 & 100.0000 & 82.2018 & 0.000149904 & 1.17523 & \(1.212\mathrm{e}{-7}\) & \(1.040\mathrm{e}{-7}\)\\
DeBERTa & BoolQ & 4 & 100.0000 & 81.1621 & 0.00087868 & 1.15909 & \(1.244\mathrm{e}{-7}\) & \(1.087\mathrm{e}{-7}\)\\
DeBERTa & BoolQ & 5 & 100.0000 & 81.6208 & 0.000436976 & 1.12353 & \(1.200\mathrm{e}{-7}\) & \(1.078\mathrm{e}{-7}\)\\
\midrule
DeBERTa & CoLA & 1 & 99.7664 & 83.3174 & 0.00193567 & 0.861684 & \(2.710\mathrm{e}{-7}\) & \(1.659\mathrm{e}{-7}\)\\
DeBERTa & CoLA & 2 & 100.0000 & 84.8514 & 0.00113117 & 0.796737 & \(1.029\mathrm{e}{-7}\) & \(9.033\mathrm{e}{-8}\)\\
DeBERTa & CoLA & 3 & 99.7664 & 84.2761 & 0.00271651 & 0.876833 & \(1.092\mathrm{e}{-7}\) & \(9.477\mathrm{e}{-8}\)\\
DeBERTa & CoLA & 4 & 99.6495 & 85.6184 & 0.00448548 & 0.826907 & \(1.032\mathrm{e}{-7}\) & \(9.095\mathrm{e}{-8}\)\\
DeBERTa & CoLA & 5 & 99.8832 & 83.7009 & 0.00176441 & 0.98918 & \(1.036\mathrm{e}{-7}\) & \(9.016\mathrm{e}{-8}\)\\
\midrule
DeBERTa & MRPC & 1 & 100.0000 & 88.7255 & 0.000827894 & 0.519969 & \(1.015\mathrm{e}{-7}\) & \(8.278\mathrm{e}{-8}\)\\
DeBERTa & MRPC & 2 & 100.0000 & 89.9510 & 0.00188786 & 0.53742 & \(9.960\mathrm{e}{-8}\) & \(8.090\mathrm{e}{-8}\)\\
DeBERTa & MRPC & 3 & 100.0000 & 89.2157 & 0.00252293 & 0.56003 & \(1.027\mathrm{e}{-7}\) & \(8.679\mathrm{e}{-8}\)\\
DeBERTa & MRPC & 4 & 100.0000 & 88.7255 & 0.000431744 & 0.583102 & \(9.807\mathrm{e}{-8}\) & \(8.243\mathrm{e}{-8}\)\\
DeBERTa & MRPC & 5 & 100.0000 & 90.1961 & 0.00214415 & 0.448555 & \(1.001\mathrm{e}{-7}\) & \(8.249\mathrm{e}{-8}\)\\
\midrule
DeBERTa & RTE & 1 & 100.0000 & 77.9783 & 0.000837575 & 0.928425 & \(1.080\mathrm{e}{-7}\) & \(8.528\mathrm{e}{-8}\)\\
DeBERTa & RTE & 2 & 100.0000 & 74.3682 & 0.000552609 & 1.32221 & \(9.979\mathrm{e}{-8}\) & \(8.371\mathrm{e}{-8}\)\\
DeBERTa & RTE & 3 & 100.0000 & 78.3394 & 0.000473645 & 1.16197 & \(1.119\mathrm{e}{-7}\) & \(8.722\mathrm{e}{-8}\)\\
DeBERTa & RTE & 4 & 100.0000 & 76.8953 & 0.001102 & 0.993483 & \(9.699\mathrm{e}{-8}\) & \(8.264\mathrm{e}{-8}\)\\
DeBERTa & RTE & 5 & 100.0000 & 77.6173 & 0.000590181 & 0.998765 & \(1.133\mathrm{e}{-7}\) & \(9.348\mathrm{e}{-8}\)\\
\midrule
DeBERTa & SST-2 & 1 & 99.9852 & 93.3486 & 0.00068339 & 0.485707 & \(5.205\mathrm{e}{-7}\) & \(3.473\mathrm{e}{-7}\)\\
DeBERTa & SST-2 & 2 & 100.0000 & 93.1193 & 0.000274231 & 0.519295 & \(4.986\mathrm{e}{-7}\) & \(3.203\mathrm{e}{-7}\)\\
DeBERTa & SST-2 & 3 & 99.9555 & 93.8073 & 0.00124173 & 0.419316 & \(4.931\mathrm{e}{-7}\) & \(3.300\mathrm{e}{-7}\)\\
DeBERTa & SST-2 & 4 & 99.9852 & 93.3486 & 0.000811524 & 0.48535 & \(4.183\mathrm{e}{-7}\) & \(2.769\mathrm{e}{-7}\)\\
DeBERTa & SST-2 & 5 & 99.9555 & 93.6927 & 0.000626589 & 0.471433 & \(4.684\mathrm{e}{-7}\) & \(3.365\mathrm{e}{-7}\)\\
\midrule
DeBERTa & WiC & 1 & 100.0000 & 69.9060 & 0.00157496 & 1.79798 & \(1.307\mathrm{e}{-7}\) & \(1.035\mathrm{e}{-7}\)\\
DeBERTa & WiC & 2 & 100.0000 & 69.9060 & 0.000582111 & 2.04717 & \(1.105\mathrm{e}{-7}\) & \(8.951\mathrm{e}{-8}\)\\
DeBERTa & WiC & 3 & 100.0000 & 69.5925 & 0.000144258 & 2.00799 & \(1.300\mathrm{e}{-7}\) & \(9.641\mathrm{e}{-8}\)\\
DeBERTa & WiC & 4 & 99.8158 & 71.4734 & 0.00299774 & 1.59646 & \(1.279\mathrm{e}{-7}\) & \(9.174\mathrm{e}{-8}\)\\
DeBERTa & WiC & 5 & 100.0000 & 70.6897 & 0.000225498 & 1.95782 & \(1.090\mathrm{e}{-7}\) & \(8.843\mathrm{e}{-8}\)\\
\midrule
Qwen3-1.7B & BoolQ & 1 & 100.0000 & 85.3211 & 0.00082768 & 1.02536 & \(7.729\mathrm{e}{-7}\) & \(6.195\mathrm{e}{-7}\)\\
Qwen3-1.7B & BoolQ & 2 & 100.0000 & 84.6483 & 0.00132187 & 1.1307 & \(8.603\mathrm{e}{-7}\) & \(7.455\mathrm{e}{-7}\)\\
Qwen3-1.7B & BoolQ & 3 & 100.0000 & 84.9847 & 9.72817e-05 & 1.13387 & \(8.815\mathrm{e}{-7}\) & \(7.131\mathrm{e}{-7}\)\\
Qwen3-1.7B & BoolQ & 4 & 100.0000 & 84.9235 & 0.000156565 & 1.04893 & \(1.011\mathrm{e}{-6}\) & \(7.112\mathrm{e}{-7}\)\\
Qwen3-1.7B & BoolQ & 5 & 100.0000 & 85.4128 & 0.000117134 & 1.0067 & \(8.574\mathrm{e}{-7}\) & \(7.321\mathrm{e}{-7}\)\\
\midrule
Qwen3-1.7B & CoLA & 1 & 100.0000 & 85.4267 & 0.00121346 & 1.0541 & \(6.859\mathrm{e}{-6}\) & \(4.638\mathrm{e}{-6}\)\\
Qwen3-1.7B & CoLA & 2 & 99.8832 & 84.4679 & 0.00147603 & 1.11826 & \(4.537\mathrm{e}{-6}\) & \(3.899\mathrm{e}{-6}\)\\
Qwen3-1.7B & CoLA & 3 & 100.0000 & 83.9885 & 0.00203114 & 1.10056 & \(9.190\mathrm{e}{-6}\) & \(4.748\mathrm{e}{-6}\)\\
Qwen3-1.7B & CoLA & 4 & 99.8832 & 84.4679 & 0.00374022 & 0.919996 & \(1.006\mathrm{e}{-5}\) & \(5.301\mathrm{e}{-6}\)\\
Qwen3-1.7B & CoLA & 5 & 99.8832 & 83.7967 & 0.00125292 & 1.08368 & \(4.459\mathrm{e}{-6}\) & \(3.324\mathrm{e}{-6}\)\\
\midrule
Qwen3-1.7B & MRPC & 1 & 100.0000 & 87.5000 & 0.00124111 & 0.795683 & \(3.377\mathrm{e}{-7}\) & \(2.981\mathrm{e}{-7}\)\\
Qwen3-1.7B & MRPC & 2 & 100.0000 & 86.5196 & 0.000217924 & 0.841098 & \(3.401\mathrm{e}{-7}\) & \(3.000\mathrm{e}{-7}\)\\
Qwen3-1.7B & MRPC & 3 & 100.0000 & 87.5000 & 0.000206326 & 0.69584 & \(2.950\mathrm{e}{-7}\) & \(2.527\mathrm{e}{-7}\)\\
Qwen3-1.7B & MRPC & 4 & 100.0000 & 86.7647 & 0.000352512 & 0.69986 & \(3.380\mathrm{e}{-7}\) & \(3.139\mathrm{e}{-7}\)\\
Qwen3-1.7B & MRPC & 5 & 100.0000 & 87.0098 & 0.000131487 & 0.803078 & \(3.220\mathrm{e}{-7}\) & \(2.758\mathrm{e}{-7}\)\\
\midrule
Qwen3-1.7B & RTE & 1 & 100.0000 & 89.1697 & 0.000162942 & 0.73223 & \(4.005\mathrm{e}{-6}\) & \(1.534\mathrm{e}{-6}\)\\
Qwen3-1.7B & RTE & 2 & 100.0000 & 85.5596 & 6.83531e-05 & 0.871289 & \(5.543\mathrm{e}{-6}\) & \(3.122\mathrm{e}{-6}\)\\
Qwen3-1.7B & RTE & 3 & 100.0000 & 89.1697 & 0.000187583 & 0.62574 & \(4.733\mathrm{e}{-6}\) & \(2.016\mathrm{e}{-6}\)\\
Qwen3-1.7B & RTE & 4 & 100.0000 & 85.1986 & 0.000292814 & 0.680328 & \(4.743\mathrm{e}{-6}\) & \(1.898\mathrm{e}{-6}\)\\
Qwen3-1.7B & RTE & 5 & 100.0000 & 85.5596 & 4.42912e-05 & 0.914481 & \(5.312\mathrm{e}{-6}\) & \(2.891\mathrm{e}{-6}\)\\
\midrule
Qwen3-1.7B & SST-2 & 1 & 100.0000 & 95.1835 & 4.09856e-05 & 0.512685 & \(2.762\mathrm{e}{-6}\) & \(1.960\mathrm{e}{-6}\)\\
Qwen3-1.7B & SST-2 & 2 & 100.0000 & 96.1009 & 4.44003e-06 & 0.464817 & \(2.140\mathrm{e}{-6}\) & \(1.884\mathrm{e}{-6}\)\\
Qwen3-1.7B & SST-2 & 3 & 99.9852 & 96.2156 & 0.000367619 & 0.459512 & \(2.589\mathrm{e}{-6}\) & \(2.051\mathrm{e}{-6}\)\\
Qwen3-1.7B & SST-2 & 4 & 99.9852 & 95.6422 & 0.000338219 & 0.391191 & \(6.693\mathrm{e}{-6}\) & \(2.720\mathrm{e}{-6}\)\\
Qwen3-1.7B & SST-2 & 5 & 100.0000 & 96.1009 & 9.68524e-05 & 0.362871 & \(2.365\mathrm{e}{-6}\) & \(1.881\mathrm{e}{-6}\)\\
\midrule
Qwen3-1.7B & WiC & 1 & 100.0000 & 71.1599 & 8.84378e-05 & 2.04542 & \(1.015\mathrm{e}{-6}\) & \(5.764\mathrm{e}{-7}\)\\
Qwen3-1.7B & WiC & 2 & 100.0000 & 71.0031 & 9.18411e-05 & 2.52354 & \(7.125\mathrm{e}{-7}\) & \(5.715\mathrm{e}{-7}\)\\
Qwen3-1.7B & WiC & 3 & 100.0000 & 69.4357 & 0.000122925 & 2.22586 & \(7.241\mathrm{e}{-7}\) & \(5.560\mathrm{e}{-7}\)\\
Qwen3-1.7B & WiC & 4 & 100.0000 & 70.5329 & 8.61298e-05 & 2.21311 & \(8.089\mathrm{e}{-7}\) & \(5.679\mathrm{e}{-7}\)\\
Qwen3-1.7B & WiC & 5 & 100.0000 & 73.0408 & 4.18091e-05 & 2.37831 & \(7.277\mathrm{e}{-7}\) & \(6.155\mathrm{e}{-7}\)\\
\bottomrule
\end{longtable}

\endgroup

\subsubsection{Controlled Hierarchy}
\label{app:final-resmlp}
\begingroup\footnotesize
\setlength{\tabcolsep}{3pt}\renewcommand{\arraystretch}{1.05}
\begin{longtable}{@{}llrrrrrr@{}}
\caption{Final recorded checkpoint of every run of the 25-epoch controlled hierarchy. Accuracies are percentages; loss columns give cross-entropy. The normalized structural readouts are paired with the behavioral measurements from the same checkpoint.}\label{tab:final-resmlp-runs}\\
\toprule
Private share $p$ & Run & \multicolumn{2}{c}{Accuracy (\%)} & \multicolumn{2}{c}{Cross-entropy} & $D_{\mathrm{den}}$ & $D_{\mathrm{within}}^{\mathrm{wt}}$\\
\cmidrule(lr){3-4}\cmidrule(lr){5-6}
 &  & Train & Validation & Train & Validation &  & \\
\midrule
\endfirsthead
\toprule
Private share $p$ & Run & \multicolumn{2}{c}{Accuracy (\%)} & \multicolumn{2}{c}{Cross-entropy} & $D_{\mathrm{den}}$ & $D_{\mathrm{within}}^{\mathrm{wt}}$\\
\cmidrule(lr){3-4}\cmidrule(lr){5-6}
 &  & Train & Validation & Train & Validation &  & \\
\midrule
\endhead
.3 & 1 & 100.0000 & 80.4878 & 0.0308979 & 1.25231 & \(4.472\mathrm{e}{-6}\) & \(2.355\mathrm{e}{-6}\)\\
.3 & 2 & 100.0000 & 83.8235 & 0.0115706 & 1.04563 & \(4.545\mathrm{e}{-6}\) & \(2.362\mathrm{e}{-6}\)\\
.3 & 3 & 100.0000 & 86.8293 & 0.00416103 & 0.810944 & \(3.735\mathrm{e}{-6}\) & \(1.871\mathrm{e}{-6}\)\\
.3 & 4 & 100.0000 & 78.9216 & 0.00439673 & 1.75954 & \(1.414\mathrm{e}{-5}\) & \(2.369\mathrm{e}{-6}\)\\
.3 & 5 & 100.0000 & 77.6699 & 0.0138152 & 2.18861 & \(2.874\mathrm{e}{-6}\) & \(1.022\mathrm{e}{-6}\)\\
\midrule
.5 & 1 & 98.9189 & 65.3659 & 0.0430272 & 2.2372 & \(4.477\mathrm{e}{-6}\) & \(2.753\mathrm{e}{-6}\)\\
.5 & 2 & 95.6757 & 63.7255 & 0.101287 & 2.0734 & \(4.804\mathrm{e}{-6}\) & \(3.433\mathrm{e}{-6}\)\\
.5 & 3 & 100.0000 & 71.7073 & 0.00317956 & 1.97123 & \(3.936\mathrm{e}{-6}\) & \(2.217\mathrm{e}{-6}\)\\
.5 & 4 & 98.3784 & 62.2549 & 0.0646098 & 2.02132 & \(3.987\mathrm{e}{-6}\) & \(2.415\mathrm{e}{-6}\)\\
.5 & 5 & 100.0000 & 64.5631 & 0.00406322 & 3.16627 & \(3.064\mathrm{e}{-6}\) & \(1.240\mathrm{e}{-6}\)\\
\midrule
.7 & 1 & 100.0000 & 46.8293 & 0.00947923 & 3.82337 & \(3.214\mathrm{e}{-6}\) & \(1.890\mathrm{e}{-6}\)\\
.7 & 2 & 96.7568 & 45.0980 & 0.186086 & 5.96216 & \(1.347\mathrm{e}{-5}\) & \(4.708\mathrm{e}{-6}\)\\
.7 & 3 & 97.8378 & 48.7805 & 0.0727181 & 3.06903 & \(3.480\mathrm{e}{-6}\) & \(2.314\mathrm{e}{-6}\)\\
.7 & 4 & 98.3784 & 46.5686 & 0.0566307 & 3.42248 & \(3.862\mathrm{e}{-6}\) & \(2.216\mathrm{e}{-6}\)\\
.7 & 5 & 98.9189 & 41.2621 & 0.0816254 & 3.52193 & \(3.286\mathrm{e}{-5}\) & \(8.586\mathrm{e}{-6}\)\\
\bottomrule
\end{longtable}

\endgroup

\subsubsection{Vision Extension}
\label{app:final-vision}
\begin{table}[htbp]
\centering\scriptsize
\caption{Final-checkpoint metrics of the vision extension. Runs~1--5 match the run-level correlation table within each dataset. TestAcc/TestLoss are evaluated on the official test split.}
\label{tab:final-vision-runs}
\resizebox{\linewidth}{!}{%
\begin{tabular}{llrrrrcc}
\toprule
Dataset & Run & TrainAcc (\%) & TestAcc (\%) & TrainLoss & TestLoss & $D_{\mathrm{den}}$ & $D_{\mathrm{within}}^{\mathrm{wt}}$\\
\midrule
CIFAR-10 & 1 & 100.0000 & 94.8500 & 3.03219e-06 & 0.264189 & \(2.736\mathrm{e}{-6}\) & \(2.263\mathrm{e}{-6}\)\\
CIFAR-10 & 2 & 100.0000 & 94.7500 & 6.44428e-06 & 0.270813 & \(2.728\mathrm{e}{-6}\) & \(2.270\mathrm{e}{-6}\)\\
CIFAR-10 & 3 & 100.0000 & 94.8500 & 3.42947e-06 & 0.265377 & \(2.752\mathrm{e}{-6}\) & \(2.286\mathrm{e}{-6}\)\\
CIFAR-10 & 4 & 100.0000 & 94.7700 & 3.74822e-06 & 0.267286 & \(2.798\mathrm{e}{-6}\) & \(2.303\mathrm{e}{-6}\)\\
CIFAR-10 & 5 & 100.0000 & 94.8200 & 3.73474e-06 & 0.265249 & \(2.686\mathrm{e}{-6}\) & \(2.224\mathrm{e}{-6}\)\\
CIFAR-100 & 1 & 100.0000 & 78.1600 & 0.000144386 & 1.12144 & \(2.913\mathrm{e}{-6}\) & \(1.243\mathrm{e}{-6}\)\\
CIFAR-100 & 2 & 99.9600 & 77.8500 & 0.000317865 & 1.14559 & \(3.072\mathrm{e}{-6}\) & \(1.296\mathrm{e}{-6}\)\\
CIFAR-100 & 3 & 99.9800 & 77.8000 & 0.000313581 & 1.13682 & \(2.924\mathrm{e}{-6}\) & \(1.247\mathrm{e}{-6}\)\\
CIFAR-100 & 4 & 100.0000 & 78.3600 & 0.000143657 & 1.12486 & \(2.917\mathrm{e}{-6}\) & \(1.248\mathrm{e}{-6}\)\\
CIFAR-100 & 5 & 99.9600 & 77.6000 & 0.000324084 & 1.14023 & \(2.900\mathrm{e}{-6}\) & \(1.220\mathrm{e}{-6}\)\\
STL-10 & 1 & 100.0000 & 94.5375 & 0.000318968 & 0.188267 & \(2.069\mathrm{e}{-6}\) & \(1.377\mathrm{e}{-6}\)\\
STL-10 & 2 & 100.0000 & 94.4500 & 0.000225182 & 0.206944 & \(1.887\mathrm{e}{-6}\) & \(1.270\mathrm{e}{-6}\)\\
STL-10 & 3 & 100.0000 & 94.3250 & 0.000244507 & 0.193014 & \(1.948\mathrm{e}{-6}\) & \(1.254\mathrm{e}{-6}\)\\
STL-10 & 4 & 100.0000 & 94.0875 & 0.000281542 & 0.201091 & \(1.982\mathrm{e}{-6}\) & \(1.305\mathrm{e}{-6}\)\\
STL-10 & 5 & 100.0000 & 94.2750 & 0.000325436 & 0.196887 & \(1.917\mathrm{e}{-6}\) & \(1.264\mathrm{e}{-6}\)\\
DTD & 1 & 100.0000 & 70.4787 & 0.000153053 & 1.3222 & \(4.228\mathrm{e}{-6}\) & \(2.392\mathrm{e}{-6}\)\\
DTD & 2 & 100.0000 & 70.2128 & 0.000153009 & 1.28153 & \(4.549\mathrm{e}{-6}\) & \(2.491\mathrm{e}{-6}\)\\
DTD & 3 & 100.0000 & 70.7447 & 0.000155808 & 1.26661 & \(4.645\mathrm{e}{-6}\) & \(2.430\mathrm{e}{-6}\)\\
DTD & 4 & 100.0000 & 70.1596 & 0.000150404 & 1.32738 & \(4.380\mathrm{e}{-6}\) & \(2.453\mathrm{e}{-6}\)\\
DTD & 5 & 100.0000 & 69.9468 & 0.000157194 & 1.28227 & \(4.632\mathrm{e}{-6}\) & \(2.442\mathrm{e}{-6}\)\\
\bottomrule
\end{tabular}
}
\end{table}

\FloatBarrier
\subsection{Alternative Explanations}
\label{app:alternatives}
Table~\ref{tab:alternatives} collects the analyses that address alternative sources of the readout--gap association discussed in Section~\ref{sec:discussion}.

\begin{table}[htbp]
\centering
\caption{Alternative explanations of the readout--gap association and the analyses that address them.}
\label{tab:alternatives}
\small
\begin{tabular}{@{}L{0.2\linewidth}L{0.3\linewidth}L{0.42\linewidth}@{}}
\toprule
Explanation & How it could produce the association & Analysis and result\\
\midrule
Common dependence on training time & Both series trend with training & Linear, quadratic, and logarithmic detrending keep $\Dden$ positive in $\NumNlpDdenPartialPos$, $\NumNlpDdenPartialQuadPos$, and $\NumNlpDdenPartialLogPos$ of 90 language runs; one-epoch changes co-vary in $\NumNlpDdenDiffPos$/90\\
Separation of class centres & The between-class term can raise $\Dden$ & $\Dw$ removes that term and is positively associated with the gap in 90/90 language and 15/15 hierarchy runs\\
Changing confidence & Gradient size shrinks as probes are fitted & Normalization cancels the confidence weight exactly in the binary language tasks (Proposition~\ref{prop:binary-confidence}); excluding low-mass probes changes the raw hierarchy association by at most $.002$\\
Example difficulty & Hard examples carry large gradients & Each normalized probe contributes equally, so difficulty enters through its allocation pattern; the hierarchy controls example-specific demand through the private share\\
Representation specialization & More differentiated processing separates allocations & The account links this differentiation to the held-out benefit of the newly used structure; the hierarchy and the intervention test that consequence\\
\bottomrule
\end{tabular}
\end{table}

\clearpage
\clearpage
\section{Training-Stage Intervention Matrices}
\label{app:do}
The intervention comprises a 240-branch main matrix and a 48-branch supplementary matrix. The main matrix crosses four physical batch sizes, three private shares, five training-progress anchors, and four doses in Run 1. The supplementary matrix adds Runs 2--5 at batch 16 and the $75\%$ training-accuracy trigger. The triggers and the initial state are the entry points of the main text. Every branch completes 150 updates. Each batch--private-share baseline in the main matrix and each run--private-share baseline in the supplementary matrix supplies its own anchors. All comparisons are made within these matched conditions.

\subsection{Anchors, Measurement Grid, and Pairing}
\label{app:do-definitions}
Let $N_B$ be the actual number of minibatches per original epoch. Baselines check the fixed 10\% training-evaluation subset after $\lceil N_B/2\rceil$ and $N_B$ minibatches, up to 25 epochs, and serialize the first state reaching each target accuracy. Baselines use the hierarchy's peak learning rate $1.5\times10^{-5}$ with linear warmup over the first 10\% of a 25-epoch schedule and linear decay to zero at its end, without gradient accumulation. The initial state supplies the initialization anchor. The anchor contains the model and optimizer state, random-number state, stream cursor, and next-minibatch identity. Its four dose branches retain the same physical batch and data stream, use the original AdamW without accumulation, and run 150 updates with constant learning rate $1.5\times10^{-5}$. Training and native probes use BF16; evaluations on the fixed training-evaluation subset and the validation split use FP32.

The primary grid consists of the anchor and the subsequent integer epochs counted from the start of the baseline, using the actual stream cursor. For example, an anchor at epoch 2.5 is followed by epochs 3 and 4 when they fall within the update budget. Pairing uses the last primary point shared by all four doses. At batch 16 there are only two primary points in this window; an off-grid step-150 observation is reported as auxiliary. If update 150 itself falls on an integer epoch, it is a primary point. Dense evaluations at relative updates $0,25,50,75,100,125,150$ and native probes at $0,50,100,150$ are also recorded.

For any endpoint quantity $X$, $\Delta X(\alpha)=X(B,p,\psi,\mathrm{run},\alpha;t)-X(B,p,\psi,\mathrm{run},0;t)$. Accuracy changes are in percentage points, and $\Delta\GapAcc=\Delta\TrainAcc-\Delta\ValAcc$. Positive-dose Pearson correlations use $.25,.5,1$ within a fixed design cell; zero dose supplies the paired reference. Dose correlations use the three positive doses within one condition. Main-matrix summaries give equal weight to the 12 batch--private-share conditions at each stage. Supplementary summaries first average the three private shares within each run and then report the mean and sample SD of the five run means. Runs are the unit of replication.

\subsection{Native Gradient Measurements}
\label{app:do-native}
The update transformation is Eq.~\eqref{eq:do-transform}. Its selected coordinates, which the native probes also use, are all trainable parameters of the sublayers of the final 10 ResMLP-24 blocks: the two affine normalizations, the token-mixing layer, and the channel MLP, 12,217,160 coordinates in total. The layer-scale vectors, the final affine normalization after the last block, and the classification head are excluded. The native probe readouts are measured at the resulting parameter state using untransformed gradients. With $a_i=|g_i|$, probe mean $\overline a$, class mean $\overline a_y$, and counts $N,N_y$ of the probes with finite gradients, including zero-gradient probes, the unnormalized readouts are
\begin{equation}
U_{\mathrm{den}}=\frac1N\sum_i\|a_i-\overline a\|_2^2,\qquad
\Uw=\sum_y\frac{N_y}{N}\left(\frac1{N_y}\sum_{i\in I_y}\|a_i-\overline a_y\|_2^2\right).
\end{equation}
For the valid examples in an individual transformed update, both magnitude dispersions contract by $(1-\alpha)^2$. A later native probe measures the changed model. Every recorded point holds native $U_{\mathrm{den}}$ and $\Uw$, both accuracies, both cross-entropies, and their gaps at every recorded point.

At the matched primary endpoint, full-dose branches end with larger native unnormalized dispersion than their zero-dose controls. $U_{\mathrm{den}}$ is higher in 56 of the 60 main-matrix conditions (median ratio to the control 23.9) and in all 15 five-run conditions (median ratio 5.1), and $\Uw$ is higher in 55/60 and 15/15. Unnormalized dispersion retains gradient magnitude and fitting progress (Section~\ref{sec:readout-roles}), and full-dose branches fit the training data more slowly. The native allocation dispersion moves with the reduced gap: $\Dden$ is lower in 51/60 and 12/15 conditions, with median ratios of .67 and .88. The paired dose response of the gap is negative in 14 of the 15 run--private-share conditions of the five-run group (Table~\ref{tab:do-seed-runs}).

\subsection{Main Matrix: Stage and Batch}
\label{app:do-endpoints}
Table~\ref{tab:do-batch-stage} and Figure~\ref{fig:do-batch} give the batch dimension of the response. Tables~\ref{tab:do-main-cells-b16}--\ref{tab:do-main-cells-b128} give all positive-dose gap contrasts at their matched primary endpoints. Each row is a batch--private-share--trigger condition; its four branches share one anchor of Run 1. At batch 16 and $p=.3$, one checkpoint is the first to reach both the 87.5\% and the 93.75\% thresholds, so these two triggers share an anchor and their rows in Table~\ref{tab:do-main-cells-b16} coincide.

\begin{table}[htbp]
\centering\small
\caption{Main matrix by intervention stage. The three dose columns give the median paired accuracy-gap change (percentage points) over the 12 batch--private-share conditions. At each fixed condition, dose $r$ is the Pearson correlation over the three positive doses $.25,.5,1$; the last columns give its equal-condition mean and the number of negative correlations. All conditions use Run 1.}
\label{tab:do-summary}
\begin{tabular}{lrrrrr}
\toprule
Trigger & $\Delta\GapAcc_{.25}$ & $\Delta\GapAcc_{.5}$ & $\Delta\GapAcc_1$ & Dose $r$ & Negative\\
\midrule
Init. & +0.57 & -0.14 & -0.53 & -.047 & 6/12\\
50\% & -0.43 & -0.62 & -3.65 & -.167 & 7/12\\
75\% & -3.11 & -0.40 & -7.30 & -.445 & 9/12\\
87.5\% & -0.90 & -1.05 & -11.79 & -.839 & 11/12\\
93.75\% & -1.05 & -1.52 & -4.91 & -.630 & 10/12\\
\bottomrule
\end{tabular}

\end{table}

\begin{table}[htbp]
\centering\small
\caption{Median full-dose paired gap change across the three private shares, in percentage points. Columns are the training-accuracy triggers.}
\label{tab:do-batch-stage}
\begin{tabular}{lrrrrr}
\toprule
Batch & Init. & $50\%$ & $75\%$ & $87.5\%$ & $93.75\%$\\
\midrule
16 & +0.00 & +4.06 & -6.92 & -15.45 & -2.91\\
32 & +0.00 & -3.89 & -4.69 & -9.14 & -1.19\\
64 & -1.05 & -5.58 & -7.67 & -9.19 & -5.41\\
128 & -1.45 & -3.94 & -12.33 & -14.38 & -18.84\\
\bottomrule
\end{tabular}

\end{table}

\begin{figure}[htbp]
\centering
\includegraphics[width=0.78\linewidth]{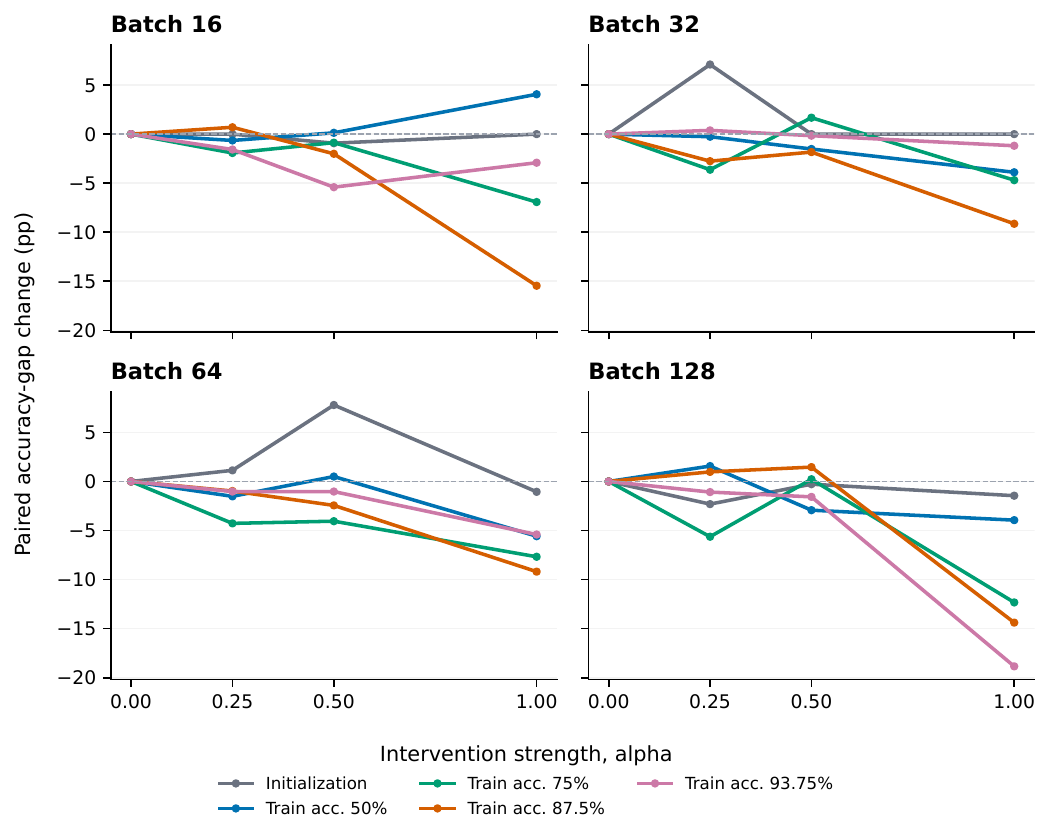}
\caption{Intervention response by batch size in the 240-branch matrix. Colours identify the training-accuracy trigger. Horizontal axis: intervention strength $\alpha$; vertical axis: paired accuracy-gap change relative to $\alpha=0$ at the same batch, private share, anchor, and measurement point (median of the three private shares).}
\label{fig:do-batch}
\end{figure}
\begingroup\small
\setlength{\tabcolsep}{4pt}
\begin{table}[!htbp]
\centering\small
\caption{Main intervention matrix, batch 16: paired accuracy-gap changes in percentage points at the matched primary endpoint.}
\label{tab:do-main-cells-b16}
\begin{tabular}{rrrrrrrr}
\toprule
Batch & $p$ & Trigger & Run & $\Delta\GapAcc_{.25}$ & $\Delta\GapAcc_{.5}$ & $\Delta\GapAcc_1$ & Dose $r$\\
\midrule
16 & 0.3 & 0 & 1 & +0.00 & -2.49 & -2.98 & -.847\\
16 & 0.3 & 0.5 & 1 & -0.65 & +0.13 & +4.06 & .984\\
16 & 0.3 & 0.75 & 1 & -1.92 & -3.49 & -2.62 & -.267\\
16 & 0.3 & 0.875 & 1 & +1.31 & +0.76 & +0.66 & -.847\\
16 & 0.3 & 0.9375 & 1 & +1.31 & +0.76 & +0.66 & -.847\\
16 & 0.5 & 0 & 1 & +1.52 & +3.47 & +1.52 & -.189\\
16 & 0.5 & 0.5 & 1 & -0.32 & +4.32 & +6.78 & .934\\
16 & 0.5 & 0.75 & 1 & -2.60 & -0.87 & -6.92 & -.817\\
16 & 0.5 & 0.875 & 1 & -2.06 & -4.17 & -16.93 & -.980\\
16 & 0.5 & 0.9375 & 1 & -5.14 & -6.66 & -4.85 & .334\\
16 & 0.7 & 0 & 1 & -2.18 & -0.92 & +0.00 & .962\\
16 & 0.7 & 0.5 & 1 & -2.11 & -1.46 & -3.40 & -.786\\
16 & 0.7 & 0.75 & 1 & +0.49 & -0.05 & -9.78 & -.959\\
16 & 0.7 & 0.875 & 1 & +0.70 & -2.00 & -15.45 & -.984\\
16 & 0.7 & 0.9375 & 1 & -1.57 & -5.41 & -2.91 & -.162\\
\bottomrule
\end{tabular}

\end{table}
\begin{table}[!htbp]
\centering\small
\caption{Main intervention matrix, batch 32: paired accuracy-gap changes in percentage points at the matched primary endpoint.}
\label{tab:do-main-cells-b32}
\begin{tabular}{rrrrrrrr}
\toprule
Batch & $p$ & Trigger & Run & $\Delta\GapAcc_{.25}$ & $\Delta\GapAcc_{.5}$ & $\Delta\GapAcc_1$ & Dose $r$\\
\midrule
32 & 0.3 & 0 & 1 & +7.40 & -6.22 & -5.88 & -.741\\
32 & 0.3 & 0.5 & 1 & +1.08 & -1.52 & -3.89 & -.977\\
32 & 0.3 & 0.75 & 1 & -3.63 & +1.31 & -4.69 & -.350\\
32 & 0.3 & 0.875 & 1 & -3.61 & +0.49 & -1.67 & .297\\
32 & 0.3 & 0.9375 & 1 & +0.38 & -0.16 & +1.79 & .822\\
32 & 0.5 & 0 & 1 & +7.09 & +6.70 & +14.07 & .928\\
32 & 0.5 & 0.5 & 1 & -0.26 & -6.14 & -10.18 & -.956\\
32 & 0.5 & 0.75 & 1 & -3.73 & +1.67 & +1.89 & .777\\
32 & 0.5 & 0.875 & 1 & -2.16 & -3.57 & -9.14 & -.990\\
32 & 0.5 & 0.9375 & 1 & +3.89 & +3.78 & -1.19 & -.951\\
32 & 0.7 & 0 & 1 & -2.27 & +0.00 & +0.00 & .756\\
32 & 0.7 & 0.5 & 1 & -0.54 & +5.79 & +4.06 & .556\\
32 & 0.7 & 0.75 & 1 & +3.36 & +5.58 & -9.23 & -.890\\
32 & 0.7 & 0.875 & 1 & -2.76 & -1.83 & -29.12 & -.935\\
32 & 0.7 & 0.9375 & 1 & -1.52 & -2.43 & -24.75 & -.956\\
\bottomrule
\end{tabular}

\end{table}
\begin{table}[!htbp]
\centering\small
\caption{Main intervention matrix, batch 64: paired accuracy-gap changes in percentage points at the matched primary endpoint.}
\label{tab:do-main-cells-b64}
\begin{tabular}{rrrrrrrr}
\toprule
Batch & $p$ & Trigger & Run & $\Delta\GapAcc_{.25}$ & $\Delta\GapAcc_{.5}$ & $\Delta\GapAcc_1$ & Dose $r$\\
\midrule
64 & 0.3 & 0 & 1 & +1.13 & -1.08 & -17.42 & -.975\\
64 & 0.3 & 0.5 & 1 & -2.98 & +0.50 & +10.51 & .997\\
64 & 0.3 & 0.75 & 1 & -2.54 & -4.05 & -4.26 & -.825\\
64 & 0.3 & 0.875 & 1 & -3.73 & -5.08 & -5.99 & -.955\\
64 & 0.3 & 0.9375 & 1 & -1.03 & -1.03 & -5.41 & -.945\\
64 & 0.5 & 0 & 1 & +12.70 & +8.45 & -1.05 & -1.000\\
64 & 0.5 & 0.5 & 1 & +2.11 & +1.03 & -5.58 & -.979\\
64 & 0.5 & 0.75 & 1 & -8.82 & -4.81 & -7.67 & .091\\
64 & 0.5 & 0.875 & 1 & +1.41 & -0.28 & -9.19 & -.983\\
64 & 0.5 & 0.9375 & 1 & +4.44 & +0.00 & -4.90 & -.987\\
64 & 0.7 & 0 & 1 & -4.05 & +7.79 & +4.55 & .555\\
64 & 0.7 & 0.5 & 1 & -1.52 & -0.16 & -17.18 & -.919\\
64 & 0.7 & 0.75 & 1 & -4.27 & -2.97 & -14.63 & -.907\\
64 & 0.7 & 0.875 & 1 & -0.98 & -2.44 & -18.21 & -.967\\
64 & 0.7 & 0.9375 & 1 & -2.06 & -4.39 & -32.95 & -.965\\
\bottomrule
\end{tabular}

\end{table}
\begin{table}[!htbp]
\centering\small
\caption{Main intervention matrix, batch 128: paired accuracy-gap changes in percentage points at the matched primary endpoint.}
\label{tab:do-main-cells-b128}
\begin{tabular}{rrrrrrrr}
\toprule
Batch & $p$ & Trigger & Run & $\Delta\GapAcc_{.25}$ & $\Delta\GapAcc_{.5}$ & $\Delta\GapAcc_1$ & Dose $r$\\
\midrule
128 & 0.3 & 0 & 1 & +1.17 & -0.28 & -1.45 & -.969\\
128 & 0.3 & 0.5 & 1 & +3.30 & -1.08 & +7.91 & .666\\
128 & 0.3 & 0.75 & 1 & -9.14 & +0.22 & -2.91 & .498\\
128 & 0.3 & 0.875 & 1 & -0.82 & +1.24 & -5.66 & -.809\\
128 & 0.3 & 0.9375 & 1 & -1.08 & -1.57 & -4.92 & -.977\\
128 & 0.5 & 0 & 1 & -2.31 & +8.85 & +8.73 & .750\\
128 & 0.5 & 0.5 & 1 & +1.57 & -6.00 & -3.94 & -.557\\
128 & 0.5 & 0.75 & 1 & -5.63 & -0.75 & -12.33 & -.720\\
128 & 0.5 & 0.875 & 1 & +7.15 & +4.50 & -14.38 & -.976\\
128 & 0.5 & 0.9375 & 1 & -3.41 & -4.88 & -18.84 & -.970\\
128 & 0.7 & 0 & 1 & -15.68 & -7.13 & -12.38 & .201\\
128 & 0.7 & 0.5 & 1 & -1.79 & -2.93 & -14.16 & -.969\\
128 & 0.7 & 0.75 & 1 & +3.41 & +0.44 & -21.12 & -.975\\
128 & 0.7 & 0.875 & 1 & +0.98 & +1.46 & -23.01 & -.939\\
128 & 0.7 & 0.9375 & 1 & +0.00 & -1.46 & -30.74 & -.958\\
\bottomrule
\end{tabular}

\end{table}
\FloatBarrier

\endgroup

\subsection{Supplementary Matrix: Stability Across Runs}
\label{app:do-changes}
At batch 16 and the $75\%$ trigger, all five training runs have a negative full-dose gap change after averaging the private shares (Figure~\ref{fig:do-seed-response}).

\begin{figure}[htbp]
\centering
\includegraphics[width=\linewidth]{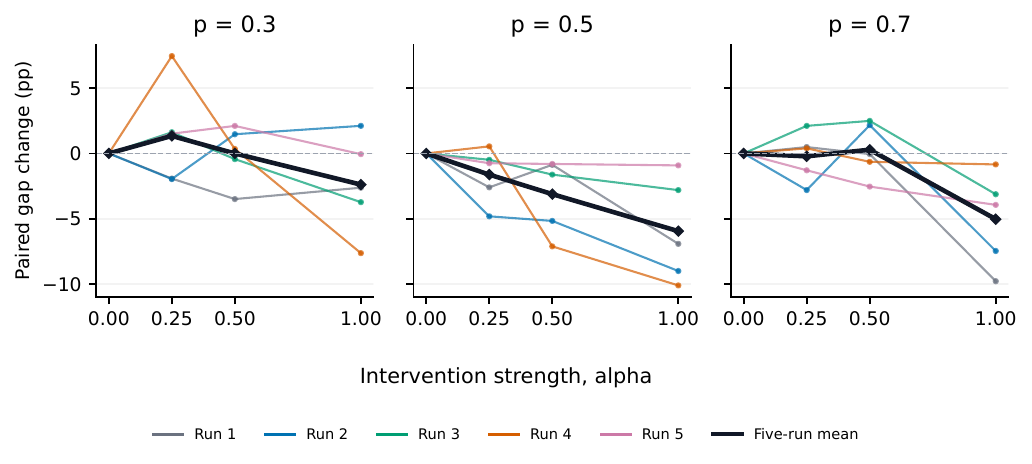}
\caption{Five-run replication at batch 16 and the $75\%$ training-accuracy trigger. Panels fix the private share. Coloured lines are Runs 1--5; the black line is their mean. Run 1 belongs to the main matrix and is counted once. Every dose is paired with its own same-run, same-anchor zero-dose control at the last common integer-epoch point.}
\label{fig:do-seed-response}
\end{figure}

 Tables~\ref{tab:do-seed-summary} and~\ref{tab:do-seed-runs} give the five-run summary and the individual-run responses. The final design-cell table retains every private-share--run combination, including the matching Run 1 cells from the main matrix.

\begin{table}[htbp]
\centering\small
\caption{Fixed-stage five-run response, in percentage points. Gap changes are mean $\pm$ sample SD of the five run means; training and validation changes are their corresponding means.}
\label{tab:do-seed-summary}
\begin{tabular}{lrrr}
\toprule
$\alpha$ & $\Delta\GapAcc$ & $\Delta\TrainAcc$ & $\Delta\ValAcc$\\
\midrule
0.25 & $-0.17 \pm 2.28$ & +0.58 & +0.74\\
0.5 & $-0.94 \pm 1.03$ & -1.37 & -0.42\\
1 & $-4.46 \pm 2.03$ & -5.44 & -0.99\\
\bottomrule
\end{tabular}

\end{table}
\begin{table}[htbp]
\centering\small
\caption{Individual-run primary responses at batch 16 and the $75\%$ trigger. Gap changes are means across the three private shares. Dose $r$ averages the three within-private-share positive-dose correlations.}
\label{tab:do-seed-runs}
\begin{tabular}{lrrrrr}
\toprule
Run & $\Delta\GapAcc_{.25}$ & $\Delta\GapAcc_{.5}$ & $\Delta\GapAcc_1$ & Dose $r$ & Negative\\
\midrule
1 & -1.34 & -1.47 & -6.44 & -.681 & 3/3\\
2 & -3.19 & -0.51 & -4.79 & -.255 & 2/3\\
3 & +1.08 & +0.15 & -3.22 & -.969 & 3/3\\
4 & +2.79 & -2.47 & -6.19 & -.913 & 3/3\\
5 & -0.18 & -0.41 & -1.64 & -.937 & 3/3\\
\bottomrule
\end{tabular}

\end{table}
\begingroup\small
\setlength{\tabcolsep}{4pt}
\begin{table}[!htbp]
\centering\small
\caption{Five-run replication at batch 16 and trigger .75: every run--private-share condition, including Run 1 from the main matrix. Changes are in percentage points.}
\label{tab:do-replication-cells}
\begin{tabular}{rrrrrrrr}
\toprule
Batch & $p$ & Trigger & Run & $\Delta\GapAcc_{.25}$ & $\Delta\GapAcc_{.5}$ & $\Delta\GapAcc_1$ & Dose $r$\\
\midrule
16 & 0.3 & 0.75 & 1 & -1.92 & -3.49 & -2.62 & -.267\\
16 & 0.3 & 0.75 & 2 & -1.96 & +1.47 & +2.11 & .844\\
16 & 0.3 & 0.75 & 3 & +1.62 & -0.44 & -3.73 & -.998\\
16 & 0.3 & 0.75 & 4 & +7.45 & +0.34 & -7.63 & -.988\\
16 & 0.3 & 0.75 & 5 & +1.51 & +2.11 & -0.06 & -.823\\
16 & 0.5 & 0.75 & 1 & -2.60 & -0.87 & -6.92 & -.817\\
16 & 0.5 & 0.75 & 2 & -4.81 & -5.17 & -9.00 & -.967\\
16 & 0.5 & 0.75 & 3 & -0.49 & -1.62 & -2.81 & -.984\\
16 & 0.5 & 0.75 & 4 & +0.54 & -7.11 & -10.11 & -.906\\
16 & 0.5 & 0.75 & 5 & -0.75 & -0.81 & -0.92 & -1.000\\
16 & 0.7 & 0.75 & 1 & +0.49 & -0.05 & -9.78 & -.959\\
16 & 0.7 & 0.75 & 2 & -2.80 & +2.16 & -7.47 & -.641\\
16 & 0.7 & 0.75 & 3 & +2.11 & +2.49 & -3.12 & -.923\\
16 & 0.7 & 0.75 & 4 & +0.39 & -0.64 & -0.84 & -.847\\
16 & 0.7 & 0.75 & 5 & -1.29 & -2.54 & -3.94 & -.988\\
\bottomrule
\end{tabular}

\end{table}
\FloatBarrier

\endgroup

\paragraph{Fixed-update auxiliary endpoint}
At relative update 150, the full-dose gap response in the five-run subgroup is $-7.67\pm2.91$ percentage points; the corresponding training and validation changes are $-10.02$ and $-2.35$ points. These fixed-update values are separate from the integer-epoch primary statistics.

\clearpage
\section{Sensitivity and Readout Ablations}
\label{app:sensitivity}
The monitor-depth, probe-count, signed-gradient, and precision analyses use the controlled hierarchy, supplemented by a depth check and recorded gradient-mass diagnostics in the language backbones. Two later subsections examine negative-gap checkpoints and the module-level projection in the hierarchy and language runs, and the last compares signed and absolute readouts in the hierarchy. All use the shared-output weight-gradient scheme of Section~\ref{app:monitor}. The main protocol fixes $k=4,6,2$ for RoBERTa-base, DeBERTa-v3-base, and Qwen3-1.7B, respectively, and $k=10$ for the hierarchy; the depth analyses below vary only the number of final blocks in this scheme. The probe-count analysis reads the training runs of Table~\ref{tab:resmlp-main} with different numbers of probes. The hierarchy's monitor-depth, signed-gradient, low-mass, and precision analyses each train their own runs with the main configuration.

\subsection{Monitor Depth in the Hierarchy}
Fifteen training runs, five for each $p\in\{.3,.5,.7\}$, are each read with the final $k\in\{2,4,6,8,10\}$ blocks monitored. The five readings of one run share its accuracy and loss and differ only in the monitored coordinates. Table~\ref{tab:app-depth} shows that both the mean association with the accuracy gap and the number of positive runs grow with the number of monitored blocks. The mean raw correlation of $\Dden$ rises from $\NumDepthTwoDdenRaw$ ($\NumDepthTwoDdenPos$/15 positive) with two blocks to $\NumDepthTenDdenRaw$ ($\NumDepthTenDdenPos$/15) with ten.

\begin{table}[htbp]
\centering
\caption{Monitor depth. Mean within-run correlation with the accuracy gap over 15 runs, with the number of positive runs.}
\label{tab:app-depth}
\small
\begin{tabular}{ccccc}
\toprule
Monitored blocks $k$ & $D_{\mathrm{within}}^{\mathrm{wt}}$ raw & $D_{\mathrm{within}}^{\mathrm{wt}}$ LOWESS & $D_{\mathrm{den}}$ raw & $D_{\mathrm{den}}$ LOWESS\\
\midrule
2 & .276 (9/15) & .318 (11/15) & .366 (11/15) & .425 (11/15)\\
4 & .371 (10/15) & .421 (11/15) & .486 (11/15) & .570 (12/15)\\
6 & .451 (12/15) & .541 (14/15) & .565 (14/15) & .674 (14/15)\\
8 & .516 (15/15) & .639 (15/15) & .617 (14/15) & .726 (15/15)\\
10 & .562 (15/15) & .690 (15/15) & .645 (15/15) & .753 (15/15)\\
\bottomrule
\end{tabular}

\end{table}

\subsection{Monitor Depth in the Language Backbones}
\label{app:sensitivity-depth-nlp}
The language depth check uses one complete 10-epoch BoolQ trajectory per backbone, separate from the short calibration runs of Section~\ref{app:monitor}. Each run uses the BoolQ configuration of the main study: 10 epochs, the selected learning rate and the backbone's batch size, weight decay, and dropout (Table~\ref{tab:config-runs}), full fine-tuning in FP32 for RoBERTa-base and DeBERTa-v3-base and rank-32 LoRA in BF16 for Qwen3-1.7B, 128 class-proportional training probes, and evaluation about every half epoch (20 checkpoints for RoBERTa-base and DeBERTa-v3-base, 22 for Qwen3-1.7B). Each run is read with the shared-output weight-gradient scheme of Table~\ref{tab:monitors} over the final $k\in\{2,4,6,8,10\}$ blocks. As in the hierarchy, the five readings of one run share its accuracy and loss and differ only in the monitored coordinates.

Each backbone contributes one training trajectory, so Table~\ref{tab:app-depth-nlp} reports its within-run correlations. These runs are separate from both the calibration runs and the 90 main runs.

Every association of $\Dw$ and $\Dden$ with the accuracy gap is positive at every depth and under both views, with raw correlations from $\NumBoolqDepthMin$ to $\NumBoolqDepthMax$ and LOWESS correlations from $\NumBoolqDepthLowessMin$ to $\NumBoolqDepthLowessMax$. The strength depends on the depth in a backbone-specific way: it changes little for DeBERTa-v3-base, decreases moderately with depth for RoBERTa-base, and, as in the hierarchy, increases with depth for Qwen3-1.7B. For Qwen3-1.7B, the raw correlation of $\Dden$ rises from $\NumBoolqQwenDdenTwo$ with two blocks to $\NumBoolqQwenDdenTen$ with ten.

\begin{table}[htbp]
\centering
\caption{Monitor-depth sensitivity in the language backbones: one BoolQ run per backbone, separate from the calibration runs, read with the shared-output weight-gradient scheme over the final $k$ blocks. Each cell gives the raw / LOWESS within-run Pearson correlation with the accuracy gap; the five depths of a backbone come from the same training run. The main protocol fixes $k=4,6,2$ for RoBERTa-base, DeBERTa-v3-base, and Qwen3-1.7B, respectively. Monitored coordinates per block: 2{,}950{,}656 for RoBERTa-base and DeBERTa-v3-base, 393{,}216 for Qwen3-1.7B.}
\label{tab:app-depth-nlp}
\small
\begin{tabular}{ccccccc}
\toprule
 & \multicolumn{2}{c}{RoBERTa-base} & \multicolumn{2}{c}{DeBERTa-v3-base} & \multicolumn{2}{c}{Qwen3-1.7B}\\
\cmidrule(lr){2-3}\cmidrule(lr){4-5}\cmidrule(lr){6-7}
Monitored blocks $k$ & $D_{\mathrm{within}}^{\mathrm{wt}}$ & $D_{\mathrm{den}}$ & $D_{\mathrm{within}}^{\mathrm{wt}}$ & $D_{\mathrm{den}}$ & $D_{\mathrm{within}}^{\mathrm{wt}}$ & $D_{\mathrm{den}}$\\
\midrule
2 & .805/.876 & .875/.910 & .898/.931 & .902/.916 & .773/.787 & .762/.778\\
4 & .816/.911 & .860/.914 & .941/.968 & .911/.927 & .845/.858 & .853/.866\\
6 & .694/.820 & .799/.894 & .946/.973 & .906/.925 & .875/.888 & .888/.899\\
8 & .667/.812 & .755/.867 & .932/.970 & .901/.927 & .900/.912 & .917/.927\\
10 & .682/.829 & .737/.846 & .878/.941 & .864/.906 & .915/.927 & .945/.953\\
\bottomrule
\end{tabular}

\end{table}

\paragraph{Language runs without the calibration task}
Because the language depths were chosen by the readout--gap association on BoolQ (Section~\ref{app:monitor}), BoolQ is the calibration task. Excluding its 15 runs, $\Dden$ is positively associated with the accuracy gap under both views in $\NumNlpNoBoolqDdenJoint$/75 of the remaining runs and $\Dw$ in $\NumNlpNoBoolqWJoint$/75. Their mean raw / LOWESS correlations are $\NumNlpNoBoolqDdenRaw$ / $\NumNlpNoBoolqDdenLowess$ and $\NumNlpNoBoolqWRaw$ / $\NumNlpNoBoolqWLowess$, compared with $\NumNlpDdenRaw$ / $\NumNlpDdenLowess$ and $\NumNlpWRaw$ / $\NumNlpWLowess$ over all 90 runs (per-dataset values in Table~\ref{tab:app-nlp-dataset-level}).

\subsection{Number of Probes}
\label{app:sensitivity-probes}
Table~\ref{tab:app-probe-count} reports the association for 64, 128, 256, and 512 class-balanced probes. All four columns read the same 15 training runs, whose accuracy trajectories are identical across columns, so the cells differ only in the probe sample; the $N=128$ column reproduces Table~\ref{tab:resmlp-main}. Every cell has five of five runs positive under both views, and the associations change little with $N$.

\begin{table}[htbp]
\centering
\caption{Number of training probes. Mean raw / LOWESS within-run correlation with the accuracy gap over five runs.}
\label{tab:app-probe-count}
\small
\begin{tabular}{llcccc}
\toprule
Readout & $p$ & $N=64$ & $N=128$ & $N=256$ & $N=512$\\
\midrule
\multirow{3}{*}{$D_{\mathrm{within}}^{\mathrm{wt}}$} & .3 & .698 / .749 & .703 / .757 & .730 / .790 & .734 / .795\\
 & .5 & .669 / .716 & .632 / .679 & .626 / .675 & .615 / .662\\
 & .7 & .893 / .919 & .891 / .911 & .873 / .895 & .872 / .895\\
\midrule
\multirow{3}{*}{$D_{\mathrm{den}}$} & .3 & .911 / .950 & .911 / .951 & .907 / .948 & .905 / .946\\
 & .5 & .795 / .831 & .777 / .812 & .784 / .819 & .776 / .812\\
 & .7 & .970 / .983 & .972 / .984 & .970 / .983 & .968 / .981\\
\bottomrule
\end{tabular}

\end{table}

\subsection{Near-Zero Gradient Mass and Precision}
A probe that is fitted with high confidence has a small gradient, and its allocation is then sensitive to numerical error (Section~\ref{app:norm-stability}). In the $k=10$ depth runs, 460 of 48{,}000 probe evaluations (0.96\%) at 61 of 375 checkpoints had monitored mass at most $\varepsilon$ and were excluded as invalid; the median over checkpoints of the 1st, 5th, 10th, and 50th mass percentiles was $2.3\times10^{-4}$, $2.5\times10^{-3}$, $2.2\times10^{-2}$, and $4.6$. Excluding in addition the lowest 1\%, 5\%, or 10\% of masses at each checkpoint leaves the association of $\Dden$ almost unchanged (Table~\ref{tab:app-low-mass}).

\begin{table}[htbp]
\centering
\caption{Association of $\Dden$ with the accuracy gap after excluding the lowest-mass probes at each checkpoint ($k=10$ depth runs; 15 runs).}
\label{tab:app-low-mass}
\small
\begin{tabular}{ccc}
\toprule
Additional exclusion of lowest-mass probes & $D_{\mathrm{den}}$ raw $r$ & $D_{\mathrm{den}}$ LOWESS $r$\\
\midrule
0\% & .645 (15/15) & .753 (15/15)\\
1\% & .643 (15/15) & .752 (15/15)\\
5\% & .645 (15/15) & .756 (15/15)\\
10\% & .643 (15/15) & .757 (15/15)\\
\bottomrule
\end{tabular}

\end{table}

Two precision controls use one run per private share. With all gradients computed in FP32, no probe evaluation among 9{,}600 was invalid, the smallest mass was $2.6\times10^{-7}$, and the maximum $\varepsilon/S_i$ was approximately $3.8\times10^{-6}$; the raw / LOWESS associations of $\Dden$ with the accuracy gap were $.862/.913$, $.841/.874$, and $.914/.937$ at $p=.3,.5,.7$. With BF16 autocast, 26 of 9{,}600 evaluations were invalid, all at $p=.3$; the mean associations were $.891/.955$ for $\Dw$ and $.941/.980$ for $\Dden$, positive in all three runs, and excluding the lowest 1\%, 5\%, or 10\% of masses changed those of $\Dden$ to $.940/.980$, $.940/.982$, and $.939/.983$. The following diagnostics report the monitored masses recorded in $\NumMassRuns$ language runs.

\paragraph{Recorded gradient masses in language runs}
The $\NumMassRuns$ runs of six language backbone--dataset pairs (RoBERTa-base on CoLA and MRPC, DeBERTa-v3-base on MRPC and WiC, Qwen3-1.7B on CoLA and RTE) recorded the monitored-gradient mass of every probe at every checkpoint. None of their $\NumMassEvaluations$ probe evaluations was invalid. The smallest mass was $\NumMassEncoderMin$ in the FP32 encoder runs and $\NumMassQwenMin$ in the $\NumMassQwenRuns$ BF16 LoRA runs of Qwen3-1.7B. Because the stabilized allocation equals $q_i\,S_i/(S_i+\varepsilon)$, it differs from the exact allocation by a rescaling of relative size below $\varepsilon/S_i$, at most $\NumMassEncoderRatio$ in the encoder runs and $\NumMassQwenRatio$ in the Qwen3-1.7B runs, so the stabilizer is negligible in these runs.

\subsection{Negative-Gap Checkpoints}
\label{app:negative-gap-checkpoints}
Excluding the few checkpoints with a negative accuracy gap ($\NumNlpNegCheckpoints$ of 1{,}870 language checkpoints, all at or before epoch~$\NumNlpNegLastEpoch$, and $\NumResNegCheckpoints$ of 375 hierarchy checkpoints, all within the first $\NumResNegLastEpoch$ of 25 epochs) leaves the mean raw association of $\Dden$ at $\NumNlpNonNegDdenRaw$ ($\NumNlpNonNegDdenRawPos$/90 positive) and $\NumResNonNegDdenRaw$ ($\NumResNonNegDdenRawPos$/15).

\FloatBarrier
\subsection{Module-Level Projection}
\label{app:results-module}
For a monitored allocation $q_i$, a module-level projection groups its coordinates. The monitored coordinates are partitioned into groups $G_1,\dots,G_M$, one for the weight and one for the bias of each monitored projection (for Qwen3-1.7B, one for the trainable LoRA parameters of each projection). The module allocation of probe $i$ has the entries $q^{\mathrm{mod}}_{i,m}(t)=\sum_{c\in G_m}q_{i,c}(t)$, the shares of its sensitivity carried by the groups, and with their mean $\bar q^{\mathrm{mod}}(t)$ over the probes,
\begin{equation}
\label{eq:dmod}
\Dmod(t)=\frac1N\sum_{i=1}^N\bigl\|q^{\mathrm{mod}}_i(t)-\bar q^{\mathrm{mod}}(t)\bigr\|_2^2 .
\end{equation}
$\Dmod$ keeps how each probe divides its sensitivity among these parameter groups of the final blocks and discards how the sensitivity is distributed within a group. It is a coarser view of the same allocations rather than a component of $\Dden$. Table~\ref{tab:module} compares the associations of the two readouts along the same training trajectories.

\begin{table}[htbp]
\centering
\caption{Coordinate-level readout $\Dden$ and module-level readout $\Dmod$ on the same runs. Within-run correlation with the accuracy gap as mean raw $r$ / mean LOWESS $r$ / runs positive under both; SD: between-run standard deviation of the raw correlation; $\Dden$ higher: runs in which the raw correlation of $\Dden$ exceeds that of $\Dmod$. Tables~\ref{tab:app-module} and~\ref{tab:app-module-time} give the values by dataset and private share and the time-controlled checks.}
\label{tab:module}
\small
\resizebox{\linewidth}{!}{%
\begin{tabular}{llccccc}
\toprule
 & & \multicolumn{2}{c}{$D_{\mathrm{den}}$} & \multicolumn{2}{c}{$D_{\mathrm{mod}}$} & \\
\cmidrule(lr){3-4}\cmidrule(lr){5-6}
Study & Backbone & Raw / LOWESS / positive & SD & Raw / LOWESS / positive & SD & $D_{\mathrm{den}}$ higher\\
\midrule
Language & RoBERTa & .877 / .933 / 30/30 & .076 & .316 / .359 / 21/30 & .432 & 30/30\\
 & DeBERTa & .878 / .935 / 30/30 & .132 & .554 / .644 / 29/30 & .258 & 29/30\\
 & Qwen3-1.7B & .770 / .819 / 30/30 & .250 & .744 / .801 / 29/30 & .248 & 22/30\\
\midrule
 & All 90 runs & .841 / .895 / 90/90 & .174 & .538 / .601 / 79/90 & .365 & 81/90\\
\midrule
Hierarchy & ResMLP-24 & .887 / .916 / 15/15 & .163 & .855 / .894 / 15/15 & .139 & 11/15\\
\bottomrule
\end{tabular}
}
\end{table}

In the language runs, $\Dmod$ is associated with the accuracy gap, but more weakly and less stably than $\Dden$. Its mean correlation is $\NumModNlpRaw$ raw and $\NumModNlpLowess$ after LOWESS, positive under both views in $\NumModNlpJoint$ of 90 runs, against $\NumNlpDdenRaw$, $\NumNlpDdenLowess$, and 90 of 90 for $\Dden$. Its between-run standard deviation is $\NumModNlpSd$, against $\NumModNlpDdenSd$, and $\Dden$ has the larger raw correlation in $\NumModNlpDenHigher$ of 90 runs. All five runs are positive in $\NumModCondAllPos$ of the 18 backbone--dataset pairs for $\Dmod$ and in all 18 for $\Dden$; the lowest pair mean is $\NumModCondLowest$ for $\Dmod$ (RoBERTa-base on RTE) and $\NumModCondDdenLowest$ for $\Dden$. The difference is largest in the encoder backbones, where the mean raw correlation of $\Dmod$ is $\NumModRobertaRaw$ (RoBERTa-base) and $\NumModDebertaRaw$ (DeBERTa-v3-base); in Qwen3-1.7B it is $\NumModQwenRaw$, close to the $\NumModQwenDdenRaw$ of $\Dden$.

In the controlled hierarchy, $\Dmod$ follows the accuracy gap almost as closely as $\Dden$: its mean correlation is $\NumModResRaw$ raw and $\NumModResLowess$ after LOWESS, positive in all 15 runs, against $\NumModResDdenRaw$ and $\NumModResDdenLowess$ for $\Dden$, which has the larger raw correlation in $\NumModResDenHigher$ of 15 runs.

The module-level projection retains a strong association with the gap in the hierarchy and in Qwen3-1.7B, whereas its association is weaker and varies in sign between runs in the encoder backbones. Across both studies, the coordinate-level readout is positively associated with the gap in every run.

Table~\ref{tab:app-module} gives this comparison for every backbone--dataset pair and private share. In the language runs, $\Dmod$ falls below $\Dden$ in every RoBERTa-base and DeBERTa-v3-base pair, with mean raw correlations of $\NumModRobertaMrpcRaw$ on MRPC ($\NumModRobertaMrpcPos$ of five runs positive under both views) and $\NumModRobertaRteRaw$ on RTE ($\NumModRobertaRtePos$ of five) for RoBERTa-base; in the Qwen3-1.7B pairs and in the hierarchy the two readouts are close.

Table~\ref{tab:app-module-time} applies the stricter checks of Section~\ref{app:results-nlp} to the language runs: the correlation after removing a linear time trend, and the correlation of one-epoch changes. $\Dmod$ remains positive after the time control in $\NumModNlpPartialPos$/90 runs (mean $\NumModNlpPartial$) and its one-epoch changes in $\NumModNlpDiffPos$/90 (mean $\NumModNlpDiff$), compared with $\NumNlpDdenPartialPos$/90 and $\NumNlpDdenDiffPos$/90 for $\Dden$.

\begin{table}[htbp]
\centering
\caption{Coordinate-level readout $\Dden$ and module-level projection $\Dmod$ by dataset and private share: mean raw / LOWESS within-run Pearson $r$ with the accuracy gap over five runs, the number of runs positive under both views, and the number of runs in which the raw correlation of $\Dden$ exceeds that of $\Dmod$.}
\label{tab:app-module}
\small
\begin{tabular}{llccc}
\toprule
Backbone & Dataset & $D_{\mathrm{den}}$ & $D_{\mathrm{mod}}$ & $D_{\mathrm{den}}$ higher\\
\midrule
RoBERTa & BoolQ & $.912/.935$ (5/5) & $.673/.783$ (5/5) & 5/5\\
RoBERTa & CoLA & $.830/.936$ (5/5) & $.141/.155$ (3/5) & 5/5\\
RoBERTa & MRPC & $.864/.945$ (5/5) & $.116/.075$ (2/5) & 5/5\\
RoBERTa & RTE & $.934/.963$ (5/5) & ${-.022}/{-.102}$ (2/5) & 5/5\\
RoBERTa & SST-2 & $.797/.851$ (5/5) & $.373/.444$ (4/5) & 5/5\\
RoBERTa & WiC & $.923/.969$ (5/5) & $.616/.800$ (5/5) & 5/5\\
\midrule
DeBERTa & BoolQ & $.924/.972$ (5/5) & $.550/.630$ (5/5) & 5/5\\
DeBERTa & CoLA & $.870/.944$ (5/5) & $.455/.557$ (5/5) & 5/5\\
DeBERTa & MRPC & $.889/.963$ (5/5) & $.363/.411$ (4/5) & 5/5\\
DeBERTa & RTE & $.871/.946$ (5/5) & $.601/.718$ (5/5) & 5/5\\
DeBERTa & SST-2 & $.912/.967$ (5/5) & $.715/.809$ (5/5) & 5/5\\
DeBERTa & WiC & $.799/.818$ (5/5) & $.637/.737$ (5/5) & 4/5\\
\midrule
Qwen3-1.7B & BoolQ & $.942/.957$ (5/5) & $.855/.903$ (5/5) & 5/5\\
Qwen3-1.7B & CoLA & $.687/.785$ (5/5) & $.774/.838$ (5/5) & 2/5\\
Qwen3-1.7B & MRPC & $.932/.963$ (5/5) & $.902/.953$ (5/5) & 3/5\\
Qwen3-1.7B & RTE & $.568/.641$ (5/5) & $.548/.628$ (5/5) & 4/5\\
Qwen3-1.7B & SST-2 & $.550/.601$ (5/5) & $.538/.600$ (4/5) & 3/5\\
Qwen3-1.7B & WiC & $.943/.964$ (5/5) & $.845/.883$ (5/5) & 5/5\\
\midrule
ResMLP-24 & $p=.3$ & $.911/.951$ (5/5) & $.821/.869$ (5/5) & 4/5\\
ResMLP-24 & $p=.5$ & $.777/.812$ (5/5) & $.783/.828$ (5/5) & 3/5\\
ResMLP-24 & $p=.7$ & $.972/.984$ (5/5) & $.962/.985$ (5/5) & 4/5\\
\bottomrule
\end{tabular}

\end{table}

\begin{table}[htbp]
\centering
\caption{Stricter checks of $\Dden$ and $\Dmod$ in the language runs. Linear-time control: correlation with the accuracy gap after removing a linear epoch trend from both series. $\Delta$: correlation of one-epoch changes. Mean within-run $r$ with the number of positive runs.}
\label{tab:app-module-time}
\small
\begin{tabular}{lcccc}
\toprule
 & \multicolumn{2}{c}{Linear-time control} & \multicolumn{2}{c}{$\Delta$ vs.\ $\Delta\GapAcc$}\\
\cmidrule(lr){2-3}\cmidrule(lr){4-5}
Backbone & $D_{\mathrm{den}}$ & $D_{\mathrm{mod}}$ & $D_{\mathrm{den}}$ & $D_{\mathrm{mod}}$\\
\midrule
RoBERTa & .795 (30/30) & .341 (25/30) & .644 (30/30) & .228 (25/30)\\
DeBERTa & .807 (30/30) & .262 (21/30) & .596 (29/30) & .083 (18/30)\\
Qwen3-1.7B & .715 (29/30) & .561 (26/30) & .522 (26/30) & .288 (23/30)\\
\midrule
All 90 runs & .772 (89/90) & .388 (72/90) & .587 (85/90) & .200 (66/90)\\
\bottomrule
\end{tabular}

\end{table}

\FloatBarrier
\subsection{Signed Gradients}
\label{app:sensitivity-signed}
A separate set of 15 hierarchy runs with the configuration of the main study records signed readouts alongside the absolute ones on the same training trajectories; these runs are distinct from those of Table~\ref{tab:resmlp-main}. The signed readouts keep the sign of each coordinate: over the same valid probes, $r_i(t)=g_i(t)/(S_i(t)+\varepsilon)$ replaces $\tilde q_i(t)$ in Eqs.~\eqref{eq:dden} and~\eqref{eq:within-between}, giving $D_{\mathrm{den}}^{\pm}$, $D_{\mathrm{within}}^{\mathrm{wt},\pm}$, and $D_{\mathrm{between}}^{\pm}=D_{\mathrm{den}}^{\pm}-D_{\mathrm{within}}^{\mathrm{wt},\pm}$; the unnormalized signed companions $U_{\mathrm{den}}^{\pm}$ and $U_{\mathrm{within}}^{\mathrm{wt},\pm}$ use $g_i(t)$ in place of $|g_i(t)|$. Because $|r_i(t)|=\tilde q_i(t)$, the signed and absolute readouts differ only through the signs. Table~\ref{tab:app-signed} summarizes the associations with the accuracy gap and Table~\ref{tab:app-signed-runs} lists the individual runs.

The signed readouts retain a strong association with the gap: the mean raw correlation is $\NumSignedDdenRaw$ for $D_{\mathrm{den}}^{\pm}$ and $\NumSignedDwRaw$ for $D_{\mathrm{within}}^{\mathrm{wt},\pm}$, both positive under both views in $\NumSignedDdenPos$/15 runs. As expected from Section~\ref{sec:absolute-gradients}, the signed and absolute readouts differ mainly in the between-class term, where label-dependent gradient signs enter: $D_{\mathrm{between}}^{\pm}$ is positive in $\NumSignedDbPos$/15 runs with a mean raw correlation of $\NumSignedDbRaw$, against $\NumSignedAbsDbPos$/15 for $\Db$. The absolute pooled readout $\Dden$ is more stable across runs than its signed counterpart: the between-run standard deviation of its raw correlation is $\NumSignedAbsDdenSd$ against $\NumSignedDdenSd$, and its smallest raw correlation is $\NumSignedAbsDdenMin$ against $\NumSignedDdenMin$.

\begin{table}[htbp]
\centering
\caption{Signed-gradient control: 15 hierarchy runs with the main-study configuration, recording absolute and signed readouts on the same trajectories. Cells give the mean raw / LOWESS within-run Pearson correlation with the accuracy gap and the number of runs positive under both views; SD and Min are the between-run standard deviation and the minimum of the raw correlation over the 15 runs. Signed readouts ($\pm$) keep the sign of each gradient coordinate.}
\label{tab:app-signed}
\small
\resizebox{\linewidth}{!}{%
\begin{tabular}{llcccccc}
\toprule
 & & \multicolumn{3}{c}{Private share $p$ (five runs each)} & \multicolumn{3}{c}{All 15 runs}\\
\cmidrule(lr){3-5}\cmidrule(lr){6-8}
Gradients & Readout & .3 & .5 & .7 & Mean & SD & Min\\
\midrule
\multirow{3}{*}{Absolute} & $D_{\mathrm{den}}$ & .825 / .943 (5/5) & .930 / .967 (5/5) & .887 / .937 (5/5) & .881 / .949 (15/15) & .137 & .485\\
 & $D_{\mathrm{within}}^{\mathrm{wt}}$ & .796 / .921 (5/5) & .932 / .962 (5/5) & .840 / .891 (5/5) & .856 / .925 (15/15) & .179 & .350\\
 & $D_{\mathrm{between}}$ & .832 / .936 (5/5) & .787 / .877 (5/5) & .882 / .930 (5/5) & .833 / .914 (15/15) & .142 & .433\\
\midrule
\multirow{3}{*}{Signed} & $D_{\mathrm{den}}^{\pm}$ & .764 / .876 (5/5) & .923 / .974 (5/5) & .800 / .875 (5/5) & .829 / .908 (15/15) & .240 & .189\\
 & $D_{\mathrm{within}}^{\mathrm{wt},\pm}$ & .871 / .948 (5/5) & .942 / .969 (5/5) & .896 / .928 (5/5) & .903 / .948 (15/15) & .084 & .739\\
 & $D_{\mathrm{between}}^{\pm}$ & .521 / .583 (4/5) & \ensuremath{-.135} / \ensuremath{-.160} (2/5) & .201 / .182 (3/5) & .196 / .202 (9/15) & .662 & \ensuremath{-.750}\\
\midrule
\multirow{2}{*}{Unnormalized, absolute} & $U_{\mathrm{den}}$ & .170 / .233 (5/5) & .267 / .317 (5/5) & .608 / .685 (5/5) & .349 / .412 (15/15) & .284 & .031\\
 & $U_{\mathrm{within}}^{\mathrm{wt}}$ & .189 / .262 (5/5) & .291 / .342 (5/5) & .617 / .693 (5/5) & .366 / .432 (15/15) & .276 & .059\\
\midrule
\multirow{2}{*}{Unnormalized, signed} & $U_{\mathrm{den}}^{\pm}$ & .125 / .173 (4/5) & .220 / .256 (4/5) & .596 / .661 (5/5) & .314 / .363 (13/15) & .307 & \ensuremath{-.037}\\
 & $U_{\mathrm{within}}^{\mathrm{wt},\pm}$ & .161 / .227 (4/5) & .268 / .313 (5/5) & .617 / .689 (5/5) & .349 / .410 (14/15) & .291 & .027\\
\bottomrule
\end{tabular}
}
\end{table}

\begin{table}[htbp]
\centering
\caption{Individual runs of the signed-gradient control: final training-evaluation and validation accuracy, and the raw / LOWESS within-run correlation of each readout with the accuracy gap. Run numbers identify the five runs within each private share.}
\label{tab:app-signed-runs}
\small
\resizebox{\linewidth}{!}{%
\begin{tabular}{ccrrcccccccc}
\toprule
 & & \multicolumn{2}{c}{Final accuracy (\%)} & \multicolumn{8}{c}{Within-run $r$ with the accuracy gap, raw / LOWESS}\\
\cmidrule(lr){3-4}\cmidrule(lr){5-12}
$p$ & Run & Train & Validation & $D_{\mathrm{den}}$ & $D_{\mathrm{within}}^{\mathrm{wt}}$ & $D_{\mathrm{between}}$ & $D_{\mathrm{den}}^{\pm}$ & $D_{\mathrm{within}}^{\mathrm{wt},\pm}$ & $D_{\mathrm{between}}^{\pm}$ & $U_{\mathrm{den}}$ & $U_{\mathrm{den}}^{\pm}$\\
\midrule
.3 & 1 & 100.00 & 80.98 & .946/.990 & .943/.985 & .910/.987 & .932/.978 & .936/.976 & .657/.965 & .041/.075 & .015/.039\\
.3 & 2 & 98.38 & 84.80 & .861/.933 & .861/.933 & .850/.931 & .879/.950 & .845/.916 & .836/.938 & .405/.622 & .399/.590\\
.3 & 3 & 100.00 & 84.88 & .867/.957 & .881/.967 & .828/.929 & .852/.940 & .879/.969 & .704/.783 & .153/.178 & .100/.119\\
.3 & 4 & 99.46 & 78.43 & .968/.992 & .946/.969 & .924/.958 & .966/.990 & .956/.977 & .915/.968 & .221/.286 & .149/.196\\
.3 & 5 & 100.00 & 81.07 & .485/.841 & .350/.750 & .646/.875 & .189/.521 & .739/.902 & \ensuremath{-.508}/\ensuremath{-.737} & .031/.004 & \ensuremath{-.037}/\ensuremath{-.080}\\
\midrule
.5 & 1 & 98.38 & 66.83 & .964/.991 & .961/.989 & .957/.996 & .961/.992 & .950/.983 & .796/.918 & .477/.616 & .468/.590\\
.5 & 2 & 94.05 & 63.73 & .779/.882 & .768/.850 & .433/.651 & .718/.897 & .824/.886 & \ensuremath{-.743}/\ensuremath{-.807} & .198/.322 & .171/.272\\
.5 & 3 & 100.00 & 71.71 & .946/.969 & .967/.981 & .722/.837 & .975/.990 & .971/.984 & .420/.424 & .063/.026 & \ensuremath{-.025}/\ensuremath{-.080}\\
.5 & 4 & 98.38 & 61.76 & .978/.995 & .978/.993 & .858/.924 & .982/.997 & .981/.996 & \ensuremath{-.750}/\ensuremath{-.871} & .266/.275 & .194/.198\\
.5 & 5 & 98.92 & 66.50 & .986/.998 & .986/.996 & .962/.979 & .980/.996 & .986/.995 & \ensuremath{-.398}/\ensuremath{-.466} & .334/.344 & .293/.300\\
\midrule
.7 & 1 & 100.00 & 45.37 & .941/.984 & .941/.987 & .856/.924 & .913/.964 & .954/.987 & \ensuremath{-.336}/\ensuremath{-.421} & .208/.252 & .174/.200\\
.7 & 2 & 99.46 & 53.43 & .866/.904 & .735/.773 & .954/.985 & .805/.872 & .742/.779 & .266/.382 & .936/.959 & .952/.972\\
.7 & 3 & 90.81 & 45.37 & .953/.971 & .951/.971 & .948/.964 & .946/.965 & .949/.969 & .924/.951 & .931/.948 & .933/.949\\
.7 & 4 & 97.84 & 46.08 & .976/.991 & .973/.989 & .835/.881 & .979/.990 & .976/.988 & .727/.800 & .538/.655 & .500/.603\\
.7 & 5 & 98.38 & 47.57 & .700/.837 & .599/.733 & .818/.894 & .355/.584 & .859/.917 & \ensuremath{-.575}/\ensuremath{-.804} & .428/.613 & .422/.583\\
\bottomrule
\end{tabular}
}
\end{table}
\FloatBarrier

\clearpage
\section{Training Configuration}
\label{app:config}
Table~\ref{tab:config} summarizes the study settings and the language calibration protocol. The split and hyperparameter selection are detailed in Section~\ref{app:hpo}; the BoolQ depth check of Section~\ref{app:sensitivity-depth-nlp} uses separate training runs.

Table~\ref{tab:config-runs} lists the training settings of each experiment; the five runs of an experiment share them. All experiments use AdamW without gradient clipping. The listed dropout is applied to the hidden, attention, and classifier components of the language backbones; LoRA dropout is specified separately below. Language inputs are truncated to 256 tokens. Qwen3-1.7B starts from the base checkpoint (Qwen3-1.7B-Base) and trains rank-32 LoRA adapters with $\alpha=128$ and LoRA dropout 0 on the query, key, value, output, gate, up, and down projections. ResMLP-24 starts from the \texttt{timm} checkpoint \texttt{resmlp\_24\_224.fb\_distilled\_in1k} and accumulates gradients over 8 batches of 16, an effective batch size of 128. The learning rate rises linearly over the warmup fraction of optimizer steps and then decays linearly to zero at the end of training.

\begin{table}[htbp]
\centering
\caption{Study settings and language calibration protocol.}
\label{tab:config}
\small
\begin{tabular}{@{}p{0.22\linewidth}p{0.3\linewidth}p{0.2\linewidth}p{0.18\linewidth}@{}}
\toprule
Setting & Language & Controlled hierarchy & Vision\\
\midrule
Backbone & RoBERTa-base, DeBERTa-v3-base, Qwen3-1.7B & ResMLP-24 (pretrained) & ResNet-18 (ImageNet)\\
Adaptation & full (RoBERTa, DeBERTa); frozen backbone with rank-32 LoRA on all attention and MLP projections (Qwen) & full & full\\
Monitor-depth calibration & BoolQ, once per backbone & -- & STL-10, once\\
Learning-rate search & per dataset from $\{1,1.5,2,3\}\times10^{-5}$ & -- (specified $1.5\times10^{-5}$) & from $\{1,2,3,5\}\times10^{-4}$\\
Dropout & fixed at 0 & fixed at 0 & fixed at 0\\
HPO train / selection data & 90\% / fixed class-stratified 10\% of official train & -- & 90\% / 10\% of official train; DTD: official train / val\\
Formal training / evaluation source & full official train / full official validation & full training / full validation split & full official train (DTD: train+val) / full official test\\
Epochs; checkpoints & 10; 20--22 checkpoints, about every half epoch & 25; every epoch & 40; every epoch\\
Gradient probes & 128, class-proportional & 128, 32 per class & 10 per class (CIFAR-10, STL-10), 2 per class (CIFAR-100), 4 per class (DTD)\\
HPO evaluation & 1 epoch, 2 seeds disjoint from main runs; equal-weight mean accuracy & -- & 5 epochs; selection-split accuracy\\
Monitored final blocks $k$ & 4 / 6 / 2 (RoBERTa / DeBERTa / Qwen) & 10 & 2\\
Monitoring scheme & shared-output weight gradients & shared-output weight gradients & shared-output weight gradients\\
Training-evaluation subset & 10\%, class-proportional & 10\% (185 examples) & 10\%; DTD: full train+val\\
\bottomrule
\end{tabular}
\end{table}

\begin{table}[htbp]
\centering
\caption{Training settings of each experiment. The five runs of an experiment share these settings. Batch gives the training / evaluation batch size, Accum.\ the gradient-accumulation steps, and Warmup the warmup ratio. The table gives the adopted settings. Learning rates are selected as described in Section~\ref{app:hpo}; dropout and the remaining settings are fixed per backbone.}
\label{tab:config-runs}
\small
\setlength{\tabcolsep}{1.8pt}
\begin{tabular}{@{}ll*{6}{c}l@{}}
\toprule
Backbone & Experiment & Learning rate & Batch & Accum. & Weight decay & Warmup & Dropout & Precision\\
\midrule
RoBERTa-base & BoolQ & $2\times10^{-5}$ & 16 / 32 & 1 & 0 & 0 & 0 & FP32\\
 & CoLA & $1\times10^{-5}$ & 16 / 32 & 1 & 0 & 0 & 0 & FP32\\
 & MRPC & $1\times10^{-5}$ & 16 / 32 & 1 & 0 & 0 & 0 & FP32\\
 & RTE & $2\times10^{-5}$ & 16 / 32 & 1 & 0 & 0 & 0 & FP32\\
 & SST-2 & $2\times10^{-5}$ & 16 / 32 & 1 & 0 & 0 & 0 & FP32\\
 & WiC & $2\times10^{-5}$ & 16 / 32 & 1 & 0 & 0 & 0 & FP32\\
\midrule
DeBERTa-v3-base & BoolQ & $3\times10^{-5}$ & 16 / 32 & 1 & 0 & 0.06 & 0 & FP32\\
 & CoLA & $1\times10^{-5}$ & 16 / 32 & 1 & 0 & 0.06 & 0 & FP32\\
 & MRPC & $2\times10^{-5}$ & 16 / 32 & 1 & 0 & 0.06 & 0 & FP32\\
 & RTE & $3\times10^{-5}$ & 16 / 32 & 1 & 0 & 0.06 & 0 & FP32\\
 & SST-2 & $1\times10^{-5}$ & 16 / 32 & 1 & 0 & 0.06 & 0 & FP32\\
 & WiC & $2\times10^{-5}$ & 16 / 32 & 1 & 0 & 0.06 & 0 & FP32\\
\midrule
Qwen3-1.7B (LoRA) & BoolQ & $2\times10^{-5}$ & 8 / 16 & 1 & 0.01 & 0.05 & 0 & BF16\\
 & CoLA & $3\times10^{-5}$ & 8 / 16 & 1 & 0.01 & 0.05 & 0 & BF16\\
 & MRPC & $2\times10^{-5}$ & 8 / 16 & 1 & 0.01 & 0.05 & 0 & BF16\\
 & RTE & $3\times10^{-5}$ & 8 / 16 & 1 & 0.01 & 0.05 & 0 & BF16\\
 & SST-2 & $2\times10^{-5}$ & 8 / 16 & 1 & 0.01 & 0.05 & 0 & BF16\\
 & WiC & $2\times10^{-5}$ & 8 / 16 & 1 & 0.01 & 0.05 & 0 & BF16\\
\midrule
ResMLP-24 & $p=.3$ & $1.5\times10^{-5}$ & 16 / 32 & 8 & 0.05 & 0.1 & 0 & BF16\\
 & $p=.5$ & $1.5\times10^{-5}$ & 16 / 32 & 8 & 0.05 & 0.1 & 0 & BF16\\
 & $p=.7$ & $1.5\times10^{-5}$ & 16 / 32 & 8 & 0.05 & 0.1 & 0 & BF16\\
\midrule
ResNet-18 & CIFAR-10 & $3\times10^{-4}$ & 64 / 256 & 1 & 0.01 & 0 & 0 & BF16\\
 & CIFAR-100 & $3\times10^{-4}$ & 64 / 256 & 1 & 0.01 & 0 & 0 & BF16\\
 & STL-10 & $1\times10^{-4}$ & 64 / 256 & 1 & 0.01 & 0 & 0 & BF16\\
 & DTD & $3\times10^{-4}$ & 64 / 256 & 1 & 0.01 & 0 & 0 & BF16\\
\bottomrule
\end{tabular}
\end{table}

\end{document}